\documentclass{article}

\usepackage{arxiv}

\usepackage[utf8]{inputenc} 
\usepackage[T1]{fontenc}    
\usepackage{url}            
\usepackage{nicefrac}       

\usepackage{graphicx}
\usepackage{natbib} 
\usepackage{multibib}
\usepackage{hyperref} 
\hypersetup{ 
    colorlinks=true,
    linkcolor=blue,  
    citecolor=blue,  
    urlcolor=blue,  
    }

\usepackage{amsmath}
\usepackage{amsfonts} 
\usepackage{bm} 
\usepackage{dsfont} 

\usepackage{amsthm} 
\newtheorem{example}{Example}[section]%

\newtheorem{assumption}{Assumption}

\newtheorem{theorem}{Theorem}[section]
\newtheorem{lemma}{Lemma}[section]
\newtheorem{corollary}{Corollary}[section]

\usepackage{booktabs}  
\usepackage{adjustbox}
\usepackage{multirow}
\usepackage{multicol}
\usepackage{makecell}
\usepackage{rotating} 

\usepackage{subcaption}   
\usepackage{mwe}
\newlength{\tempheight}
\newlength{\tempwidth}
\newcommand{\rowname}[1]
{\rotatebox{90}{\makebox[\tempheight][c]{{#1}}}}
\newcommand{\columnname}[1]
{\makebox[\tempwidth][c]{{#1}}}

\usepackage{algorithm} 
\usepackage{algorithmicx} 
\usepackage{algpseudocode} 

\title{Statistical Properties of $k$-means Clustering for Data Missing Completely at Random}

\author{
 Xin Guan \\
  Graduate School of Information Sciences\\
  Tohoku University\\
  Sendai, Miyagi 980-8579, Japan\\
  \texttt{guan.xin.c5@tohoku.ac.jp} \\
}

\begin{document}
\maketitle
\begin{abstract}
The classical $k$-means clustering cannot be directly used to incomplete data, and existing $k$-means-based clustering for missing data primarily focus on improving the practical accuracy of clustering, whereas most of them lack theoretical guarantees in the asymptotic sense. 
  In this paper, we investigate the statistical properties of $k$-means clustering in the presence of missing data. 
  We first establish the $\sqrt{n}$-excess risk bound and prove the consistency of the estimated cluster centers under general missing mechanisms. For the Missing Completely at Random (MCAR) mechanism, we further derive the $\sqrt{n}$-convergence rate and asymptotic normality of the estimated cluster centers. 
  Moreover, we study in what cases the cluster centers estimated by incomplete data converge to the true cluster centers of original fully observed data, and give a sufficient condition about the missing probability and the separation among true clusters. 
  These results provide a theoretical guarantee for missing-data-$k$-means. 
  Notably, our analysis reveal that under MCAR mechanism, both achieving the $\sqrt{n}$-rate and converging to the true cluster centers require $k$ true centers to be distinct in every dimension, highlighting the significant challenges of application in high-dimensional regimes. 
  Finally, we conduct numerical simulations on synthetic incomplete datasets to support our theoretical analysis results. 
\end{abstract}


\section{Introduction}

The $k$-means clustering is one of the most widely used and fundamental clustering methods, which gives a partition for a dataset by finding $k$ cluster centers and then assigning each data point to its nearest cluster center. 
The estimated $k$ cluster centers are given by minimizing the sum of distances of each data point to its nearest cluster clusters. 
However, the requirement for a fully observed dataset of $k$-means limits its direct capacity for missing data. 

To adapt $k$-means to incomplete datasets, extensive research has been conducted, primarily focusing on improving the practical accuracy of clustering. 
Many new $k$-means-based clustering methods to deal with various missingness have been proposed, which generally fall into two categories. 
Some works aim to reconstruct the original fully observed data by various imputation methods \citep{liu2016knn,yoon2018gain,Choudhury2019nn,spinelli2020}, followed by conducting complete-data-clustering, while others consider new measurements for dissimilarity between two incomplete data points as a substitute for the Euclidean distance used in full observed data clustering \citep{abdallah2014mean,datta2018clustering,mesquita2017}. 
However, although these existing methods have good practical performance in dealing with various missing data, most of them lack a theoretical guarantee in the asymptotic sense, leaving them mathematical ungrounded.

In contrast to the lack of theoretical analysis for the missing-data-$k$-means, there are already rich literature about statistical properties of $k$-means and its variants on fully observed data. 
The most pioneering works are the celebrated results of \citet{pollard1981strong,Pollard1982}, which prove the strong consistency and Central Limitation Theorem for $k$-means clustering under mild assumptions. 
On this basis, the strong consistency of many variants of $k$-means has also been proved. For example, trimmed and robust $k$-means for dealing with outliers by \citet{cuesta1997trimmed,georgogiannis2016robust}, regularized, reduced and weighted $k$-means for high-dimensional data by \citet{sun2012regularized,raymaekers2022regularized, terada2014strong,chakraborty2019strong}, and kernel $k$-means for non-linear clustering structure by \citet{von2008consistency,paul2022implicit,liang2023consistency}. 
Moreover, some works follow the spirit of \citet{Pollard1982} and develop the Central Limitation Theorem for some variants of $k$-means, such as \citet{garcia1999central, yang1994asymptotic}. 
These existing works consider various complex data structures, whereas, the missingness of data is often excluded from analysis.

The main difficulty of analyzing missing-data-$k$-means is that whether each data point $\bm{\mathrm{x}}_i$ is missing or not is a random event, and we can use a random vector $\bm{\mathrm{r}}_i$ with each component $\mathrm{r}_{ij}$ following a Bernoulli distribution, to serve as an indicator for missingness of $\bm{\mathrm{x}}_i$. It follows that the partition of $\bm{\mathrm{x}}_i$'s is related to the distribution of $\bm{\mathrm{r}}_i$'s, and the theoretical analysis should consider the joint distribution of $\bm{\mathrm{x}}_i$'s and $\bm{\mathrm{r}}_i$'s. 
Unfortunately, existing theoretical tools for analyzing complete-data-$k$-means involves only the randomness of $\bm{\mathrm{x}}_i$'s, thus are not applicable for analyzing missing data. 
Although the missing data is quite common in practice, to the best of our knowledge, \cite{terada2025a} is the first to prove the consistency of $k$-means for missing data under the Missing Completely at Random (MCAR) mechanism, which considers a canonical framework of missing-data-$k$-means proposed by \citep{chi2016k,wang2019k} and widely used by \citep{lithio2018efficient,aschenbruck2023imputation,agliz2025joint}. 
Yet, the asymptotic property of the estimated cluster centers under more general missing mechanisms and the convergence rates still remain unclear.

In this paper, we focus on the consistency and convergence rate of the estimated cluster centers of $k$-means on missing data, especially under MCAR mechanism. 
Moreover, as pointed out by \cite{terada2025a}, in general, even under simplest MCAR mechanism, the cluster centers estimated by incomplete data may not necessarily converge to the true cluster centers of original fully observed data. 
Thus, in this paper, we are also interested in figuring out in what cases the true cluster centers can be asymptotically recovered via the estimator based on incomplete data. 
Our contributions are: 
\begin{itemize}
    \item We establish the $\sqrt{n}$-excess risk bound and prove the consistency of the estimated cluster centers under general missing mechanisms, improving the result of \cite{terada2025a}. 
    \item We derive the $\sqrt{n}$-convergence rate and asymptotic normality of the estimated cluster centers under MCAR mechanism, revealing challenges of application in high-dimensional regimes.  
    \item We provide a sufficient condition under which the cluster centers estimated by incomplete data converge to the true cluster centers of original fully observed data.  
    \item Numerical simulations carried out on synthetic incomplete datasets support our theoretical analysis results.
\end{itemize}

\section{Notations and preliminaries}

\subsection{Notations}

Let $\mathcal{X}\subset \mathbb{R}^p$ be the data space. 
Denote by $\mathbb{P}$ a probability measure supported on $\mathcal{X}$. 
Let $\bm{\mathrm{x}}_1,\dots,\bm{\mathrm{x}}_n \in \mathcal{X}$ be independent random vectors distributed from $\mathbb{P}$, where $\bm{\mathrm{x}}_i=(\mathrm{x}_{i1},\dots,\mathrm{x}_{ip})^T$ for any $i=1,\dots,n$. 
Denote by $\mathbb{Q}$ a probability measure supported on $\{0,1\}^p$. 
Let $\bm{\mathrm{r}}_1,\dots,\bm{\mathrm{r}}_n \in \{0,1\}^p$ be independent random vectors distributed from $\mathbb{Q}$, where $\bm{\mathrm{r}}_i=(\mathrm{r}_{i1},\dots,\mathrm{r}_{ip})^T$ for any $i=1,\dots,n$. The $\mathrm{r}_{ij}=1$ means $\mathrm{x}_{ij}$ is observed, and 0 means missing. 
Write $\widetilde{\mathbb{P}}$ for the joint probability measure of $\mathbb{P}$ and $\mathbb{Q}$.  
Write $\widetilde{\mathbb{P}}_n$ for the empirical measure obtained by placing $1/n$ at each of $(\bm{\mathrm{x}}_1,\bm{\mathrm{r}}_1),\dots,(\bm{\mathrm{x}}_n,\bm{\mathrm{r}}_n)$, and denote the associated empirical process by $\mathbb{G}_n$, i.e., for any function $g:\mathcal{X}\times \{0,1\}^p \mapsto \mathbb{R}$,
\begin{align*}
    \mathbb{G}_n g=\sqrt{n}\left( \widetilde{\mathbb{P}}_n g - \widetilde{\mathbb{P}}g \right)
    =\frac{1}{\sqrt{n}}\sum_{i=1}^{n}  \left\{ g(\bm{\mathrm{x}}_i,\bm{\mathrm{r}}_i) - 
    \int_{(\bm{x},\bm{r}) } g(\bm{x},\bm{r}) \;d\widetilde{\mathbb{P}}(\bm{x},\bm{r}) \right\}.
\end{align*}
In addition, denote by $\|\cdot\|_2$ the $l_2$ norm of any vector, $\langle \cdot,\cdot \rangle$ the inner product, and $\circ$ the element-wise product in $\mathbb{R}^p$. For any matrix, $\|\cdot\|_F$ is the Frobenius norm and $\text{vec}(\cdot)$ is the vectorization. The $\mathds{1}(\cdot)$ is the indicator function, and $\bm{1}_k$ and $\bm{0}_k$ are all-one and all-zero vectors of length $k$, and $\bm{I}_p$ is the identical matrix of size $p\times p$.

\subsection{The $k$-means clustering for missing data}

Let $\bm{M}=(\bm{\mu}_1,\dots,\bm{\mu}_k)\in\mathbb{R}^{p\times k}$ be the matrix of $k$ cluster centers, where $\bm{\mu}_l=(\mu_{l1},\dots,\mu_{lp})^T\in\mathbb{R}^{p}$ is the $l$-th cluster center. Let $\mathcal{X}^k=\bigotimes_{l=1}^{k} \mathcal{X}$ be the Cartesian product of $\mathcal{X}$, and write $\bm{M}\in \mathcal{X}^k$ to express each $\bm{\mu}_l\in\mathcal{X}$. 
We consider the same framework as \citet{terada2025a}, where the objective function of $k$-means clustering on the incomplete data $(\bm{\mathrm{x}}_1,\bm{\mathrm{r}}_1),\dots,(\bm{\mathrm{x}}_n,\bm{\mathrm{r}}_n)$ is given by 
\begin{align*}
    \widehat{L}_n(\bm{M}) 
    = \frac{1}{n}\sum_{i=1}^{n} \min_{l=1,\dots,k} \sum_{j=1}^{p} \mathrm{r}_{ij} (\mathrm{x}_{ij} - \mu_{lj})^2.
\end{align*}
The population-level counterpart of $\widehat{L}_n(\cdot)$ is given by 
\begin{align*}
    L(\bm{M})
    = \mathbb{E}_{\bm{\mathrm{x}}_1,\bm{\mathrm{r}}_1}\left[ \min_{l=1,\dots,k} \sum_{j=1}^{p} \mathrm{r}_{1j} (\mathrm{x}_{1j} - \mu_{lj})^2 \right].
\end{align*}
Through this paper, we give the following assumptions: 
\begin{assumption}
\label{assumption_X_compact}
    Suppose that $\mathcal{X}\subset \mathbb{R}^p$ is a compact metric space, where $\|\bm{x}\|_2\leq B$ for any $\bm{x}\in \mathcal{X}$.
\end{assumption}
\begin{assumption}
\label{assumption_unique_minimizer}
    Suppose that the minimizers of $\widehat{L}_n(\bm{M})$ and $L(\bm{M})$ in $\mathcal{X}^k$ are limited and unique without considering the permutation of cluster indexes. 
\end{assumption}
The Assumption~\ref{assumption_X_compact} implies $\mathbb{E}_{\bm{\mathrm{x}}_1}[\|\bm{\mathrm{x}}_1\|_2^2]<\infty$, and as a consequence of Assumption~\ref{assumption_unique_minimizer}, we can define the ``\textit{estimated cluster centers}" and the ``\textit{population-level optimizer}" to be 
    \begin{align*}
        \widehat{\bm{\mathrm{M}}}_n = \mathop{\arg\min}_{\bm{M}\in \mathcal{X}^k} \widehat{L}_n(\bm{M})
        \quad\text{and}\quad 
        \bm{M}^{\ast} = \mathop{\arg\min}_{\bm{M}\in \mathcal{X}^k} L(\bm{M}).
    \end{align*}
It should be noted that here we consider the global minimizer instead of some local minimizer obtained by some specific algorithm.

\section{Main results}

\subsection{Finite-sample analysis and consistency}
\label{sec_finite_result}

Our first aim is to establish the uniform convergence bound of $\widehat{L}_n$ to $L$ for any $\bm{M}\in\mathcal{X}^k$, based on which we can give the excess risk bound and prove the consistency of the estimator. 

To this end, we pose the missing-data-$k$-means clustering problem as a risk minimization task, where we call the $\widehat{L}_n(\cdot)$ the empirical loss and $L(\cdot)$ the expected loss. 
Then, we employ the Rademacher complexity to give the following uniform convergence bound. 
\begin{lemma}
\label{lemma_uniform_convergence}
    Under Assumption~\ref{assumption_X_compact}, it holds that for any $\delta\in (0,1)$, 
    \begin{align*}
        \textnormal{Pr}\left( 
        \sup_{\bm{M}\in \mathcal{X}^k}
        \left| \widehat{L}_n(\bm{M}) - L(\bm{M})  \right| \leq  \frac{B^2}{\sqrt{n}} \cdot \left\{ 4k(\sqrt{p}+2) + 4\sqrt{2}\cdot \sqrt{\log(1/\delta)} \right\}
        \right)
        > 1-\delta.
    \end{align*}
\end{lemma}

The above lemma shows that at any $\bm{M}$, the empirical loss converges to the population loss in a rate of $n^{-1/2}$ in probability. Thus, we can immediately obtain the excess risk bound of $O_P(n^{-1/2})$, which implies that the gap between the expected loss at $\widehat{\bm{\mathrm{M}}}_n$ and the minimum expected loss will vanish in a rate of $n^{-1/2}$.  
\begin{theorem}
\label{theorem_excess_risk}
    Under Assumptions~\ref{assumption_X_compact}-\ref{assumption_unique_minimizer}, the excess risk $L(\widehat{\bm{\mathrm{M}}}_n) - \min_{\bm{M}\in \mathcal{X}^k}L(\bm{M})=O_P(n^{-1/2})$. 
\end{theorem}

Moreover, based on the uniform convergence of loss function and the identifiability of $\bm{M}^{\ast}$ given in the following Lemma~\ref{lemma_identifiability}, we obtain the consistency of $\widehat{\bm{\mathrm{M}}}_n$ to the population-level optimizer $\bm{M}^{\ast}$. 

\begin{lemma}
\label{lemma_identifiability}
    Under Assumptions~\ref{assumption_X_compact}-\ref{assumption_unique_minimizer}, the minimizer of $L(\cdot)$ in $\mathcal{X}^k$ is identifiable, that is, for any $\epsilon>0$, $L(\bm{M}^{\ast}) < \inf \left\{ L(\bm{M}) \mid \|\textnormal{vec}(\bm{M} - \bm{M}^{\ast})\|_2>\epsilon,\; \bm{M}\in \mathcal{X}^k  \right\}$. 
\end{lemma}

\begin{theorem}
\label{theorem_consistency}
    Under Assumptions~\ref{assumption_X_compact}-\ref{assumption_unique_minimizer}, we have $\lim_{n\rightarrow\infty}\textnormal{Pr}\left( \|\textnormal{vec}(\widehat{\mathbf{M}}_n-\bm{M}^{\ast})\|_2>\epsilon \right)=0$ for any $\epsilon>0$. 
\end{theorem}

It should be noted that the uniform convergence of $\widehat{L}_n$ to $L$ given in Lemma~\ref{lemma_uniform_convergence} essentially does not rely on the assumption of missing mechanisms, which means that it holds not only for MCAR mechanism, but also for other missing mechanisms. 
Consequently, the convergence rate of the excess risk and the consistency of estimated cluster centers given in Theorems~\ref{theorem_excess_risk}-\ref{theorem_consistency} hold for other missing mechanisms as well. 
This benefits from the use of the Rademacher complexity. In contrast, \cite{terada2025a} only prove the case of MCAR.

\subsection{Convergence rate and asymptotic normality}
\label{sec_convergence_rate}

Our second aim is to prove the sequence $\{\widehat{\bm{\mathrm{M}}}_n\}_{n=1}^{\infty}$ converges to $\bm{M}^{\ast}$ in a rate of $n^{-1/2}$, and is asymptotically normal at $\bm{M}^{\ast}$ as well. 
It suffices to derive the quadratic approximation of the empirical loss $\widehat{L}_n$ at $\bm{M}^{\ast}$, for which we give the following assumptions. 
\begin{assumption}
\label{assumption_surface_zeromeasure}
    Suppose that $\mathbb{P}$ gives zero measure to any hyperplane in $\mathbb{R}^p$. 
\end{assumption}
\begin{assumption}
\label{assumption_mcar}
    Suppose that the missingness of data is completely at random (MCAR), i.e., for any $i=1,\dots,n$ and $j=1,\dots,p$, $\textnormal{Pr}\left( \mathrm{r}_{ij} = 1 \;|\; \bm{\mathrm{x}}_1,\dots,\bm{\mathrm{x}}_n \right) = \textnormal{Pr} \left( \mathrm{r}_{ij} = 1 \right) = q_{ij}$, where each $q_{ij}\in (0,1]$ is a constant. 
\end{assumption}

The Assumption~\ref{assumption_surface_zeromeasure} is commonly used to derive asymptotic properties of many variants of $k$-means clustering, which ensures that for the fully observed data, the probability of lying on the cluster boundaries is zero. 
It should be noted that some related works \citep{Pollard1982,sun2012regularized,raymaekers2022regularized} assumed that  $\mathop{\arg\min}_{l=1,\dots,k}\|\bm{\mathrm{x}}_i - \bm{\mu}_l \|_2^2$ is unique with probability one, which is implied by Assumption~\ref{assumption_surface_zeromeasure}. 
The Assumption~\ref{assumption_mcar} requires the missingness of $(i,j)$-th position to be independent with the distribution of $\{\bm{\mathrm{x}}_1,\dots,\bm{\mathrm{x}}_n\}$.

In the case of missing data, the cluster assignment is based on $\mathop{\arg\min}_{l=1,\dots,k} \|\bm{x}\circ\bm{r} - \bm{\mu_l}\circ\bm{r} \|_2^2$, which is not necessarily unique for some $\bm{r}\in\{0,1\}^p$. 
To address this issue, we here let $\ell(\bm{x}, \bm{r}, \bm{M})$ be the unique cluster assignment of $(\bm{x}, \bm{r})$ under $\bm{M}$, i.e.,
\begin{align*}
    \ell(\bm{x},\bm{r},\bm{M}) = \min \left\{  \mathop{\arg\min}_{l=1,\dots,k} \|\bm{x} \circ\bm{r} -\bm{\mu}_l \circ\bm{r}\|_2^2 \right\},
\end{align*}
where the $\min$ function ensures the uniqueness of $\mathop{\arg\min}$ for all $\bm{r}\in\{0,1\}^p$, facilitating our proofs. 
Moreover, we define a function $\phi:\mathbb{R}^p\times \{0,1\}^p\times \mathbb{R}^{p\times k} \mapsto \mathbb{R}$ to be 
\begin{align*}
    \phi(\bm{x},\bm{r},\bm{M}) = \min_{l=1,\dots,k} \|\bm{x} \circ\bm{r} -\bm{\mu}_l \circ\bm{r}\|_2^2,
\end{align*}
then the empirical loss $\widehat{L}_n$ and expected loss $L$ can be rewritten as $\widehat{L}_n(\bm{M}) = \widetilde{\mathbb{P}}_n \phi(\cdot,\cdot,\bm{M})$ and $L(\bm{M})= \widetilde{\mathbb{P}} \phi(\cdot,\cdot,\bm{M})$. 
It also implies that $\widehat{L}_n(\bm{M}) = \widetilde{\mathbb{P}}\phi(\cdot,\cdot,\bm{M}) + n^{-1/2} \mathbb{G}_n\phi(\cdot,\cdot,\bm{M})$. 

First, we prove the first-order differentiability of $\widehat{L}_n(\bm{M})$, it suffices to prove the differentiability of $L(\bm{M})$ and $\mathbb{G}_n\phi(\cdot,\cdot,\bm{M})$, for which we provide Lemma~\ref{lemma_L_1st_derivative}-\ref{lemma_Gn_1st_derivative}, respectively.  

\begin{lemma}
\label{lemma_L_1st_derivative}
    Under Assumptions~\ref{assumption_X_compact} and \ref{assumption_surface_zeromeasure}, we have the map $\bm{M}\mapsto \phi(\cdot,\cdot,\bm{M})$ from $\mathcal{X}^k$ into $\mathfrak{L}^2(\widetilde{\mathbb{P}})$ is differentiable in quadratic mean. 
    Moreover, the $L(\bm{M})$ is differentiable with the derivative being $\widetilde{\mathbb{P}}\Delta(\cdot,\cdot,\bm{M})$, where $\Delta(\cdot,\cdot,\bm{M}): \mathcal{X}\times \{0,1\}^p \mapsto \mathbb{R}^{p\times k}$ is the matrix-valued function whose $l$-th column is $\Delta_l (\cdot,\cdot,\bm{M}): \mathcal{X}\times \{0,1\}^p \mapsto \mathbb{R}^{p}$, 
    \begin{align*}
        \Delta_l (\bm{x},\bm{r},\bm{M}) = -2\cdot \mathds{1}\left(  \ell(\bm{x},\bm{r},\bm{M}) = l \right)\cdot ( \bm{x} \circ\bm{r} -\bm{\mu}_l \circ\bm{r} ) .
    \end{align*}
\end{lemma}

\begin{lemma}
\label{lemma_Gn_1st_derivative}
    Let $\bm{\mathrm{V}}_n$ be a sequence of random matrices in $\mathcal{X}^k$ with $\|\textnormal{vec}(\bm{\mathrm{V}}_n - \bm{M}) \|_2=o_P(1)$ for a fixed matrix $\bm{M}\in\mathbb{R}^{p\times k}$. Then, Under Assumptions~\ref{assumption_X_compact} and \ref{assumption_surface_zeromeasure}, we have 
    \begin{align*}
        \mathbb{G}_n\phi(\cdot,\cdot,\bm{\mathrm{V}}_n) = \mathbb{G}_n\phi(\cdot,\cdot,\bm{M}) + \big\langle \textnormal{vec}(\bm{\mathrm{V}}_n - \bm{M}) \;,\; \textnormal{vec}\left( \mathbb{G}_n \Delta(\cdot,\cdot,\bm{M}) \right) \big\rangle + o_P(\|\textnormal{vec}(\bm{\mathrm{V}}_n - \bm{M}) \|_2).
    \end{align*}
\end{lemma}

It should be noted that the first-order differentiability of $\widehat{L}_n(\bm{M})$ essentially does not rely on the assumption of missing mechanisms. The first-order derivative $\widetilde{\mathbb{P}}\Delta(\cdot,\cdot,\bm{M})$ is unique, because the function $\Delta_l(\cdot,\cdot,\bm{M})$ is well-defined and unique even for $(\bm{x},\bm{r})$ having multiple solutions of $\mathop{\arg\min}_{l=1,\dots,k} \|\bm{x} \circ\bm{r} -\bm{\mu}_l \circ\bm{r}\|_2^2 $.

Next, we prove the second-order differentiability of $L(\bm{M})$. 
To this end, for a fixed $\bm{r}\in \{0,1\}^p$ and $l\in \{1,\dots,k\}$, we define the subset $\mathcal{C}_l^{\bm{r}}(\bm{M})=\{\bm{x}\in\mathcal{X}\;|\; \ell(\bm{x},\bm{r},\bm{M})=l \}$.  
Moreover, for any $l\neq l'$, we define $\mathcal{S}_{ll'}^{\bm{r}}(\bm{M}) = \{\bm{x}\in\mathcal{X}\;|\; \| \bm{x}\circ \bm{r} - \bm{\mu}_l \circ\bm{r} \|_2^2 = \| \bm{x}\circ \bm{r} - \bm{\mu}_{l'} \circ\bm{r} \|_2^2 \}$, which is a $(p-1)$-dimensional hyperplane in $\mathcal{X}$ if $\bm{\mu}_l \circ\bm{r}\neq \bm{\mu}_{l'} \circ\bm{r}$, and coincides with $\mathcal{X}$ otherwise.
Also, we need the following Assumptions~\ref{assumption_about_density_f}-\ref{assumption_about_S_ll'_integral}, which are commonly-used technical conditions ensuring the existence of the second-order derivative of $L(\bm{M})$ given in Lemma~\ref{lemma_L_2nd_derivative}. 

\begin{assumption}
\label{assumption_about_density_f} 
     Suppose that $\mathbb{P}$ has a continuous density $f(\cdot)$.
\end{assumption}
\begin{assumption}
\label{assumption_about_S_ll'_integral}
    Given a fixed $\bm{r}\in\{0,1\}^p$, for any $l,l'\in \{1,\dots,k\}$ and fixed vectors $\bm{m},\bm{m'}\in\mathbb{R}^p$, the integral $\int_{\bm{x}\in \mathcal{S}_{ll'}^{\bm{r}}(\bm{M})}  f(\bm{x})\cdot  \left\{(\bm{x} - \bm{m})\circ \bm{r}\right\} \cdot \left\{ (\bm{x}- \bm{m'})\circ \bm{r} \right\}^T  \; dS(\bm{x})$ exists and depends continuously on $\bm{M}$, where $S$ is a measure on the surface of the $\mathbb{R}^p$ space. 
\end{assumption}

\begin{lemma}
\label{lemma_L_2nd_derivative}
    Under Assumption~\ref{assumption_X_compact} and Assumptions~\ref{assumption_surface_zeromeasure}-\ref{assumption_about_S_ll'_integral}, at any $\bm{M}$ satisfying that $\mu_{lj}\neq \mu_{l'j}$ for any $l\neq l'$ and $j=1,\dots,p$,
    the $L(\bm{M})$ is second-order differentiable and has a second-order derivative $\bm{\Gamma}(\bm{M})$ made up of $k^2$ block matrices $\bm{\Gamma}_{ll'}(\bm{M})\in \mathbb{R}^{p\times p}$ ($l,l'=1,\dots,k$) as follows:
    \begin{align*}
        \bm{\Gamma}_{ll}(\bm{M})
        &= \sum_{\bm{r}\in \{0,1\}^p } \text{Pr}(\bm{\mathrm{r}}_1=\bm{r}) \cdot \bigg[  
        2\cdot\text{diag}(\bm{r})\cdot \int \mathds{1}(\ell(\bm{x},\bm{r},\bm{M})=l ) \; d\mathbb{P}(\bm{x}) \\
        &\quad \left. -2 \sum_{t\neq l} \frac{ \mathds{1}(\bm{\mu}_{t}\circ\bm{r} \neq \bm{\mu}_{l}\circ\bm{r}) }{ \|\bm{\mu}_{t}\circ\bm{r} - \bm{\mu}_{l}\circ\bm{r}\|_2 }\cdot 
        \int_{ \bm{x}\in\mathcal{S}_{lt}^{\bm{r}}(\bm{M}) } f(\bm{x})\cdot (\bm{x}\circ\bm{r} - \bm{\mu}_l\circ\bm{r}) \cdot 
        (\bm{x}\circ\bm{r} - \bm{\mu}_l\circ\bm{r} )^T
        \; dS(\bm{x})
        \right] ,
    \end{align*}
    and for $l\neq l'$,
    \begin{align*}
        \bm{\Gamma}_{ll'}(\bm{M})
        &= \sum_{\bm{r}\in \{0,1\}^p } \text{Pr}(\bm{\mathrm{r}}_1=\bm{r}) \cdot \bigg[ 
        0\cdot \bm{I}_p  \\
        &\quad \left. + 2\cdot \frac{ \mathds{1}(\bm{\mu}_{l'}\circ\bm{r} \neq \bm{\mu}_{l}\circ\bm{r}) }{ \|\bm{\mu}_{l'}\circ\bm{r} - \bm{\mu}_{l}\circ\bm{r}\|_2 }\cdot \int_{ \bm{x}\in\mathcal{S}_{ll'}^{\bm{r}}(\bm{M}) } f(\bm{x})\cdot (\bm{x}\circ\bm{r} - \bm{\mu}_l\circ\bm{r}) \cdot 
        (\bm{x}\circ\bm{r} - \bm{\mu}_{l'}\circ\bm{r} )^T
        \; dS(\bm{x})
        \right]. 
    \end{align*}
\end{lemma}

It should be noted that the second-order differentiability of $L(\bm{M})$ depends on the underlying missing mechanism, and Lemma~\ref{lemma_L_2nd_derivative} provides the result specifically for the MCAR mechanism (i.e., Assumption~\ref{assumption_mcar}), under which the missingness is independent with the data distribution. 
Moreover, the requirement of Lemma~\ref{lemma_L_2nd_derivative} that $\bm{M}$ has $k$ distinct cluster centers in each dimension is stricter than that of classical $k$-means on fully observed data, which only requires $\bm{\mu}_l$'s distinct with each other in the $\mathbb{R}^p$ space. 
The technical intuition behind our requirement is as follows. If there exist $l\neq l'$ and $j$ such that $\mu_{lj}= \mu_{l'j}$, then we can choose a missing pattern $\bm{r}\in \{0,1\}^p$ with $r_j=1$ such that $\bm{\mu}_{l'}\circ\bm{r} = \bm{\mu}_{l}\circ\bm{r}$, which causes a singularity that $\bm{\Gamma}(\bm{M})$ is not continuous at such $\bm{M}$. 
More essentially, as shown by \cite{terada2025a}, under MCAR mechanism, the $L(\bm{M})$ can be decomposed into a weighted sum of $2^p$ complete-data-$k$-means loss functions with any possible subset of dimensions. And the second-order differentiability of complete-data-$k$-means in any given subspace holds only if $\bm{\mu}_l$'s are distinct in that specific subspace. 
Consequently, to ensure the second-order differentiability of $L(\bm{M})$ under MCAR mechanism, the stricter requirement of Lemma~\ref{lemma_L_2nd_derivative} is reasonable and essential.

The above lemmas help to establish the quadratic approximation of $\widehat{L}_n(\bm{M})$ at $\bm{M}^{\ast}$ if it satisfies the distinction requirement. 

\begin{assumption}
\label{assumption_M_star_distinct}
    Suppose that the $\bm{M}^{\ast}$ satisfies $\mu_{lj}^{\ast}\neq \mu_{l'j}^{\ast}$ for any $l\neq l'$ and $j=1,\dots,p$. 
\end{assumption}
\begin{theorem}
\label{theorem_quadratic_approximation} 
    Let $\{\bm{\mathrm{V}}_n\}$ be any sequence of random matrices in $\mathcal{X}^k$ with $\|\textnormal{vec}(\bm{\mathrm{V}}_n - \bm{M}^{\ast}) \|_2=o_P(1)$. Then, under Assumptions~\ref{assumption_X_compact}-\ref{assumption_M_star_distinct}, we have
    \begin{align*}
        \widehat{L}_n(\bm{\mathrm{V}}_n) &= \widehat{L}_n(\bm{M}^{\ast}) - n^{-1/2} \bm{\xi}_n^{T}\cdot \textnormal{vec}(\bm{\mathrm{V}}_n - \bm{M}^{\ast}) + \frac{1}{2} \textnormal{vec}(\bm{\mathrm{V}}_n - \bm{M}^{\ast})^{T}\cdot \bm{\Gamma}(\bm{M}^{\ast}) \cdot \textnormal{vec}(\bm{\mathrm{V}}_n - \bm{M}^{\ast})  \notag \\
        &\quad +  o_P(n^{-1/2} \|\textnormal{vec}(\bm{\mathrm{V}}_n - \bm{M}^{\ast})\|_2 ) + o_P(\|\textnormal{vec}(\bm{\mathrm{V}}_n - \bm{M}^{\ast})\|_2^2),
    \end{align*} 
    where $\bm{\xi}_n=-\textnormal{vec}(\mathbb{G}_n \Delta(\cdot,\cdot,\bm{M}^*))$ is a random vector in $\mathbb{R}^{kp}$ and has an asymptotic normal distribution $\mathcal{N}(\bm{0}_{kp},\bm{\Xi})$, and $\bm{\Xi}\in\mathbb{R}^{kp\times kp}$ is a block diagonal matrix with $l$-th block being
    \begin{align*}
        \bm{\Xi}_{l}=4\int \mathds{1}(\ell(\bm{x},\bm{r},\bm{M}^{\ast})=l )\cdot (\bm{x}\circ \bm{r} - \bm{\mu}_l^{\ast}\circ \bm{r})\cdot (\bm{x}\circ \bm{r} - \bm{\mu}_l^{\ast}\circ \bm{r})^T\; d\widetilde{\mathbb{P}}(\bm{x},\bm{r}).
    \end{align*}
\end{theorem}

Finally, applying the Theorem~\ref{theorem_quadratic_approximation} to the sequence $\{ \bm{M}^{\ast} + n^{-1/2} \bm{\Gamma}(\bm{M}^{\ast})\cdot\bm{\xi}_n \}$ would lead to the $\sqrt{n}$-convergence rate and asymptotic normality of $\widehat{\bm{\mathrm{M}}}_n$ under a regular condition. 
\begin{assumption}
\label{assumption_positive_definite_Gamma}
    Suppose that $\bm{\Gamma}(\bm{M})$ is positive definite at $\bm{M}^{\ast}$. 
\end{assumption}

\begin{corollary}
\label{corollary_sqrt_n_rate_asymptotic_normality}
    Under Assumptions~\ref{assumption_X_compact}-\ref{assumption_positive_definite_Gamma}, we have $\|\textnormal{vec}(\widehat{\bm{\mathrm{M}}}_n - \bm{M}^{\ast})\|_2 =O_P(n^{-1/2})$. 
    Moreover, we have $\sqrt{n}\textnormal{vec}(\widehat{\bm{\mathrm{M}}}_n - \bm{M}^{\ast})\xrightarrow{d} \mathcal{N}(\bm{0}_{pk}\;,\;\left\{\bm{\Gamma}(\bm{M}^{\ast}) \right\}^{-1} \cdot\bm{\Xi}\cdot\left\{\bm{\Gamma}(\bm{M}^{\ast}) \right\}^{-1}  )$. 
\end{corollary}

\subsection{When estimated centers based on MCAR data converge to true cluster centers?}
\label{sec_converge_to_truth}

Our third aim is to give conditions under which estimated cluster centers $\widehat{\bm{\mathrm{M}}}_n$ converges to true cluster centers. 
Here, we define the ``\textit{true cluster centers}" by the minimizer of population-level complete-data-$k$-means, that is, 
\begin{align*}
    \bm{M}^{\ast\ast}= \mathop{\arg\min}_{\bm{M}\in\mathcal{X}^k} L^{\ast\ast}(\bm{M}),
\end{align*}
where $L^{\ast\ast}$ is the population-level loss function of $k$-means on fully observed data, that is,
\begin{align*}
    L^{\ast\ast}(\bm{M})=\mathbb{E}_{\bm{\mathrm{x}}_1}\left[ \min_{l=1,\dots,k} \| \bm{\mathrm{x}}_1 - \bm{\mu}_l \|_2^2 \right]. 
\end{align*}

For our purpose in this section, first, we let the constant $\rho^{\ast\ast}\geq 0$ be the minimal separation of true cluster centers in each dimension, that is, 
\begin{align*}
    \rho^{\ast\ast} = \min_{j=1,\dots,p}\min_{l\neq l'} \left|  \mu_{lj}^{\ast\ast} - \mu_{l'j}^{\ast\ast} \right|.
\end{align*}
We also denote by $\mathrm{b}_n$ the maximum within-cluster radius, that is, 
\begin{align*}
    \mathrm{b}_n = \max_{i=1,\dots,n} \min_{l=1,\dots,k} \left\| \bm{\mathrm{x}}_i - \bm{\mu}_{l}^{\ast\ast} \right\|_2. 
\end{align*}
In addition, we define 
\begin{align*}
    &\mathrm{n}_{\min}^{\textnormal{comp}}= \min_{l=1,\dots,k} \sum_{i=1}^{n} \mathds{1}\left( l=\mathop{\arg\min}_{l=1,\dots,k} \| \bm{\mathrm{x}}_i - \bm{\mu}_l^{\ast\ast} \|_2 \;,\; \bm{\mathrm{r}}_i=\bm{1}_p \right) \\
    &\mathrm{n}_{\min}^{\textnormal{feature}} = \min_{\substack{l=1,\dots,k\\ j=1,\dots,p  }} \sum_{i=1}^{n} \mathds{1}\left( l=\mathop{\arg\min}_{l=1,\dots,k} \| \bm{\mathrm{x}}_i - \bm{\mu}_l^{\ast\ast} \|_2 \;,\; \mathrm{r}_{ij}=1  \right),
\end{align*}
It should be noted that $\mathrm{n}_{\min}^{\textnormal{comp}}$ is the minimum number of complete data points (i.e., $\bm{\mathrm{x}}_i$ has no missing value) of each cluster, while the $\mathrm{n}_{\min}^{\textnormal{feature}}$ is the minimum number of observed data points across each clusters and each dimensions, thus $\mathrm{n}_{\min}^{\textnormal{feature}}\geq \mathrm{n}_{\min}^{\textnormal{comp}}$ and the equivalence holds only when there is no missingness. 

Then, we give Lemma~\ref{lemma_converge_to_truth_condition_centers_bound} about a non-asymptotic bound for $D(\widehat{\bm{\mathrm{M}}}_n,\bm{M}^{\ast\ast})$. Here, the function $D:\mathbb{R}^{p\times k}\times \mathbb{R}^{p\times k} \mapsto \mathbb{R}$ is defined by 
\begin{align*}
    D(\bm{M},\bm{V})=\min_{\pi} \max_{l=1,\dots,k} \| \bm{\mu}_{\pi(l)} - \bm{v}_l \|_2,
\end{align*}
where $\pi$ is a bijection from $\{1,\dots,k\}$ to $\{1,\dots,k\}$, that is, a permutation of $\{1,\dots,k\}$. 

\begin{lemma}
\label{lemma_converge_to_truth_condition_centers_bound}
    Under Assumption~\ref{assumption_unique_minimizer} and suppose a fixed $n$ large enough such that the set $\{i\;|\; l=\mathop{\arg\min}_{l=1,\dots,k} \| \bm{\mathrm{x}}_i - \bm{\mu}_l^{\ast\ast} \|_2,\; \bm{r}=\bm{\mathrm{r}}_i\}\neq \emptyset$ for any $l=1,\dots,k$ and non-zero vector $\bm{r}\in\{0,1\}^p$. 
    If $\rho^{\ast\ast}> 4\mathrm{b}_n \sqrt{n/\mathrm{n}_{\min}^{\textnormal{feature}}}$ holds, then we have $D(\widehat{\bm{\mathrm{M}}}_n,\bm{M}^{\ast\ast}) \leq 2\mathrm{b}_n \sqrt{n/\mathrm{n}_{\min}^{\textnormal{comp}}}$. 
\end{lemma}
The next Lemma~\ref{lemma_converge_to_truth_perfect_label} provides a sufficient condition under which the cluster assignment given by $\widehat{\bm{\mathrm{M}}}_n$ coincides with that given by $\bm{M}^{\ast\ast}$ almost everywhere. 
\begin{lemma}
\label{lemma_converge_to_truth_perfect_label}
    Under Assumption~\ref{assumption_unique_minimizer} and suppose a fixed $n$ large enough such that the set $\{i\;|\; l=\mathop{\arg\min}_{l=1,\dots,k} \| \bm{\mathrm{x}}_i - \bm{\mu}_l^{\ast\ast} \|_2,\; \bm{r}=\bm{\mathrm{r}}_i\}\neq \emptyset$ for any $l=1,\dots,k$ and non-zero vector $\bm{r}\in\{0,1\}^p$. 
    If $\rho^{\ast\ast} > 4\mathrm{b}_n\sqrt{n/\mathrm{n}_{\min}^{\textnormal{comp}}} + 2\mathrm{b}_n $ holds, then there must exist a permutation $\pi$ such that 
    \begin{align*}
        \mathop{\arg\min}\limits_{t=1,\dots,k} \| \bm{\mathrm{x}}_i\circ\bm{\mathrm{r}}_i - \hat{\bm{\mu}}_t \circ\bm{\mathrm{r}}_i \|_2 = \pi\left( \mathop{\arg\min}\limits_{l=1,\dots,k} \| \bm{\mathrm{x}}_i - \bm{\mu}_l^{\ast\ast} \|_2 \right)
    \end{align*}
    holds for all $i=1,\dots,n$ whenever both $\mathop{\arg\min}$'s are unique.
\end{lemma}

The following theorem shows that if the true cluster centers are well-separated in each dimension, then $\widehat{\bm{\mathrm{M}}}_n$ will converge to $\bm{M}^{\ast\ast}$ under MCAR mechanism. 
\begin{assumption}
\label{assumption_max_within_cluster_distance}
    Suppose that there exists a constant $\beta^{\ast\ast}>0$ such that $\mathrm{b}_n \leq \beta^{\ast\ast}$ almost surely. 
\end{assumption}
\begin{theorem}
\label{theorem_converge_to_truth}
    Under Assumptions~\ref{assumption_unique_minimizer}-\ref{assumption_mcar} and \ref{assumption_max_within_cluster_distance}, 
    if $\lim_{n\rightarrow \infty} \textnormal{Pr}\left( \rho^{\ast\ast}/\beta^{\ast\ast} > 4\sqrt{n/\mathrm{n}_{\min}^{\textnormal{comp}}} + 2  \right) = 1$, then we have for any $\epsilon>0$, it holds that $\lim_{n\rightarrow\infty}  \textnormal{Pr}\left( D(\widehat{\bm{\mathrm{M}}}_n,\bm{M}^{\ast\ast}) > \epsilon \right) =0$.
\end{theorem}

There are several remarks that should be noted. 
First, the separation condition of true cluster centers given in Theorem~\ref{theorem_converge_to_truth} ensures not only the convergence of $\widehat{\bm{\mathrm{M}}}_n$ to $\bm{M}^{\ast\ast}$, but also the exact recovery of cluster labels with high probability. 
Second, since $\beta^{\ast\ast}$ reflects the upper bound of maximum within-cluster deviation, then the $\rho^{\ast\ast}/\beta^{\ast\ast}$ roughly serves as signal-noise-ratio of the distribution of $\bm{\mathrm{x}}_i$. Moreover, since Assumption~\ref{assumption_max_within_cluster_distance} implies $\bm{\mathrm{x}}_i\in \bigcup_{l=1}^{k} \mathcal{B}(\bm{\mu}_l^{\ast\ast},\beta^{\ast\ast})$ for any $i=1,\dots,k$, based on which Assumption~\ref{assumption_X_compact} holds. Thus, we here only require Assumption~\ref{assumption_max_within_cluster_distance}. 
Third, since $\mathrm{n}_{\min}^{\textnormal{comp}}$ is the minimum number of complete data points of each cluster, and missingness is independent with $\bm{\mathrm{x}}_i$'s distribution under MCAR mechanism (i.e., Assumption~\ref{assumption_mcar}), then $\mathrm{n}_{\min}^{\textnormal{comp}}/n$ is equivalent to the frequency of complete data points in the minimum cluster. 
Therefore, to ensure the convergence to true cluster centers, the required minimum separation $\rho^{\ast\ast}/\beta^{\ast\ast}$ is determined by the minimum cluster proportion and the probability of complete data points appearing. 

\begin{example}
    For $p\geq 2$ and $k\geq 2$, suppose each $\bm{\mathrm{x}}_i$ belongs to each cluster with equal probability $1/k$, and consider MCAR mechanism with $\textnormal{Pr}(\mathrm{x}_{ij} \textnormal{ is observed})=q$ ($0<q \leq 1$), then we have $\mathrm{n}_{\min}^{\textnormal{comp}}/n\xrightarrow{P} q^p/k$. It follows that the required separation condition should be $\rho^{\ast\ast}/\beta^{\ast\ast} > 4\sqrt{k/q^p}+2$. Therefore, the separation condition is easier to satisfy when observed probability $q$ is larger, the number of clusters $k$ is smaller, or the dimension of data $p$ is smaller. 
    For example, consider $p=2$, $k=2$ and true cluster centers given by $\bm{\mu}_1^{\ast\ast}=(0,0)^T$ and $\bm{\mu}_2^{\ast\ast}=(\rho^{\ast\ast},\rho^{\ast\ast})^T$. If $q=1/2$, then $\mathrm{n}_{\min}^{\textnormal{comp}}/n\xrightarrow{P} 1/8$. It follows that $4\sqrt{n/\mathrm{n}_{\min}^{\textnormal{comp}}} + 2 \xrightarrow{P} 8\sqrt{2}+2$. Thus, the separation condition should be $\rho^{\ast\ast}/\beta^{\ast\ast} > 8\sqrt{2}+2$. 
    In addition, if we consider the no-missingness case (i.e., $q=1$), because $\mathrm{n}_{\min}^{\textnormal{comp}}/n\xrightarrow{P} 1/2$, then $4\sqrt{n/\mathrm{n}_{\min}^{\textnormal{comp}}} + 2 \xrightarrow{P} 4\sqrt{2}+2$. Thus, the separation condition for full observed data is $\rho^{\ast\ast}/\beta^{\ast\ast}>4\sqrt{2}+2$. 
    In comparison, the required separation condition of $\rho^{\ast\ast}/\beta^{\ast\ast}$ under 50\% missingness is about 1.74 times higher than that required for fully observed data, implying the increased difficulty in recovering true cluster structure from incomplete data. 
\end{example}

Furthermore, the separation condition given in Theorem~\ref{theorem_converge_to_truth} highlights an important fact that if the true cluster centers $\bm{\mu}_l^{\ast\ast}$'s are overlapped in some dimension, which means $\rho^{\ast\ast}=0$, then the separation condition would never be satisfied. 
However, such overlapped cluster centers are very common in the high-dimensional data, where there exist many noise dimensions in which we suppose true cluster centers are the same. As a consequence, our result actually provide an explanation for the failure of missing-data-$k$-means in high-dimensional data, and further clarifies the scope of application of missing-data-$k$-means.

\section{Numerical experiments}

\subsection{Asymptotic behavior of estimated cluster centers}

In this section, we verify the theoretical results of Section~\ref{sec_finite_result}-\ref{sec_convergence_rate} via numerical experiments on synthetic incomplete datasets, which are constructed by artificially setting missing on original complete datasets. 

To construct such a missing dataset, we first generate a dataset containing $n$ fully observed data points $\bm{x}_i$ ($i=1,\dots,n$), which consists of $k=3$ clusters in $\mathbb{R}^2$. For any $l=1,\dots,k$, the probability of each $\bm{x}_i$ belonging to $l$-th cluster is equal to $1/k$, and all $\bm{x}_i$'s in $l$-th cluster are drawn independently from a truncated normal distribution that is derived from a base Gaussian distribution $\mathcal{N}(\bm{\mu}_l^{\ast\ast},\bm{I}_2)$  with each dimension bounded within the interval $[-10,10]$. Moreover, we consider $\bm{\mu}_l^{\ast\ast}$'s as follows: 
\begin{align*}
    \bm{\mu}_1^{\ast\ast}=\left(-\frac{\sqrt{6}}{2}, -\frac{\sqrt{6}}{2}\right)^T,\quad
    \bm{\mu}_2^{\ast\ast}=\left( \frac{\sqrt{6}+3\sqrt{2} }{4} , \frac{\sqrt{6}-3\sqrt{2} }{4} \right)^T,\quad
    \bm{\mu}_3^{\ast\ast}=\left( \frac{\sqrt{6}-3\sqrt{2} }{4} , \frac{\sqrt{6}+3\sqrt{2} }{4} \right)^T.
\end{align*}
Secondly, for each $x_{ij}$ ($i=1,\dots,n$, $j=1,2$), we artificially let it be missing as follows: 
\begin{itemize}
    \item MCAR: The missing probability is set to be a constant. For any $i=1,\dots,n$ and $j=1,\dots,p$, \begin{align*}
        \text{Pr}(x_{ij}\text{ is missing}) = \tau.
    \end{align*}
    \item MNAR: The missing probability is determined by the value of the data itself. For any $i=1,\dots,n$ and $j=1,\dots,p$,
    \begin{align*}
        \text{Pr}(x_{ij}\text{ is missing}) = \exp (-\lambda x_{ij}^2 ) .
    \end{align*}
\end{itemize}
Different $\tau\in (0,1)$ and $\lambda>0$ are used to meet the total missing proportion from 10\% to 50\%. 

To verify the theoretical results, the population-level optimal cluster centers $\bm{M}^{\ast}$ and the minimal value of $L(\bm{M})$ (denoted by $m^{\ast}$) are needed but often unknown. Thus, we generate a large sample with 30000 data points and use the estimated cluster centers as a substitute of $\bm{M}^{\ast}$ and the corresponding value of loss function as a substitute of $m^{\ast}$. 
Moreover, we denote by $\widehat{\bm{\mathrm{M}}}_{n,t}$ the estimated cluster centers obtained by conducting missing-data-$k$-means clustering on the incomplete dataset (with $n$ data points) in the $t$-th repetition, where $t=1,\dots,T$. Correspondingly, we denote by $\hat{m}_{n,t}$ the minimal value of loss function $\widehat{L}_n(\bm{M})$ in the $t$-th repetition. 

\subsubsection{Excess risk bound}

To verify the $O_p(n^{-1/2})$ excess risk bound, we consider $n\in \{300, 900, 1500, 3000, 6000, 9000, 12000, 15000\}$ and the number of repetitions $T=100$. We calculate the average of $\{\hat{m}_{n,1},\dots,\hat{m}_{n,T}\}$, denoted by $\hat{m}_n$, and then check the convergence of $\hat{m}_n$ to $m^{\ast}$. 
To this end, Figure~\ref{fig_excess_risk} illustrates how the loss gap $|\hat{m}_n - m^{\ast}|$ varies with different $n$ on a log-log scale. The red dashed line is the line with slope being -0.5, implying the theoretical trend. We can see that the empirical trend of $|\hat{m}_n - m^{\ast}|$ (black solid line) closely aligns with the theoretical trend (the red dashed line), which confirms the $\sqrt{n}$-convergence rate of $\hat{m}_n$ to $m^{\ast}$. Consequently, the $O_p(n^{-1/2})$ excess risk bound is verified. 

\begin{figure}[t]
\captionsetup[subfigure]{justification=centering}
    \centering
    \begin{subfigure}{0.23\textwidth}
        \includegraphics[width=\textwidth]{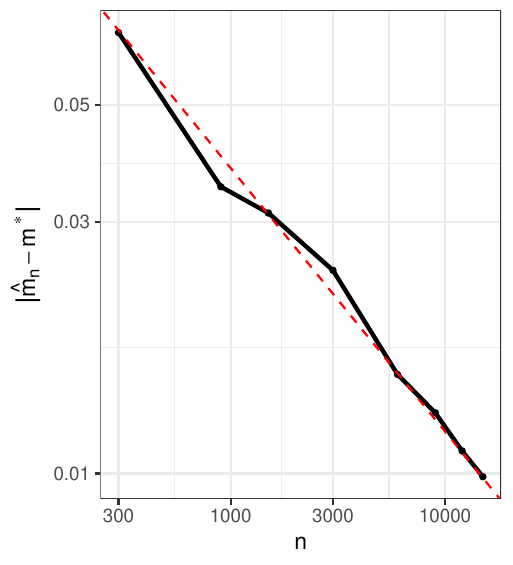}
        \caption*{MCAR, 10\%}
    \end{subfigure}
    \begin{subfigure}{0.23\textwidth}
        \includegraphics[width=\textwidth]{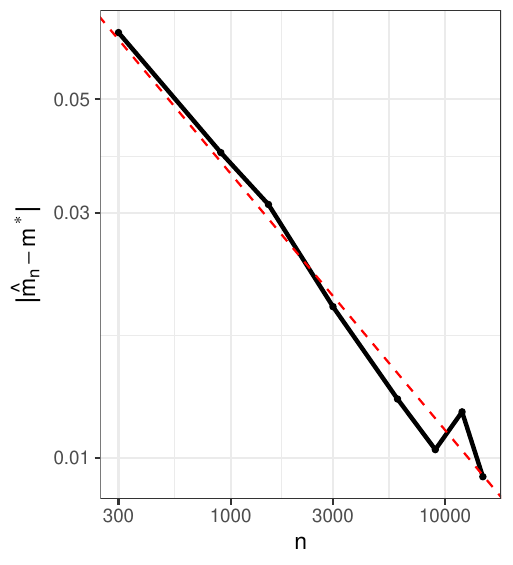}
        \caption*{MCAR, 30\%}
    \end{subfigure}
    \begin{subfigure}{0.23\textwidth}
        \includegraphics[width=\textwidth]{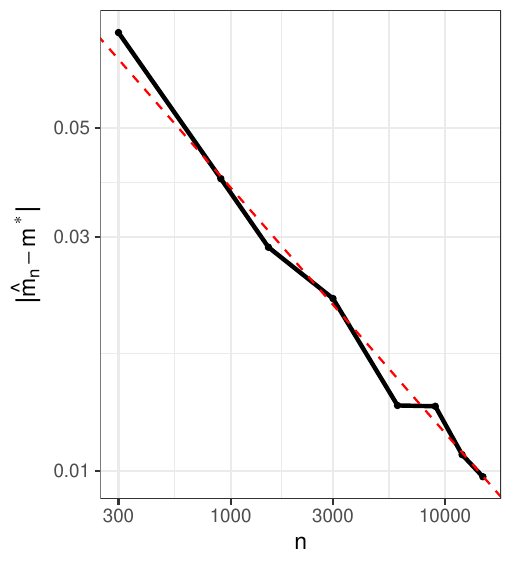}
        \caption*{MCAR, 50\%}
    \end{subfigure}
    \par
    \begin{subfigure}{0.23\textwidth}
        \includegraphics[width=\textwidth]{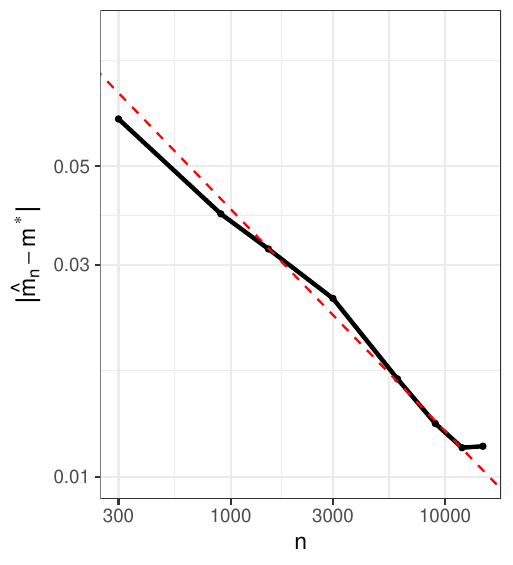}
        \caption*{MNAR, 10\%}
    \end{subfigure}
    \begin{subfigure}{0.23\textwidth}
        \includegraphics[width=\textwidth]{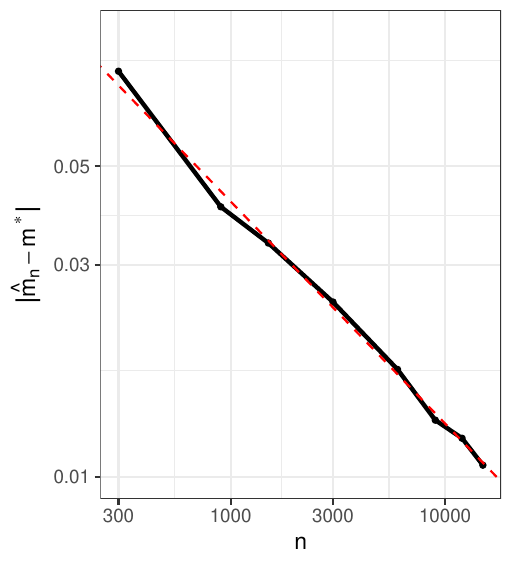}
        \caption*{MNAR, 30\%}
    \end{subfigure}
    \begin{subfigure}{0.23\textwidth}
        \includegraphics[width=\textwidth]{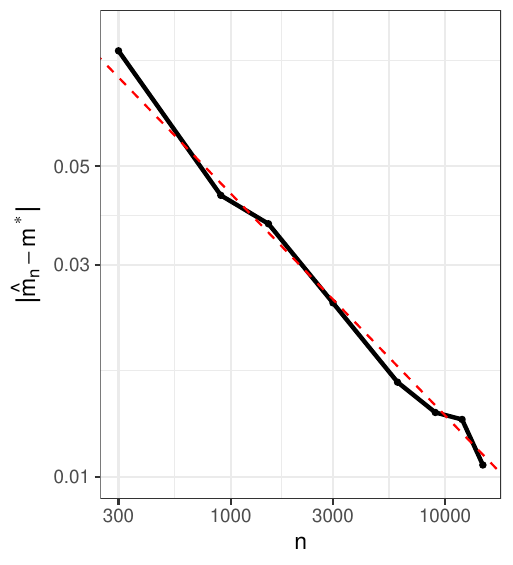}
        \caption*{MNAR, 50\%}
    \end{subfigure}
    \caption{The convergence of $\hat{m}_n$ to $m^{\ast}$ (log-log scale) under different missing mechanisms and missing proportions. The red dashed line is the line with slope being -0.5, implying the theoretical trend.}
    \label{fig_excess_risk}
\end{figure}

\subsubsection{Consistency}
For the consistency of estimated cluster centers, we consider $n\in \{300, 900, 1500, 3000, 6000, 9000, 12000, 15000\}$ and the number of repetitions $T=100$. 
Figure~\ref{fig_consistency} illustrates how the distance between $\widehat{\bm{\mathrm{M}}}_{n,t}$ and $\bm{M}^{\ast}$ varies with $n$. For each $n$, we report the average and error bar of $\{ \|\widehat{\bm{\mathrm{M}}}_{n,t} - \bm{M}^{\ast}\|_F^2\}_{t=1}^{T}$, which shows a decreasing trend to zero for each setting. Consequently, the consistency of estimated cluster centers is verified. 

\begin{figure}[htbp]
\captionsetup[subfigure]{justification=centering}
    \centering
    \begin{subfigure}{0.23\textwidth}
        \includegraphics[width=\textwidth]{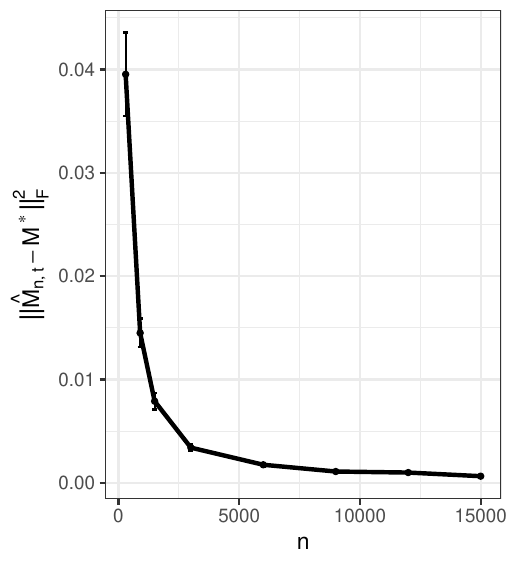}
        \caption*{MCAR, 10\%}
    \end{subfigure}
    \begin{subfigure}{0.23\textwidth}
        \includegraphics[width=\textwidth]{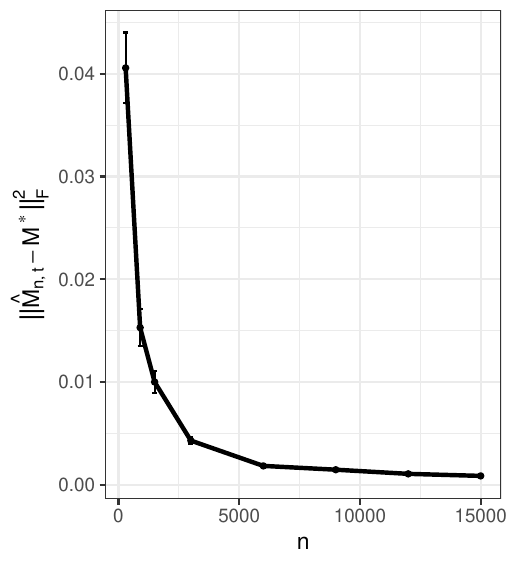}
        \caption*{MCAR, 30\%}
    \end{subfigure}
    \begin{subfigure}{0.23\textwidth}
        \includegraphics[width=\textwidth]{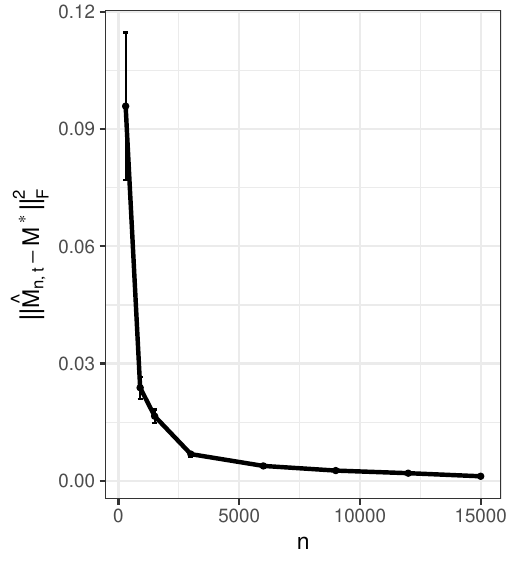}
        \caption*{MCAR, 50\%}
    \end{subfigure}
    \par
    \begin{subfigure}{0.23\textwidth}
        \includegraphics[width=\textwidth]{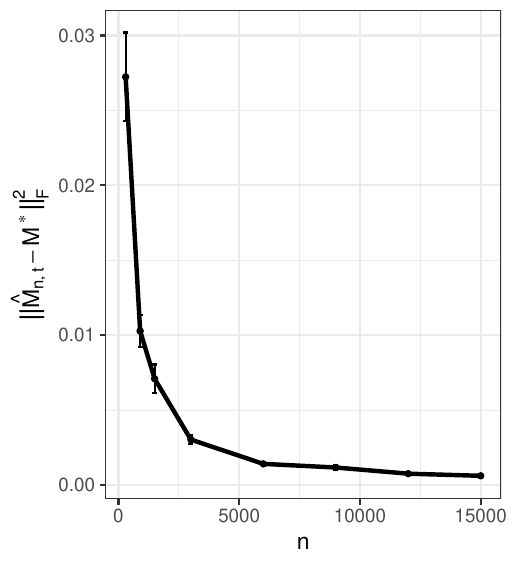}
        \caption*{MNAR, 10\%}
    \end{subfigure}
    \begin{subfigure}{0.23\textwidth}
        \includegraphics[width=\textwidth]{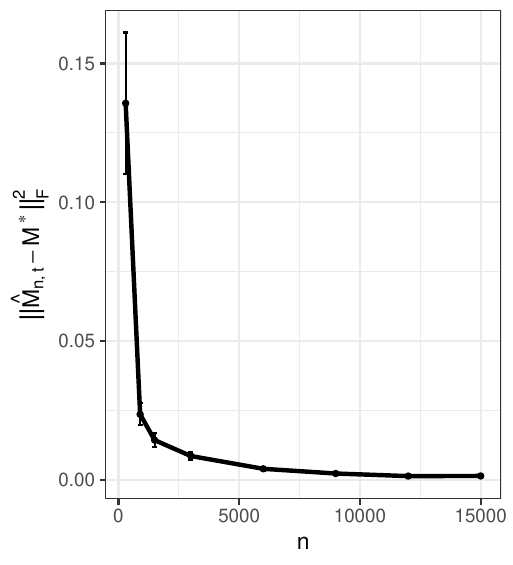}
        \caption*{MNAR, 30\%}
    \end{subfigure}
    \begin{subfigure}{0.23\textwidth}
        \includegraphics[width=\textwidth]{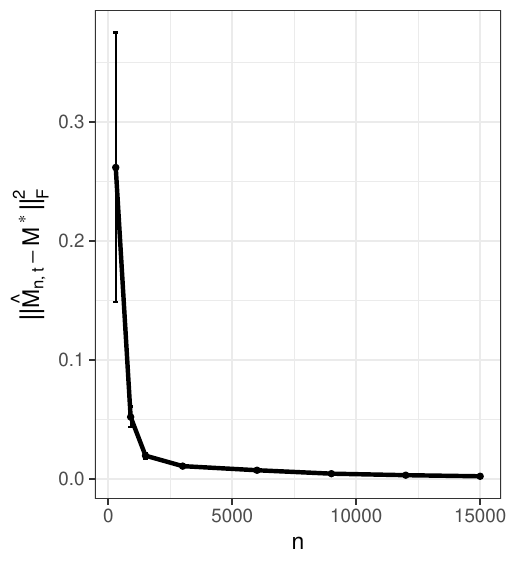}
        \caption*{MNAR, 50\%}
    \end{subfigure}
    \caption{The consistency of estimated cluster centers under different missing mechanisms and missing proportions. The y-axis is the $\|\widehat{\bm{\mathrm{M}}}_{n,t} - \bm{M}^{\ast}\|_F^2$. For each $n$, the average value and error bar of $T=100$ repetitions are reported. }
    \label{fig_consistency}
\end{figure}

\subsubsection{Convergence rate}
Figure~\ref{fig_convergence_rate} illustrates the convergence trend of $\widehat{\bm{\mathrm{M}}}_{n,t}$ to $\bm{M}^{\ast}$ on a log-log scale. For each $n$, the averaged value of $\|\widehat{\bm{\mathrm{M}}}_{n,t} - \bm{M}^{\ast}\|_F^2$ among $T=100$ repetitions are reported. The red dashed line is the line with slope being -1, implying the theoretical trend. 
We can see that the trend of empirical gap between estimated cluster centers and $\bm{M}^{\ast}$ (black solid line) closely aligns with the theoretical trend (the red dashed line), which confirms the $\sqrt{n}$-convergence rate of the estimated cluster centers. 
Moreover, not only MCAR, we can also see the alignment of trend under MNAR, which means that the $\sqrt{n}$-convergence rate may also hold under MNAR mechanism, even though our theoretical result about the $\sqrt{n}$-convergence rate is based on the MCAR mechanism. 

\begin{figure}[htbp]
\captionsetup[subfigure]{justification=centering}
    \centering
    \begin{subfigure}{0.23\textwidth}
        \includegraphics[width=\textwidth]{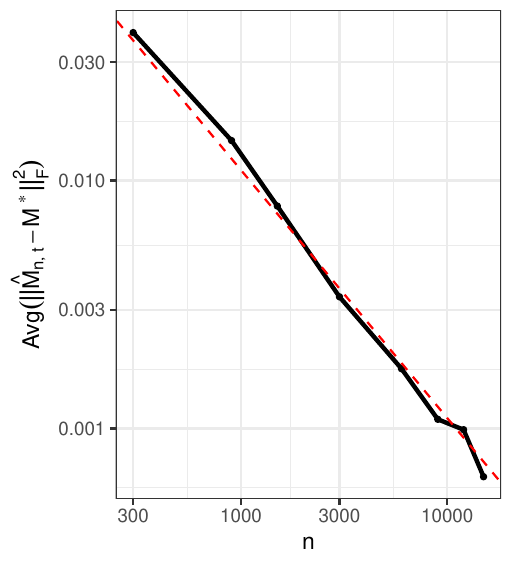}
        \caption*{MCAR, 10\%}
    \end{subfigure}
    \begin{subfigure}{0.23\textwidth}
        \includegraphics[width=\textwidth]{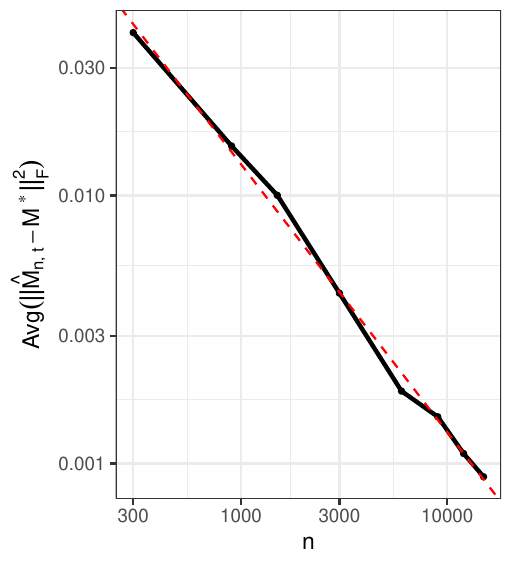}
        \caption*{MCAR, 30\%}
    \end{subfigure}
    \begin{subfigure}{0.23\textwidth}
        \includegraphics[width=\textwidth]{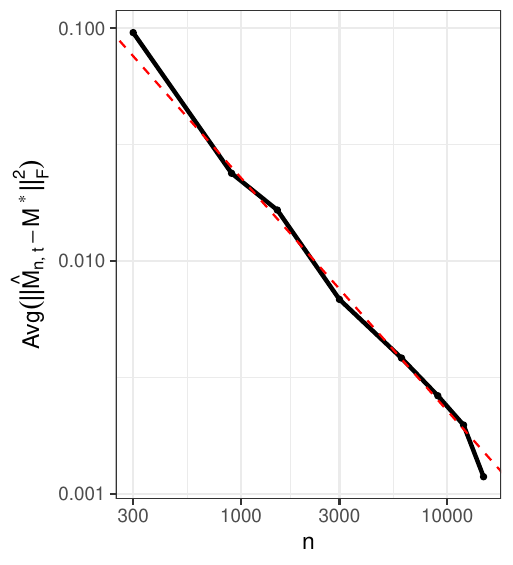}
        \caption*{MCAR, 50\%}
    \end{subfigure}
    \par
    \begin{subfigure}{0.23\textwidth}
        \includegraphics[width=\textwidth]{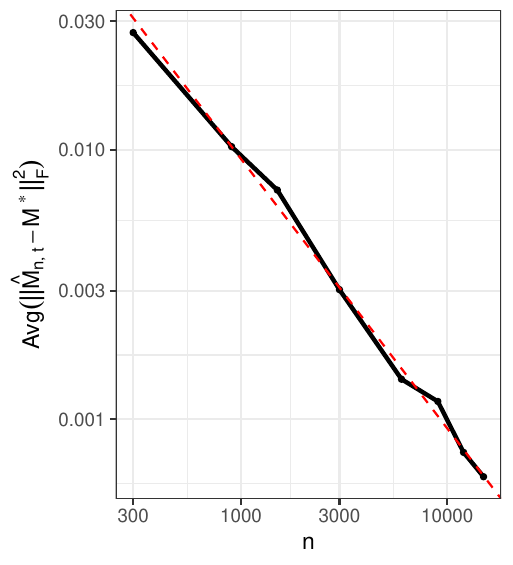}
        \caption*{MNAR, 10\%}
    \end{subfigure}
    \begin{subfigure}{0.23\textwidth}
        \includegraphics[width=\textwidth]{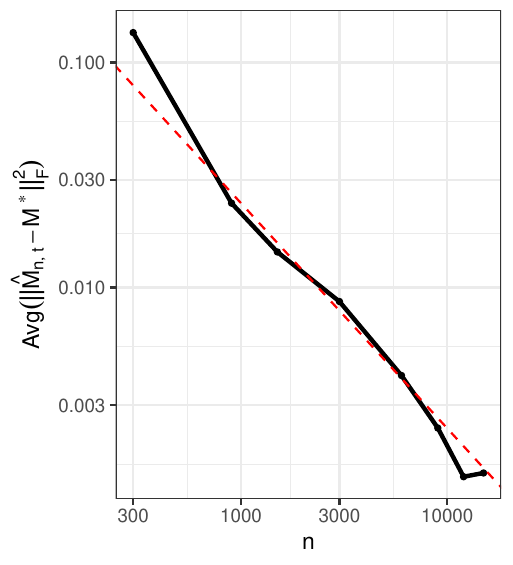}
        \caption*{MNAR, 30\%}
    \end{subfigure}
    \begin{subfigure}{0.23\textwidth}
        \includegraphics[width=\textwidth]{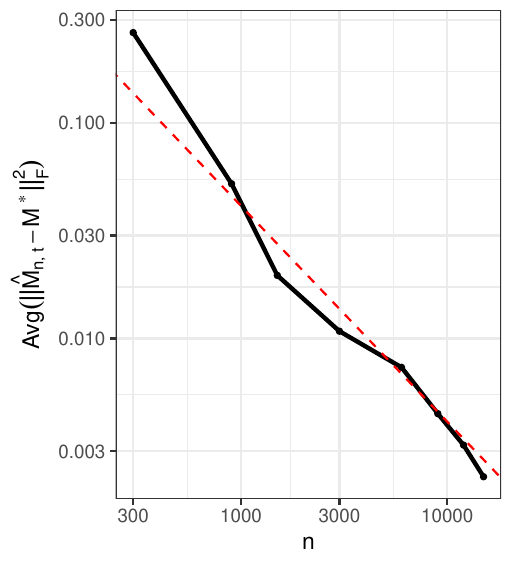}
        \caption*{MNAR, 50\%}
    \end{subfigure}
    \caption{The log-log scale convergence trend of estimated cluster centers under different missing mechanisms and missing proportions. The y-axis is the averaged value of $\|\widehat{\bm{\mathrm{M}}}_{n,t} - \bm{M}^{\ast}\|_F^2$ among $T=100$ repetitions for each $n$. The red dashed line is the line with slope being -1, implying the theoretical trend. }
    \label{fig_convergence_rate}
\end{figure}

\subsubsection{Asymptotic normality}
To verify the asymptotic normality, we consider $n=6000$ and the number of repetitions $T=1000$, and check whether the empirical distribution of $\{\text{vec}(\widehat{\bm{\mathrm{M}}}_{n,t})\}_{t=1}^{T}$ is a normal distribution. 
To this end, since $\text{vec}(\widehat{\bm{\mathrm{M}}}_{n,t})$ is a vector with length $kp$, we conduct the energy test on the $\{\text{vec}(\widehat{\bm{\mathrm{M}}}_{n,t})\}_{t=1}^{T}$, which is a non-parametric testing method specially for the multivariate normal distribution. 
Moreover, due to the fact that if $\text{vec}(\widehat{\bm{\mathrm{M}}}_{n,t})$ follows a multivariate normal distribution, then each component of it also follows a normal distribution, thus, we conduct the Shapiro test on each component of $\text{vec}(\widehat{\bm{\mathrm{M}}}_{n,t})$. 
Table~\ref{table_asymptotic_normality} illustrates the results of testing under different missing mechanisms and missing proportions, where the p-values of normality test on $\text{vec}(\widehat{\bm{\mathrm{M}}}_{n,t})$ and its each component are reported. 
It can be seen that under MCAR mechanism, almost all p-values are larger than 0.05, which confirms the asymptotic normality of $\text{vec}(\widehat{\bm{\mathrm{M}}}_{n,t})$ and each component. 
In contrast, under MNAR mechanism, the p-values are smaller. Especially for 50\% missingness, the p-value of test on $\text{vec}(\widehat{\bm{\mathrm{M}}}_{n,t})$ is less than 0.001, implying that the asymptotic normality of $\text{vec}(\widehat{\bm{\mathrm{M}}}_{n,t})$ does not hold. 
Moreover, under MNAR mechanism with 30\% and 50\% missing proportion, the p-values of the $(2,1)$ and $(3,2)$ components of $\widehat{\bm{\mathrm{M}}}_{n,t}$ are less than 0.01 and 0.001, respectively, implying that the asymptotic normality of the two components do not hold. This is related to MNAR mechanism. 
In fact, our MNAR mechanism gives a higher missing probability if the absolute value of original fully observed data is close to zero. Since the values of $\mu^{\ast}_{21}$ and $\mu^{\ast}_{32}$ are more close to zero than that of others, then more missingness occur for data points belonging to cluster 2 in dimension 1 and data points belonging to cluster 3 in dimension 2. Consequently, the asymptotic normality of estimators of $\mu^{\ast}_{21}$ and $\mu^{\ast}_{32}$ can be damaged, causing a negative influence to the asymptotic normality of $\text{vec}(\widehat{\bm{\mathrm{M}}}_{n,t})$.

\begin{table*}[htbp]
    \centering
    \caption{The p-values of normality testing on $\text{vec}(\widehat{\bm{\mathrm{M}}}_{n,t})$ and its each component}
    \label{table_asymptotic_normality}
    \begin{tabular}{lclllllll}
    \toprule
    \multirow{2}{*}{\makecell{Missing\\Mechanism}} & 
    \multirow{2}{*}{\makecell{Missing\\Proportion}} & 
    \multirow{2}{*}{\makecell{Energy test\\on $\text{vec}(\widehat{\bm{\mathrm{M}}}_{n,t})$}} & \multicolumn{6}{c}{Shapiro test of each component}\\
    \cmidrule(lr){4-9}
    & & & (1,1) & (1,2) & (2,1) & (2,2) & (3,1) & (3,2) \\
    \midrule
    MCAR & 10\% & 0.518 & 0.513 & 0.188 & 0.162 & 0.633 & 0.434 & 0.294 \\
     & 30\% & 0.799 & 0.748 & 0.545 & 0.525 & 0.163 & 0.609 & 0.021* \\
     & 50\% & 0.724 & 0.138 & 0.568 & 0.368 & 0.055 & 0.906 & 0.136 \\
    MNAR & 10\% & 0.367 & 0.207 & 0.890 & 0.804 & 0.914 & 0.786 & 0.138\\
     & 30\% & 0.256 & 0.044* & 0.908 & 0.005** & 0.380 & 0.208 & 0.003** \\
     & 50\% & 0.000*** & 0.420 & 0.897 & 0.000*** & 0.145 & 0.791 & 0.000*** \\
    \bottomrule
    \multicolumn{9}{l}{***: p-value<0.001  **: p-value<0.01  *: p-value<0.05}
    \end{tabular}
\end{table*}

\vspace{-10pt}
\subsection{Convergence to true cluster enters}

In this section, we verify the convergence of estimated cluster centers to the true cluster centers $\bm{M}^{\ast\ast}$ given in Section~\ref{sec_converge_to_truth} via numerical experiments on synthetic incomplete datasets. 
For this purpose, we conduct the missing-data-$k$-means on the incomplete dataset to obtain the estimated cluster centers, then we calculate the Mean-Squared Error (MSE) to evaluate the bias of estimated cluster centers. Specifically, $\text{MSE}(\widehat{\bm{\mathrm{M}}}_n,\bm{M}^{\ast\ast})=\sum_{l=1}^{k}\min_{l'=1,\dots,k} \|\hat{\bm{\mu}}_l - \bm{\mu}^{\ast\ast}_{l'}\|_2^2$, where $\widehat{\bm{\mathrm{M}}}_n$ is the estimated cluster centers, and $\bm{M}^{\ast\ast}=((\bm{\mu}_1^{\ast\ast})^T,\dots,(\bm{\mu}_k^{\ast\ast})^T)^T$ is the true cluster centers.

The construction of missing datasets is the same as the above experiments, while we fix $n=10000$, $p=2$ and $k=2$, and only consider the MCAR mechanism that $\textnormal{Pr}(x_{ij} \textnormal{is missing})=1-q$ for any $i=1,\dots,n,j=1,\dots,p$. 
Moreover, we consider two cases of true cluster centers: (a) $\bm{\mu}_1^{\ast\ast}=(0,0)^T$, $\bm{\mu}_2^{\ast\ast}=(\rho,\rho)^T$; (b) $\bm{\mu}_1^{\ast\ast}=(0,0)^T$, $\bm{\mu}_2^{\ast\ast}=(\rho,0)^T$, where the $\rho^{\ast\ast}=\rho$ for the case (a) and $\rho^{\ast\ast}=0$ for the case (b). 
We vary $\rho\in\{1,2,\dots,10\}$ for different separation degrees between true cluster centers, while let the within-cluster distance $\|\bm{x}_i -\bm{\mu}_l^{\ast\ast} \|_2$ bounded by a fixed $\beta^{\ast\ast}=3$. 

Figure~\ref{fig_converge_to_truth} shows how MSE varies with the parameter $\rho$ of true cluster centers, where the black solid line is the result of $k$-means on original fully observed dataset, and the blue dashed line is for missing-data-$k$-means on datasets with different missing proportions. 
 
For the case (a), it can be seen that in all settings, the MSE value of missing-data-$k$-means becomes almost zero when $\rho$ is relatively large, which means that the estimated cluster centers by incomplete data would converge to the true centers when true cluster centers are well-separated in each feature. 
Moreover, in each penal, the red point represents the empirical transition point (the minimum $\rho$ required to achieve $\textnormal{MSE}\approx 0$), and the vertical dotted line represents the approximated theoretical threshold\footnote{
Here, it should be clarified that the approximated theoretical threshold is calculated as follows: First, since original fully observed data $\bm{x}_i$'s of each cluster are generated from a Gaussian distribution, then $\|\bm{x}_i -\bm{\mu}_l^{\ast\ast} \|_2^2$ is bounded by the within-cluster variance (equal to 1) with high probability. Thus, the separation condition in Theorem~\ref{theorem_converge_to_truth} can be approximated by $\rho^{\ast\ast} > 4\sqrt{n/\mathrm{n}_{\min}^{\textnormal{comp}}} +2$ with high probability. Second, because $4\sqrt{n/\mathrm{n}_{\min}^{\textnormal{comp}}} +2\xrightarrow{P} 4\sqrt{k/q^p}+2$, then we calculate the limit value $4\sqrt{k/q^p}+2$ for each setting. Third, for the simplification of visualization, we illustrate the half value $2\sqrt{k/q^p}+1$ in Figure~\ref{fig_converge_to_truth} by the vertical dotted line, as a smaller approximation of theoretical threshold. 
}. 
We can see that although the empirical threshold becomes larger as missing proportion is larger, the vertical dotted line is always in the right of the red point, that is, the theoretical threshold is always larger than the empirical threshold. Consequently, any $\rho$ larger than the theoretical threshold (i.e., $\rho$ in the right of vertical line) ensures $\textnormal{MSE}\approx 0$. This result verifies that the proposed theoretical threshold gives a sufficient guarantee for the convergence to true cluster centers.

For the case (b), however, it can be seen that in all settings, the MSE value of missing-data-$k$-means is not close to zero even when $\rho=10$ (i.e., $\bm{\mu}_1^{\ast\ast}$ and $\bm{\mu}_2^{\ast\ast}$ are far away in $\mathbb{R}^p$ space), which means that the estimated cluster centers by incomplete data would not converge to the true centers, even though the estimators by original fully observed data indeed converge to true centers (e.g., when $\rho\geq 4$). The main reason is that the minimum distance between true centers across all features is actually $\rho^{\ast\ast}=0$ in this case. Consequently, the separation condition in Theorem~\ref{theorem_converge_to_truth} will never be satisfied, leading to the failure of exact recovery of cluster labels. 
More specifically, for those data points that is only observed in the noise feature (the second dimension), the overlap of true cluster centers in noise feature makes it almost impossible to correctly group these data points, and thus the convergence of estimated centers to true centers becomes very challenging.

\begin{figure}[t]
\captionsetup[subfigure]{justification=centering}
    \centering
    \begin{subfigure}{0.23\textwidth}
        \includegraphics[width=\textwidth]{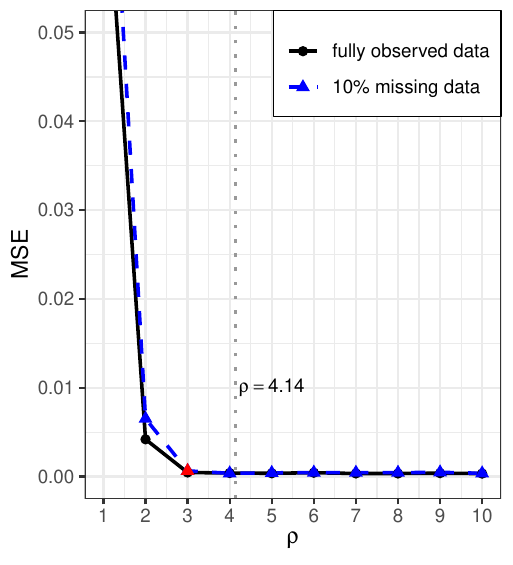}
        \caption*{Case (a), 10\%}
    \end{subfigure}
    \begin{subfigure}{0.23\textwidth}
        \includegraphics[width=\textwidth]{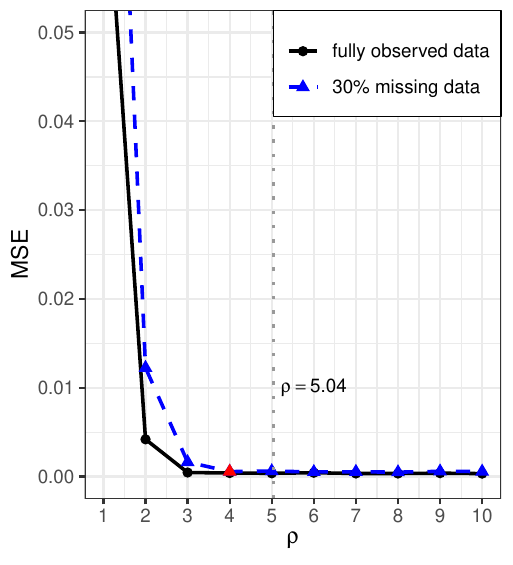}
        \caption*{Case (a), 30\%}
    \end{subfigure}
    \begin{subfigure}{0.23\textwidth}
        \includegraphics[width=\textwidth]{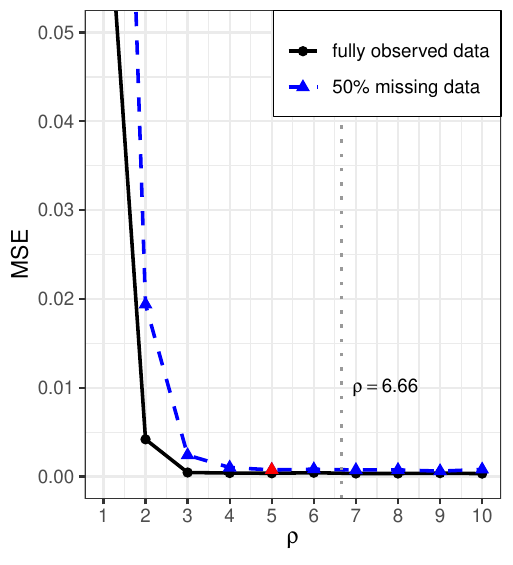}
        \caption*{Case (a), 50\%}
    \end{subfigure}
    \par
    \begin{subfigure}{0.23\textwidth}
        \includegraphics[width=\textwidth]{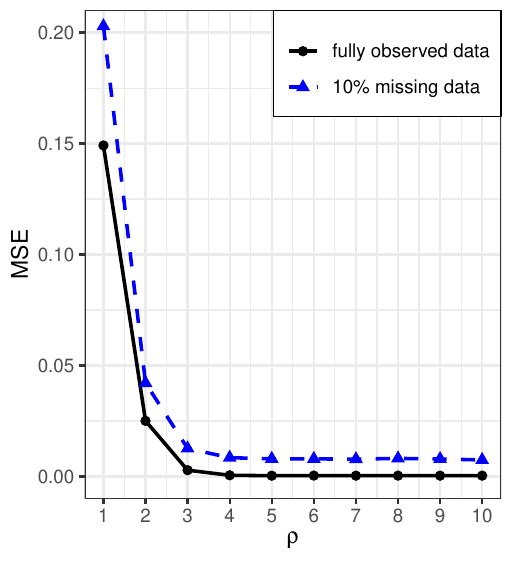}
        \caption*{Case (b), 10\%}
    \end{subfigure}
    \begin{subfigure}{0.23\textwidth}
        \includegraphics[width=\textwidth]{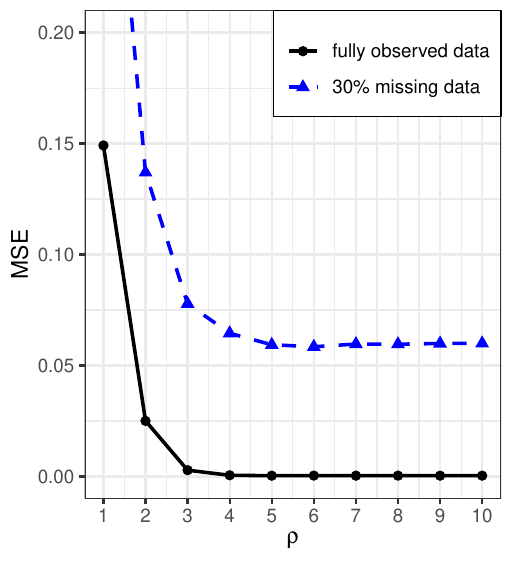}
        \caption*{Case (b), 30\%}
    \end{subfigure}
    \begin{subfigure}{0.23\textwidth}
        \includegraphics[width=\textwidth]{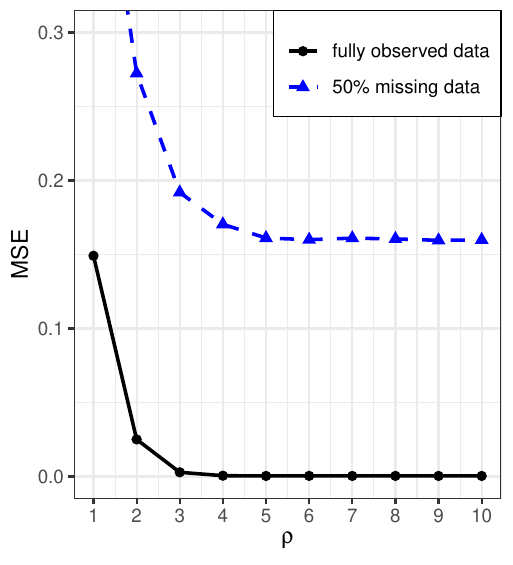}
        \caption*{Case (b), 50\%}
    \end{subfigure}
    \caption{The MSE values versus the parameter $\rho$ of true cluster centers. In panels of Case (a), the red point represents the empirical transition point (the minimum $\rho$ required to achieve MSE $\approx 0$), and the vertical dotted line represents the approximated theoretical threshold. }
    \label{fig_converge_to_truth}
\end{figure}

\section{Conclusions}

In this paper, we analyzed the statistical properties of $k$-means in the presence of missing data, including the excess risk bound, consistency, convergence rate and asymptotic normality of the estimated cluster centers. We also identified conditions for asymptotically recovering the true centers from incomplete data. 
These results provide a theoretical guarantee for extending the $k$-means clustering to incomplete datasets, and were verified by our numerical experiments. 

Our findings reveal a critical contrast, that is, even under the simplest MCAR mechanism, achieving the $\sqrt{n}$-convergence rate requires the population-level optimal cluster centers to have $k$ distinct components in every dimension, which is stricter than that of complete-data-$k$-means.  
Moreover, our separation condition for convergence to true cluster centers requires $k$ true centers distinct in every dimension, which also implies that population-level optimal cluster centers (i.e., the limit of estimated cluster centers) should be distinct in every dimension. 
These results clarify the scope of application of missing-data-$k$-means and highlight a significant challenge of using it in high-dimensional regimes, where we usually suppose that there exist \textit{noise} dimensions in which the true centers are overlapped. 

The main limitation of this work is that most results are restricted to the MCAR mechanism, whereas the asymptotic behavior under other mechanisms, such as MAR and MNAR, is still unclear. According to our experiments, under MNAR mechanism where the missingness depends on data distribution, the $\sqrt{n}$-convergence rate can hold, but asymptotic normality may not hold, which leaves us an important direction for future work, since MNAR mechanism is more common in practice.

\section*{Supplementary materials}
The proofs of all theoretical results are provided in Appendix~\ref{sec_proofs}. 
The code for reproducing all experimental results can be found from \url{https://github.com/GXguanxin/stat_property_of_missing-data-k-means}.

\bibliographystyle{abbrvnat}
\bibliography{main}

@article{chi2016k,
  title={k-pod: A method for k-means clustering of missing data},
  author={Chi, Jocelyn T and Chi, Eric C and Baraniuk, Richard G},
  journal={The American Statistician},
  volume={70},
  number={1},
  pages={91--99},
  year={2016},
  publisher={Taylor \& Francis}
}

@article{wang2019k,
  title={K-means clustering with incomplete data},
  author={Wang, Siwei and Li, Miaomiao and Hu, Ning and Zhu, En and Hu, Jingtao and Liu, Xinwang and Yin, Jianping},
  journal={IEEE Access},
  volume={7},
  pages={69162--69171},
  year={2019},
  publisher={IEEE}
}

@article{pollard1981strong,
  title={Strong consistency of k-means clustering},
  author={Pollard, David},
  journal={The annals of statistics},
  volume = {10},
  number = {1},
  pages={135--140},
  year={1981}
}

@article{Pollard1982,
author = {Pollard, David},
journal = {The Annals of Probability},
number = {4},
pages = {919--926},
title = {{A Central Limit Theorem for $k$-means Clustering}},
volume = {10},
year = {1982}
}

@article{cuesta1997trimmed,
  title={Trimmed $k$-means: an attempt to robustify quantizers},
  author={Cuesta-Albertos, Juan Antonio and Gordaliza, Alfonso and Matr{\'a}n, Carlos},
  journal={The Annals of Statistics},
  volume={25},
  number={2},
  pages={553--576},
  year={1997}
}

@article{terada2014strong,
  title={Strong consistency of reduced k-means clustering},
  author={Terada, Yoshikazu},
  journal={Scandinavian Journal of Statistics},
  volume={41},
  number={4},
  pages={913--931},
  year={2014}
}

@article{sun2012regularized,
  title={Regularized k-means clustering of high-dimensional data and its asymptotic consistency},
  author={Sun, Wei and Wang, Junhui and Fang, Yixin},
  year={2012},
  volume = {6},
  journal = {Electronic Journal of Statistics},
  pages = {148 -- 167},
}

@article{raymaekers2022regularized,
  title={Regularized k-means through hard-thresholding},
  author={Raymaekers, Jakob and Zamar, Ruben H},
  journal={Journal of Machine Learning Research},
  volume={23},
  number={93},
  pages={1--48},
  year={2022}
}

@article{georgogiannis2016robust,
  title={Robust k-means: a theoretical revisit},
  author={Georgogiannis, Alexandros},
  journal={Advances in Neural Information Processing Systems},
  volume={29},
  year={2016}
}

@article{chakraborty2019strong,
  title={On the strong consistency of feature-weighted k-means clustering in a nearmetric space},
  author={Chakraborty, Saptarshi and Das, Swagatam},
  journal={Stat},
  volume={8},
  number={1},
  pages={e227},
  year={2019},
  publisher={Wiley Online Library}
}

@article{von2008consistency,
  title={Consistency of spectral clustering},
  author={Von Luxburg, Ulrike and Belkin, Mikhail and Bousquet, Olivier},
  journal={The Annals of Statistics},
  pages={555--586},
  year={2008}
}

@article{paul2022implicit,
  title={Implicit annealing in kernel spaces: A strongly consistent clustering approach},
  author={Paul, Debolina and Chakraborty, Saptarshi and Das, Swagatam and Xu, Jason},
  journal={IEEE Transactions on Pattern Analysis and Machine Intelligence},
  volume={45},
  number={5},
  pages={5862--5871},
  year={2022},
  publisher={IEEE}
}

@inproceedings{liang2023consistency,
  title={Consistency of multiple kernel clustering},
  author={Liang, Weixuan and Liu, Xinwang and Liu, Yong and Ma, Chuan and Zhao, Yunping and Liu, Zhe and Zhu, En},
  booktitle={International Conference on Machine Learning},
  pages={20650--20676},
  year={2023},
  organization={PMLR}
}

@article{garcia1999central,
  title={A central limit theorem for multivariate generalized trimmed k-means},
  author={Garc{\'\i}a-Escudero, Luis A and Gordaliza, Alfonso and Matr{\'a}n, Carlos},
  journal={Annals of statistics},
  pages={1061--1079},
  year={1999},
  publisher={JSTOR}
}

@article{yang1994asymptotic,
  title={On asymptotic normality of a class of fuzzy c-means clustering procedures},
  author={Yang, Miin-Shen},
  journal={International Journal Of General System},
  volume={22},
  number={4},
  pages={391--403},
  year={1994},
  publisher={Taylor \& Francis}
}

@article{liu2016knn,
title = "Adaptive imputation of missing values for incomplete pattern classification", 
author = "Zhunga Liu and Quan Pan and Jean Dezert and Arnaud Martin", 
journal = "Pattern Recognition", 
volume = "52", 
pages = "85--95", 
year = "2016"
}

@inproceedings{yoon2018gain,
  title= "Gain: Missing data imputation using generative adversarial nets", 
  author= "Yoon, Jinsung and Jordon, James and Schaar, Mihaela", 
  booktitle= "International conference on machine learning", 
  pages= "5689--5698", 
  year= "2018", 
  organization= "PMLR"
}

@article{Choudhury2019nn,
title = "Imputation of missing data with neural networks for classification", 
author = "Suvra Jyoti Choudhury and Nikhil R. Pal", 
journal = "Knowledge--Based Systems", 
volume = "182", 
pages = "104838", 
year = "2019"
}

@article{spinelli2020,
title = "Missing data imputation with adversarially--trained graph convolutional networks", 
author = "Indro Spinelli and Simone Scardapane and Aurelio Uncini", 
journal = "Neural Networks", 
volume = "129", 
pages = "249--260", 
year = "2020"
}

@inproceedings{abdallah2014mean,
  title= "Mean shift clustering algorithm for data with missing values", 
  author= "AbdAllah, Loai and Shimshoni, Ilan", 
  booktitle= "International Conference on Data Warehousing and Knowledge Discovery", 
  pages= "426--438", 
  year= "2014", 
  organization= "Springer"
}

@article{datta2018clustering,
  title={Clustering with missing features: a penalized dissimilarity measure based approach},
  author={Datta, Shounak and Bhattacharjee, Supritam and Das, Swagatam},
  journal={Machine Learning},
  volume={107},
  pages={1987--2025},
  year={2018},
  publisher={Springer}
}

@article{mesquita2017,
title = " Euclidean distance estimation in incomplete datasets", 
author = "Diego P.P. Mesquita and João P.P. Gomes and Amauri H. {Souza Junior} and Juvêncio S. Nobre", 
journal = "Neurocomputing", 
volume = "248", 
pages = "11--18", 
year = "2017"
}

@article{lithio2018efficient,
  title={An efficient k-means-type algorithm for clustering datasets with incomplete records},
  author={Lithio, Andrew and Maitra, Ranjan},
  journal={Statistical Analysis and Data Mining: The ASA Data Science Journal},
  volume={11},
  number={6},
  pages={296--311},
  year={2018}
}

@article{aschenbruck2023imputation,
  title={Imputation strategies for clustering mixed-type data with missing values},
  author={Aschenbruck, Rabea and Szepannek, Gero and Wilhelm, Adalbert FX},
  journal={Journal of Classification},
  volume={40},
  number={1},
  pages={2--24},
  year={2023}
}

@inproceedings{agliz2025joint,
  title={Joint dimensionality reduction and clustering with missing data},
  author={Agliz, Yasmine and Audigier, Vincent and Niang, Nd{\`e}ye and Nadif, Mohamed},
  booktitle={Advanced Machine Learning and Data Science},
  year={2025}
}

@book{van1996weak,
    author = {Van Der Vaart, Aad W and Wellner, Jon A},
    title = {Weak convergence and empirical processes: with applications to statistics},
    publisher = {Springer},
    year = {1996}
}

@book{mohri2018foundations,
  title={Foundations of machine learning},
  author={Mohri, Mehryar and Rostamizadeh, Afshin and Talwalkar, Ameet},
  year={2018},
  publisher={MIT press}
}

@article{bartlett2002rademacher,
  title={Rademacher and gaussian complexities: Risk bounds and structural results},
  author={Bartlett, Peter L and Mendelson, Shahar},
  journal={Journal of machine learning research},
  volume={3},
  number={Nov},
  pages={463--482},
  year={2002}
}

@article{terada2025a,
title={A note on the $k$-means clustering for missing data},
author={Yoshikazu Terada and Xin Guan},
journal={Transactions on Machine Learning Research},
year={2025},
url={https://openreview.net/forum?id=pcqlTvePXS}
}

\newpage
\appendix


\numberwithin{equation}{section}

\section{Proofs}
\label{sec_proofs}

\subsection{Proofs of Section~\ref{sec_finite_result}}

\subsubsection{Proof of Lemma~\ref{lemma_uniform_convergence}}
\begin{proof}
    We first define two function classes on $\mathcal{X}\times \{0,1\}^p$ to be 
    \begin{align*}
        \mathcal{G}_{k}&=\big\{  g_{k,\bm{M}}:(\bm{x},\bm{r})\mapsto \min_{l=1,\dots,k} \| \bm{x}\circ \bm{r} - \bm{\mu}_l \circ \bm{r} \|_2^2 \;\big|\; \forall \bm{M}\in\mathcal{X}^k, \|\bm{\mu}_l\|_2\leq B, l=1,\dots,k \big\} \\
        \mathcal{G}&=\big\{ g_{\bm{M}}:(\bm{x},\bm{r})\mapsto \| \bm{x}\circ \bm{r} - \bm{\mu} \circ \bm{r} \|_2^2 \;\big|\; \forall \bm{\mu} \in \mathcal{X}, \|\bm{\mu}\|_2\leq B \big\}.
    \end{align*}
    Any $g_{k,\bm{M}}\in\mathcal{G}_{k}$ and $g_{\bm{M}}\in \mathcal{G}$ are mappings from $\mathcal{X}\times \{0,1\}^p$ to $[0,4B^2]$. 
    Then, for a size-$n$ random sample of missing data $\bm{\mathrm{X}}\circ\bm{\mathrm{R}}=\{\bm{\mathrm{x}}_1\circ\bm{\mathrm{r}}_1,\dots,\bm{\mathrm{x}}_n\circ\bm{\mathrm{r}}_n\}$, according to Theorem 3.3 of \cite{mohri2018foundations}, it holds that for any $\delta\in (0,1)$,
    \begin{align*}
        \textnormal{Pr}\left( 
        \mathbb{E}_{\bm{\mathrm{x}}_1,\bm{\mathrm{r}}_1}\left[ g_{k,\bm{M}}(\bm{\mathrm{x}}_1,\bm{\mathrm{r}}_1) \right] 
        - \frac{1}{n}\sum_{i=1}^{n} g_{k,\bm{M}}(\bm{\mathrm{x}}_i,\bm{\mathrm{r}}_i) \leq  \mathfrak{R}_n(\mathcal{G}_k) + 4B^2\sqrt{\frac{\log(1/\delta)}{2n}}\right)
        > 1-\delta,
    \end{align*}
    where $\mathfrak{R}_n(\mathcal{G}_k)$ is the Rademacher complexity of $\mathcal{G}_k$.  
    Moreover, for each $\bm{M}\in\mathcal{X}^k$, we can write 
    \begin{align*}
        L(\bm{M})=\mathbb{E}_{\bm{\mathrm{x}}_1,\bm{\mathrm{r}}_1}\left[ g_{k,\bm{M}}(\bm{\mathrm{x}}_1,\bm{\mathrm{r}}_1) \right]
        \quad \textnormal{and} \quad 
        \widehat{L}_n(\bm{M})=\frac{1}{n}\sum_{i=1}^{n} g_{k,\bm{M}}(\bm{\mathrm{x}}_i,\bm{\mathrm{r}}_i), 
    \end{align*}
    and according to Theorem 12 of \cite{bartlett2002rademacher}, the empirical Rademacher complexity of $\mathcal{G}_{k}$ and that of $\mathcal{G}$ satisfy $\mathfrak{R}_n(\mathcal{G}_k)\leq k\cdot \mathfrak{R}_n(\mathcal{G})$. 
    Then using the symmetry, we can have 
    \begin{align}
        \textnormal{Pr}\left( 
        \left| L(\bm{M}) - \widehat{L}_n(\bm{M}) \right| \leq 4k \mathfrak{R}_n(\mathcal{G}) + 8B^2\sqrt{\frac{\log(1/\delta)}{2n}}\right)
        > 1-\delta.
    \label{proof_uniform_convergence_eq_1}
    \end{align}

    We next derive the upper bound of $\mathfrak{R}_n(\mathcal{G})$. 
    Let $\bm{\epsilon} = \{\epsilon_i\}_{i=1}^n$ be independent random variables that each $\epsilon_i$ takes the value $\pm 1$ with equal probability $1/2$, and $\bm{\epsilon}$ is independent with $(\bm{\mathrm{X}},\bm{\mathrm{R}})$ then the Rademacher complexity of $\mathcal{G}$ is given by 
    \begin{align*}
        \mathfrak{R}_n(\mathcal{G})
        &=\mathbb{E}_{ \bm{\mathrm{X}},\bm{\mathrm{R}},\bm{\epsilon}}\left[ \sup_{g_{\bm{M}}\in\mathcal{G}} \frac{1}{n}\sum_{i=1}^{n}\epsilon_i g_{\bm{M}}(\bm{\mathrm{x}}_i,\bm{\mathrm{r}}_i) \right]
        =\mathbb{E}_{\bm{\mathrm{X}},\bm{\mathrm{R}},\bm{\epsilon}}\left[ \sup_{\|\bm{\mu}\|_2\leq B} \frac{1}{n}\sum_{i=1}^{n} \epsilon_i \| \bm{\mathrm{x}}_i\circ \bm{\mathrm{r}}_i - \bm{\mu}\circ \bm{\mathrm{r}}_i \|_2^2 \right] \\
        &=\mathbb{E}_{\bm{\mathrm{X}},\bm{\mathrm{R}},\bm{\epsilon}}\left[ \sup_{\|\bm{\mu}\|_2\leq B} \frac{1}{n}\sum_{i=1}^{n} \epsilon_i \bigg\{  \|\bm{\mathrm{x}}_i\circ \bm{\mathrm{r}}_i \|_2^2 + \|\bm{\mu} \circ \bm{\mathrm{r}}_i\|_2^2 -2\langle \bm{\mathrm{x}}_i\circ \bm{\mathrm{r}}_i, \bm{\mu}\circ \bm{\mathrm{r}}_i \rangle \bigg\} \right] \\
        &\leq  \mathbb{E}_{\bm{\mathrm{X}},\bm{\mathrm{R}},\bm{\epsilon}}\left[ \sup_{\|\bm{\mu}\|_2\leq B} \frac{1}{n}\sum_{i=1}^{n} \epsilon_i \|\bm{\mathrm{x}}_i\circ \bm{\mathrm{r}}_i \|_2^2\right] 
        + \mathbb{E}_{\bm{\mathrm{X}},\bm{\mathrm{R}},\bm{\epsilon}}\left[ \sup_{\|\bm{\mu}\|_2\leq B} \frac{1}{n}\sum_{i=1}^{n} \epsilon_i \|\bm{\mu} \circ \bm{\mathrm{r}}_i\|_2^2 \right] \\
        &\quad + 2\mathbb{E}_{\bm{\mathrm{X}},\bm{\mathrm{R}},\bm{\epsilon}}\left[ \sup_{\|\bm{\mu}\|_2\leq B} \left| \frac{1}{n}\sum_{i=1}^{n} \epsilon_i \langle \bm{\mathrm{x}}_i\circ \bm{\mathrm{r}}_i, \bm{\mu}\circ \bm{\mathrm{r}}_i \rangle \right| \right]. 
    \end{align*}
    It suffices to bound the three terms, respectively. 
    For the first term, under Assumption~\ref{assumption_X_compact}, since $\|\bm{x}\|_2\leq B$ for any $\bm{x}\in\mathcal{X}$, then we have
    \begin{align*}
        \mathbb{E}_{\bm{\mathrm{X}},\bm{\mathrm{R}},\bm{\epsilon}}\left[ \sup_{\|\bm{\mu}\|_2\leq B} \frac{1}{n}\sum_{i=1}^{n} \epsilon_i \|\bm{\mathrm{x}}_i\circ \bm{\mathrm{r}}_i \|_2^2\right] 
        &= \mathbb{E}_{\bm{\mathrm{X}},\bm{\mathrm{R}},\bm{\epsilon}}\left[ \frac{1}{n}\sum_{i=1}^{n} \epsilon_i \|\bm{\mathrm{x}}_i\circ \bm{\mathrm{r}}_i \|_2^2\right] \\
        &\leq \mathbb{E}_{\bm{\mathrm{X}},\bm{\mathrm{R}},\bm{\epsilon}}\left[ \frac{1}{n}\sum_{i=1}^{n} \epsilon_i \|\bm{\mathrm{x}}_i \|_2^2\right]
        = \frac{1}{n}\sum_{i=1}^{n} \bigg\{ \mathbb{E}_{\epsilon_i}\left[  \epsilon_i\right]\cdot   \mathbb{E}_{\bm{\mathrm{x}}_i,\bm{\mathrm{r}}_i}\left[\|\bm{\mathrm{x}}_i \|_2^2 \right]\bigg\} \\
        &= \frac{1}{n}\sum_{i=1}^{n} \bigg\{ 0\cdot   \mathbb{E}_{\bm{\mathrm{x}}_i}\left[\|\bm{\mathrm{x}}_i \|_2^2 \right]\bigg\}
        = 0 . 
    \end{align*}
    For the second term, since $\mathbb{E}_{\epsilon_i,\epsilon_{i'}}[\epsilon_i\epsilon_{i'}]=1$ if $i=i'$, 0 otherwise, and $\|\bm{r}\|_2^2\leq p$ for any $\bm{r}\in\{0,1\}^p$, then we have 
    \begin{align*}
        &\mathbb{E}_{\bm{\mathrm{X}},\bm{\mathrm{R}},\bm{\epsilon}}\left[ \sup_{\|\bm{\mu}\|_2\leq B} \frac{1}{n}\sum_{i=1}^{n} \epsilon_i \|\bm{\mu} \circ \bm{\mathrm{r}}_i\|_2^2 \right] \\
        &=\mathbb{E}_{\bm{\mathrm{X}},\bm{\mathrm{R}},\bm{\epsilon}}\left[ \sup_{\|\bm{\mu}\|_2\leq B} \frac{1}{n}\sum_{i=1}^{n} \epsilon_i  \left\langle \bm{\mu}\circ \bm{\mathrm{r}}_i \;,\; \bm{\mu}\circ \bm{\mathrm{r}}_i \right\rangle \right]\\
        &=\mathbb{E}_{\bm{\mathrm{X}},\bm{\mathrm{R}},\bm{\epsilon}}\left[ \sup_{\|\bm{\mu}\|_2\leq B}  \frac{1}{n}\sum_{i=1}^{n} \epsilon_i \left\langle \bm{\mu}\circ \bm{\mu} \;,\; \bm{\mathrm{r}}_i\right\rangle \right]
        =\frac{1}{n}\mathbb{E}_{\bm{\mathrm{X}},\bm{\mathrm{R}},\bm{\epsilon}}\left[ \sup_{\|\bm{\mu}\|_2\leq B}  \left\langle \bm{\mu}\circ \bm{\mu} \;,\; \sum_{i=1}^{n} \epsilon_i \bm{\mathrm{r}}_i\right\rangle  \right]\\
        &\leq \frac{1}{n}\mathbb{E}_{\bm{\mathrm{X}},\bm{\mathrm{R}},\bm{\epsilon}}\left[ \left( \sup_{\|\bm{\mu}\|_2\leq B} \left\| \bm{\mu}\circ \bm{\mu} \right\|_2 \right)\cdot \left\| \sum_{i=1}^{n} \epsilon_i \bm{\mathrm{r}}_i \right\|_2  \right] 
        =\frac{1}{n} \left( \sup_{\|\bm{\mu}\|_2\leq B}\left\| \bm{\mu}\circ \bm{\mu} \right\|_2 \right) \cdot \mathbb{E}_{\bm{\mathrm{X}},\bm{\mathrm{R}},\bm{\epsilon}}\left[ \left\| \sum_{i=1}^{n} \epsilon_i \bm{\mathrm{r}}_i \right\|_2  \right]\\
        &\leq \frac{B^2}{n}\left\{ \mathbb{E}_{\bm{\mathrm{X}},\bm{\mathrm{R}},\bm{\epsilon}}\left[ \left\| \sum_{i=1}^{n} \epsilon_i \bm{\mathrm{r}}_i \right\|_2^2  \right] \right\}^{1/2} \\
        &=\frac{B^2}{n}\left\{ \sum_{i=1}^{n} \mathbb{E}_{\epsilon_i} \left[\epsilon_i \epsilon_i \right] \cdot \mathbb{E}_{\bm{\mathrm{x}}_i,\bm{\mathrm{r}}_i}\left[\langle \bm{\mathrm{r}}_i,\bm{\mathrm{r}}_i \rangle\right] \right\}^{1/2}
        =\frac{B^2}{n}\left\{ \sum_{i=1}^{n} 1\cdot \mathbb{E}_{\bm{\mathrm{r}}_i}\left[ \|\bm{\mathrm{r}}_i\|_2^2 \right] \right\}^{1/2}\\
        &\leq \frac{B^2}{n}\cdot\sqrt{np}=\frac{B^2\sqrt{p}}{\sqrt{n}}.
    \end{align*}
    For the third term, since $\mathbb{E}_{\epsilon_i,\epsilon_{i'}}[\epsilon_i\epsilon_{i'}]=1$ if $i=i'$, 0 otherwise, and under Assumption~\ref{assumption_X_compact}, $\|\bm{x}\|_2\leq B$ for any $\bm{x}\in\mathcal{X}$, then we have 
    \begin{align*}
        &2\mathbb{E}_{\bm{\mathrm{X}},\bm{\mathrm{R}},\bm{\epsilon}}\left[ \sup_{\|\bm{\mu}\|_2\leq B} \left| \frac{1}{n}\sum_{i=1}^{n} \epsilon_i \langle \bm{\mathrm{x}}_i\circ \bm{\mathrm{r}}_i, \bm{\mu}\circ \bm{\mathrm{r}}_i \rangle \right| \right] \\
        &=2\mathbb{E}_{\bm{\mathrm{X}},\bm{\mathrm{R}},\bm{\epsilon}}\left[ \sup_{\|\bm{\mu}\|_2\leq B} \left| \frac{1}{n}\sum_{i=1}^{n} \epsilon_i \langle \bm{\mathrm{x}}_i\circ \bm{\mathrm{r}}_i\;,\; \bm{\mu} \rangle \right| \right] 
        =2\mathbb{E}_{\bm{\mathrm{X}},\bm{\mathrm{R}},\bm{\epsilon}}\left[ \sup_{\|\bm{\mu}\|_2\leq B} \left|\left\langle \frac{1}{n}\sum_{i=1}^{n} \epsilon_i \bm{\mathrm{x}}_i\circ \bm{\mathrm{r}}_i\;,\; \bm{\mu} \right\rangle \right| \right] \\
        &\leq 2\mathbb{E}_{\bm{\mathrm{X}},\bm{\mathrm{R}},\bm{\epsilon}}\left[ \sup_{\|\bm{\mu}\|_2\leq B} \|\bm{\mu}\|_2 \cdot \left\| \frac{1}{n}\sum_{i=1}^{n} \epsilon_i \bm{\mathrm{x}}_i\circ \bm{\mathrm{r}}_i \right\|_2   \right] 
        = \frac{2B}{n} \cdot \mathbb{E}_{\bm{\mathrm{X}},\bm{\mathrm{R}},\bm{\epsilon}}\left[ \left\| \sum_{i=1}^{n} \epsilon_i \bm{\mathrm{x}}_i\circ \bm{\mathrm{r}}_i \right\|_2  \right] \\
        &\leq \frac{2B}{n} \cdot \left\{ \mathbb{E}_{\bm{\mathrm{X}},\bm{\mathrm{R}},\bm{\epsilon}}\left[ \left\| \sum_{i=1}^{n} \epsilon_i \bm{\mathrm{x}}_i\circ \bm{\mathrm{r}}_i \right\|_2^2  \right]  \right\}^{1/2} 
        =\frac{2B}{n} \cdot \left\{ \mathbb{E}_{\bm{\mathrm{X}},\bm{\mathrm{R}},\bm{\epsilon}}\left[ \sum_{i,i'=1}^{n} \epsilon_i \epsilon_{i'} \langle \bm{\mathrm{x}}_i\circ \bm{\mathrm{r}}_i \;,\; \bm{\mathrm{x}}_{i'}\circ \bm{\mathrm{r}}_{i'}\rangle \right]  \right\}^{1/2}\\
        &=\frac{2B}{n}\left\{ \sum_{i=1}^{n} \mathbb{E}_{\epsilon_i}\left[ \epsilon_i\epsilon_i \right]\cdot \mathbb{E}_{\bm{\mathrm{x}}_i,\bm{\mathrm{r}}_i}\left[ \|\bm{\mathrm{x}}_i\circ \bm{\mathrm{r}}_i\|_2^2\right] \right\}^{1/2} \\
        &\leq \frac{2B}{n}\left\{ \sum_{i=1}^{n} 1 \cdot \mathbb{E}_{\bm{\mathrm{x}}_i,\bm{\mathrm{r}}_i}\left[ \|\bm{\mathrm{x}}_i\|_2^2\right] \right\}^{1/2}
        = \frac{2B}{n}\left\{ \sum_{i=1}^{n} \mathbb{E}_{\bm{\mathrm{x}}_i}\left[ \|\bm{\mathrm{x}}_i\|_2^2\right] \right\}^{1/2} \\
        &\leq \frac{2B}{n} \sqrt{nB^2}
        =\frac{2B^2}{\sqrt{n}}.
    \end{align*}
    Combining the above three bounds leads to
    \begin{align}
        \widehat{\mathfrak{R}}_n(\mathcal{G}) 
        \leq 0+\frac{B^2\sqrt{p}}{\sqrt{n}} + \frac{2B^2}{\sqrt{n}} = \frac{B^2(\sqrt{p}+2)}{\sqrt{n}}.
    \label{proof_uniform_convergence_eq_2}
    \end{align}
    Therefore, combining Eq.(\ref{proof_uniform_convergence_eq_1}) and Eq.(\ref{proof_uniform_convergence_eq_2}), we have for any $\delta\in (0,1)$, 
    \begin{align*}
        \textnormal{Pr}\left( 
        \left| L(\bm{M}) - \widehat{L}_n(\bm{M}) \right| \leq  \frac{4B^2}{\sqrt{n}} \cdot \left\{ k(\sqrt{p}+2) + \sqrt{2}\cdot \sqrt{\log(1/\delta)} \right\}
        \right)
        > 1-\delta,
    \end{align*}
    which completes the proof.
\end{proof}

\subsubsection{Proof of Theorem~\ref{theorem_excess_risk}}
\begin{proof}
    For any $n\in \mathbb{N}_+$, under Assumption~\ref{assumption_unique_minimizer} and according to the definition of $\bm{M}^{\ast}$,
    we have the excess risk is bounded by
    \begin{align*}
        L(\widehat{\bm{\mathrm{M}}}_n) - \min_{\bm{M}\in\mathcal{X}^k} L(\bm{M})
        &=
        L(\widehat{\bm{\mathrm{M}}}_n) - L(\bm{M}^{\ast})\\
        &=L(\widehat{\bm{\mathrm{M}}}_n) - \widehat{L}_n(\widehat{\bm{\mathrm{M}}}_n) + \widehat{L}_n(\widehat{\bm{\mathrm{M}}}_n) - \widehat{L}_n(\bm{M}^{\ast}) + \widehat{L}_n(\bm{M}^{\ast}) - L(\bm{M}^{\ast})\\
        &\leq 2\cdot \sup_{\bm{M}\in \mathcal{X}^k} \left| L(\bm{M}) -  \widehat{L}_n(\bm{M}) \right| + \widehat{L}_n(\widehat{\bm{\mathrm{M}}}_n) - \widehat{L}_n(\bm{M}^{\ast}) \\
        &\leq 2\cdot \sup_{\bm{M}\in \mathcal{X}^k} \left| L(\bm{M}) -  \widehat{L}_n(\bm{M}) \right|,
    \end{align*}
    where the final inequality is because $\widehat{L}_n(\widehat{\bm{\mathrm{M}}}_n) - \widehat{L}_n(\bm{M}^{\ast}) \leq 0$ according to the definitions of $\widehat{\bm{\mathrm{M}}}_n$. 
    Then, by using Lemma~\ref{lemma_uniform_convergence}, we obtain for any $\delta\in (0,1)$, 
    \begin{align*}
        \textnormal{Pr}\left( 
        L(\widehat{\bm{\mathrm{M}}}_n) - L(\bm{M}^{\ast}) \leq  \frac{8B^2}{\sqrt{n}} \cdot \left\{ k(\sqrt{p}+2) + \sqrt{2}\cdot \sqrt{\log(1/\delta)} \right\}
        \right)
        > 1-\delta,
    \end{align*}
    which completes the proof. 
\end{proof}

\subsubsection{Proof of Lemma~\ref{lemma_identifiability}}
\begin{proof}
    Through this proof, we let
    \begin{align*}
        m^{\ast}=\mathop{\min}_{\bm{M}\in\mathcal{X}^k} L(\bm{M})
        \quad\text{and}\quad
        \mathcal{L}=\{ L(\bm{M})\;|\; \bm{M}\in\mathcal{X}^k\}.
    \end{align*}
    We first prove: there exists a sequence $\{\bm{V}_n\}_{n\in \mathbb{N}_+}\subset \mathcal{X}^k$ such that $\lim_{n\rightarrow\infty} L(\bm{V}_n) = m^{\ast}$. 
    To this end, we note that for any $a>m^{\ast}$, there exists an $\bm{M}\in \mathcal{X}^k$ such that $b=L(\bm{M})<a$. 
    Then take $a_n=m^{\ast}+1/n$, which means $\lim_{n\rightarrow\infty} a_n=m^{\ast}$, and take $\mathcal{L}_n =\{ b \;|\;\forall b\in \mathcal{L}, b<a_n\}$. We have $\mathcal{L}_n\neq \emptyset$. 
    Denote by $\mathfrak{B}(\mathcal{L})$ the power set of $\mathcal{L}$. According to the axiom of choice, there exists a function $g:\mathfrak{B}(\mathcal{L})\setminus \{\emptyset\}\mapsto \mathcal{L}$ such that for any $B\in \mathfrak{B}(\mathcal{L})\setminus \{\emptyset\}$, $g(\mathcal{L})\in B$. 
    Thus, let $b_n=g(\mathcal{L}_n)$, then we have $b_n\in \mathcal{L}_n$, which means $m^{\ast}\leq b_n<a_n$. It follows that $\lim_{n\rightarrow\infty} b_n = m^{\ast}$, which implies the existence of the sequence $\{\bm{V}_n\}_{n\in\mathbb{N}_+}\subset \mathcal{X}^k$ satisfying $\lim_{n\rightarrow\infty} L(\bm{V}_n) = m^{\ast}$. 

    Next, we turn to prove the unique minimizer $\bm{M}^{\ast}$ is identifiable by deriving a contradiction. 
    For any $\epsilon>0$, write $\mathcal{M}_{\epsilon}=\{\bm{M}\in \mathcal{X}^k \;|\; \|\text{vec}(\bm{M},\bm{M}^{\ast})\|_2>\epsilon \}$. 
    If there exists an $\epsilon>0$ satisfying 
    \begin{align*}
        m^{\ast} = \inf \left\{ L(\bm{M}) \mid \|\text{vec}(\bm{M},\bm{M}^{\ast})\|_2 >\epsilon,\; \bm{M}\in \mathcal{X}^k  \right\},
    \end{align*}
    then by using the same technique in the above proof, we can obtain: there exists a sequence $\{\bm{V}_n\}_{n\in\mathbb{N}}\subset \mathcal{M}_{\epsilon}$ such that $\lim_{n\rightarrow\infty} L(\bm{V}_n)=m^{\ast}$. 
    Moreover, since under Assumption~\ref{assumption_X_compact}, $\mathcal{X}$ is compact, implying that $\mathcal{X}^k$ is also compact, then there exists a convergent subsequence $\{\bm{V}_{n_t}\}_{t\in\mathbb{N}}$ of $\bm{V}_n$, and we can write the limit to be $\bm{V}^{\dagger}=\lim_{t\rightarrow\infty} \bm{V}_{n_t}$. It follows that $L(\bm{V}^{\dagger})=m^{\ast}$. Then, the uniqueness of the minimizer of $L(\cdot)$ (i.e., Assumption~\ref{assumption_unique_minimizer}) implies $\bm{V}^{\dagger}=\bm{M}^{\ast}$. 
    On the other hand, the convergence of $\bm{V}_{n_t}$ to $\bm{V}^{\dagger}$ means that there exists $t_0\in\mathbb{N}$, such that for any $t\geq t_0$, the $\|\text{vec}(\bm{V}_{n_t},\bm{V}^{\dagger})\|_2\leq\epsilon$ holds. It follows that $\|\text{vec}(\bm{V}_{n_t},\bm{M}^{\ast})\|_2\leq\epsilon$ holds for sufficiently large $t$, which implies that $\bm{V}_{n_t}\notin \mathcal{M}_{\epsilon}$. This is a contradiction, which completes the proof. 
\end{proof}

\subsubsection{Proof of Theorem~\ref{theorem_consistency}}
\begin{proof}
    By Theorem~\ref{theorem_excess_risk}, we have for any $\delta\in(0,1)$ and $\tilde{\epsilon}>0$, there exists $N\in\mathbb{N}_+$ such that for any $n\geq N$, $\text{Pr}\left(  L(\widehat{\bm{\mathrm{M}}}_n) - L(\bm{M}^{\ast}) >  \tilde{\epsilon} \right) \leq \delta$, which means 
    \begin{align*}
        \lim_{n\rightarrow\infty} \text{Pr}\left(  L(\widehat{\bm{\mathrm{M}}}_n) - L(\bm{M}^{\ast}) >  \tilde{\epsilon} \right) = 0.
    \end{align*}
    Then, according to Lemma~\ref{lemma_identifiability}, the identifiability of $\bm{M}^{\ast}$ implies that for any $\epsilon>0$, there exists $\tilde{\epsilon}>0$ such that
    \begin{align*}
        \textnormal{Pr}\left( \|\text{vec}(\widehat{\bm{\mathrm{M}}}_n,\bm{M}^{\ast})\|_2 >\epsilon \right) 
        \leq \textnormal{Pr}\left( L(\widehat{\bm{\mathrm{M}}}_n) - L(\bm{M}^{\ast}) >  \tilde{\epsilon} \right). 
    \end{align*}
    It follows that $\lim_{n\rightarrow\infty} \textnormal{Pr}\left( \|\text{vec}(\widehat{\bm{\mathrm{M}}}_n,\bm{M}^{\ast})\|_2 >\epsilon \right) = 0$, which completes the proof. 
\end{proof}

\subsection{Proofs of Section~\ref{sec_convergence_rate}}

\subsubsection{Proof of Lemma~\ref{lemma_L_1st_derivative}}

\begin{proof}
    (1) First, we prove $\phi(\cdot,\cdot,\bm{M})$ is differentiable with respect to $\bm{M}$ in quadratic mean, which suffices to prove that there exists a matrix-valued function $\Delta(\cdot,\cdot,\bm{M}) \in \mathcal{L}^2(\widetilde{\mathbb{P}})$ such that for all $\bm{A}\in \mathbb{R}^{p\times k}$, 
    \begin{align*}
        \int \bigg\{ \phi(\bm{x},\bm{r},\bm{M}+\bm{A}) - \phi(\bm{x},\bm{r},\bm{M}) - \big\langle \text{vec}(\bm{A}), \text{vec}(\Delta(\bm{x},\bm{r},\bm{M}) ) \big\rangle \bigg\}^2 \;d\widetilde{\mathbb{P}}(\bm{x},\bm{r}) 
        = o(\|\bm{A}\|_F^2).
    \end{align*}

    We consider three cases: 
    \begin{align*}
        &\mathcal{D}_1 = \left\{ (\bm{x},\bm{r})\;\middle|\;  
        \forall l\neq l' ,\; \|\bm{x}\circ\bm{r}-\bm{\mu}_{l}\circ\bm{r}\|_2^2 \neq \|\bm{x}\circ\bm{r}-\bm{\mu}_{l'}\circ\bm{r}\|_2^2 \right\}\\
        &\mathcal{D}_2 = \left\{ (\bm{x},\bm{r})\;\middle|\;  
        \text{there exists } l\neq l' \text{ satisfying } \bm{\mu}_{l}\circ\bm{r}\neq \bm{\mu}_{l'}\circ\bm{r} \text{ and } \|\bm{x}\circ\bm{r}-\bm{\mu}_{l}\circ\bm{r}\|_2^2 = \|\bm{x}\circ\bm{r}-\bm{\mu}_{l'}\circ\bm{r}\|_2^2  \right\}\\
        &\mathcal{D}_3 = \left\{ (\bm{x},\bm{r})\;\middle|\;  
        \text{there exists } l\neq l' \text{ satisfying } \bm{\mu}_{l}\circ\bm{r}= \bm{\mu}_{l'}\circ\bm{r} \right\}.
    \end{align*}

    Consider a fixed $l^{\ast}\in \{1,\dots,k\}$. 
    (i) For any $(\bm{x},\bm{r})\in\mathcal{D}_1$ satisfying $\ell(\bm{x},\bm{r},\bm{M})=l^{\ast}$, we have $\phi(\bm{x},\bm{r},\bm{M})=\| \bm{x}\circ\bm{r}-\bm{\mu}_{l^{\ast}}\circ\bm{r} \|_2^2$, then for any $\bm{A}$ small enough, we have 
    \begin{align*}
        \phi(\bm{x},\bm{r},\bm{M}+\bm{A})
        &= \| \bm{x}\circ\bm{r} - (\bm{\mu}_{l^{\ast}} + \bm{a}_{l^{\ast}})\circ\bm{r} \|_2^2 \\
        &= \|\bm{x}\circ\bm{r}-\bm{\mu}_{l^{\ast}}\circ\bm{r}\|_2^2 
        - 2\langle \bm{a}_{l^{\ast}}\circ\bm{r} , (\bm{x}-\bm{\mu}_{l^{\ast}})\circ\bm{r} \rangle 
        + \|\bm{a}_{l^{\ast}}\circ\bm{r}\|_2^2 \\
        &= \phi(\bm{x},\bm{r},\bm{M}) - 
        2\langle \bm{a}_{l^{\ast}} , (\bm{x}-\bm{\mu}_{l^{\ast}})\circ\bm{r} \rangle 
        + \|\bm{a}_{l^{\ast}}\circ\bm{r}\|_2^2 
    \end{align*}
    Thus, we can define a matrix-valued function $\Delta(\cdot,\cdot,\bm{M}):\mathcal{X}\times \{0,1\}^p\mapsto \mathbb{R}^{p\times k}$ by constructing its $l$-th column as
    \begin{align*}
        \Delta_l (\bm{x},\bm{r},\bm{M}) = -2\mathds{1}( \ell(\bm{x},\bm{r},\bm{M}) = l ) \cdot (\bm{x}\circ\bm{r} - \bm{\mu}_l\circ \bm{r}),
    \end{align*}
    which implies that $\Delta(\cdot,\cdot,\bm{M})\in \mathcal{L}^2(\widetilde{\mathbb{P}})$. 
    Moreover, for any $(\bm{x},\bm{r})$, we define
    \begin{align*}
        R(\bm{x},\bm{r},\bm{M},\bm{A}) = \frac{ 1}{ \|\bm{A} \|_F } \sum_{l=1}^{k} \mathds{1}(\ell(\bm{x},\bm{r},\bm{M}) = l) \cdot \| \bm{a}_l \circ \bm{r} \|_2^2 .
    \end{align*}
    Thus, the following holds for any $(\bm{x},\bm{r})\in\mathcal{D}_1$: 
    \begin{align}
        &\phi(\bm{x},\bm{r},\bm{M}+\bm{A}) \notag\\
        &= \phi(\bm{x},\bm{r},\bm{M}) + \sum_{l=1}^{k} \big\{ -2 \mathds{1}(\ell(\bm{x},\bm{r},\bm{M})=l) \big\} \cdot \langle \bm{a}_l , (\bm{x}-\bm{\mu}_l)\circ\bm{r} \rangle 
        + \sum_{l=1}^{k} \mathds{1}(\ell(\bm{x},\bm{r},\bm{M})=l) \cdot \|\bm{a}_l\circ\bm{r}\|_2^2 \notag\\
        &= \phi(\bm{x},\bm{r},\bm{M}) + \big\langle \text{vec}(\bm{A}), \text{vec}(\Delta(\bm{x},\bm{r},\bm{M}) ) \big\rangle 
        + \|\bm{A}\|_F\cdot R(\bm{x},\bm{r},\bm{M},\bm{A}) .\label{eq_lemma1_phi_decompose}
    \end{align}
    (ii) For any $(\bm{x},\bm{r})\in\mathcal{D}_2$ satisfying $\ell(\bm{x},\bm{r},\bm{M})=l^{\ast}$, we suppose that there exists $l'\neq l^{\ast}$ satisfying $\bm{\mu}_{l^{\ast}}\circ\bm{r}\neq \bm{\mu}_{l'}\circ\bm{r} $ and $\|\bm{x}\circ\bm{r}-\bm{\mu}_{l^{\ast}}\circ\bm{r}\|_2^2 = \|\bm{x}\circ\bm{r}-\bm{\mu}_{l'}\circ\bm{r}\|_2^2 $. Then we have 
    \begin{align*}
        \phi(\bm{x},\bm{r},\bm{M})
        =\| \bm{x}\circ\bm{r}-\bm{\mu}_{l^{\ast}}\circ\bm{r} \|_2^2
        =\|\bm{x}\circ\bm{r}-\bm{\mu}_{l'}\circ\bm{r}\|_2^2 
        < \|\bm{x}\circ\bm{r}-\bm{\mu}_{t}\circ\bm{r}\|_2^2 ,\; \forall t\neq l^{\ast},l'.
    \end{align*}
    Moreover, for any fixed $\bm{s}\in\{0,1\}^p$, we define
    \begin{align*}
        &\mathcal{D}_2^{\bm{s}}=\left\{ (\bm{x},\bm{s})\;\middle|\;  
        \text{there exists } l\neq l' \text{ satisfying } \bm{\mu}_{l}\circ\bm{s}\neq \bm{\mu}_{l'}\circ\bm{s} \text{ and } \|\bm{x}\circ\bm{s}-\bm{\mu}_{l}\circ\bm{s}\|_2^2 = \|\bm{x}_2\circ\bm{s}-\bm{\mu}_{l'}\circ\bm{s}\|_2^2  \right\} \\
        &\mathcal{X}_2^{\bm{s}}=\left\{ \bm{x}\;\middle|\;  
        \text{there exists } l\neq l' \text{ satisfying } \bm{\mu}_{l}\circ\bm{s}\neq \bm{\mu}_{l'}\circ\bm{s} \text{ and } \|\bm{x}\circ\bm{s}-\bm{\mu}_{l}\circ\bm{s}\|_2^2 = \|\bm{x}\circ\bm{s}-\bm{\mu}_{l'}\circ\bm{s}\|_2^2  \right\}.
    \end{align*}
    Then we can write $\mathcal{D}_2 = \bigcup_{\bm{s}\in\{0,1\}^p} \mathcal{D}_2^{\bm{s}} $. 
    It follows that 
    \begin{align*}
        \widetilde{\mathbb{P}}( \mathcal{D}_2 )
        &= \sum_{\bm{s}\in\{0,1\}^p} \widetilde{\mathbb{P}}(\mathcal{D}_2^{\bm{s}} )
        = \sum_{\bm{s}\in\{0,1\}^p} \int_{(\bm{x},\bm{r})\in\mathcal{D}_2^{\bm{s}}} d\widetilde{\mathbb{P}}(\bm{x},\bm{r})
        = \sum_{\bm{s}\in\{0,1\}^p} \int_{\bm{x}\in \mathcal{X}_2^{\bm{s}} } \bigg( \sum_{\bm{\mathrm{r}}=\bm{s}} \text{Pr}(\bm{\mathrm{r}}|\bm{x})   \bigg) d\mathbb{P}(\bm{x})\\
        &= \sum_{\bm{s}\in\{0,1\}^p} \int_{\bm{x}\in \mathcal{X}_2^{\bm{s}} } \text{Pr}(\bm{s}|\bm{x})   \; d\mathbb{P}(\bm{x})
        \leq \sum_{\bm{s}\in\{0,1\}^p} \int_{\bm{x}\in \mathcal{X}_2^{\bm{s}} } 1 \; d\mathbb{P}(\bm{x})
        = \sum_{\bm{s}\in\{0,1\}^p} \mathbb{P}( \mathcal{X}_2^{\bm{s}} )
        = 0 
    \end{align*}
    where the last equation is because $\mathcal{X}^{\bm{s}}$ consists of many hyperplanes in $\mathbb{R}^p$, which has zero $\mathbb{P}-$measure under the Assumption~\ref{assumption_surface_zeromeasure}. 
    
    (iii) For any $(\bm{x},\bm{r})\in\mathcal{D}_3$ satisfying $\ell(\bm{x},\bm{r},\bm{M})=l^{\ast}$, we suppose that there exists $l'\neq l^{\ast}$ satisfying $\bm{\mu}_{l^{\ast}}\circ\bm{r}= \bm{\mu}_{l'}\circ\bm{r} $. 
    Then we have
    \begin{align*}
        \phi(\bm{x},\bm{r},\bm{M})
        =\| \bm{x}\circ\bm{r}-\bm{\mu}_{l^{\ast}}\circ\bm{r} \|_2^2
        =\|\bm{x}\circ\bm{r}-\bm{\mu}_{l'}\circ\bm{r}\|_2^2 
        < \|\bm{x}\circ\bm{r}-\bm{\mu}_{t}\circ\bm{r}\|_2^2 ,\; \forall t\neq l^{\ast},l'.
    \end{align*}
    It follows that for any $\bm{A}$ small enough,  
    \begin{align*}
        \phi(\bm{x},\bm{r},\bm{M}+\bm{A})
        = \min \big\{ \| \bm{x}\circ\bm{r}-(\bm{\mu}_{l^{\ast}} + \bm{a}_{l^{\ast}})\circ\bm{r} \|_2^2\;,\;
        \| \bm{x}\circ\bm{r}-(\bm{\mu}_{l'} + \bm{a}_{l'})\circ\bm{r} \|_2^2.
        \big\}
    \end{align*}
    Moreover, we note that 
    \begin{align*}
        \| \bm{x}\circ\bm{r} - (\bm{\mu}_{l^{\ast}} + \bm{a}_{l^{\ast}})\circ\bm{r} \|_2^2 
        &= \|\bm{x}\circ\bm{r}-\bm{\mu}_{l^{\ast}}\circ\bm{r}\|_2^2 
        - 2\langle \bm{a}_{l^{\ast}}\circ\bm{r} , (\bm{x}-\bm{\mu}_{l^{\ast}})\circ\bm{r} \rangle 
        + \|\bm{a}_{l^{\ast}}\circ\bm{r}\|_2^2 \\
        &= \phi(\bm{x},\bm{r},\bm{M}) - 
        2\langle \bm{a}_{l^{\ast}} , (\bm{x}-\bm{\mu}_{l^{\ast}})\circ\bm{r} \rangle 
        + \|\bm{a}_{l^{\ast}}\circ\bm{r}\|_2^2
    \end{align*}
    \begin{align*}
        \| \bm{x}\circ\bm{r} - (\bm{\mu}_{l'} + \bm{a}_{l'})\circ\bm{r} \|_2^2 
        &= \|\bm{x}\circ\bm{r}-\bm{\mu}_{l'}\circ\bm{r}\|_2^2 
        - 2\langle \bm{a}_{l'}\circ\bm{r} , (\bm{x}-\bm{\mu}_{l'})\circ\bm{r} \rangle 
        + \|\bm{a}_{l'}\circ\bm{r}\|_2^2 \\
        &= \phi(\bm{x},\bm{r},\bm{M}) - 
        2\langle \bm{a}_{l'} , (\bm{x}-\bm{\mu}_{l'})\circ\bm{r} \rangle 
        + \|\bm{a}_{l'}\circ\bm{r}\|_2^2.
    \end{align*}
    Due to $\bm{\mu}_{l^{\ast}}\circ\bm{r}= \bm{\mu}_{l'}\circ\bm{r} $, it implies that the first-order partial derivatives of the distances with respect to $\bm{\mu}_{l^{\ast}}$ and $\bm{\mu}_{l'}$ coincide. Consequently, the first-order variation of $\phi$ is invariant to the choice between $l^{\ast}$ and $l'$. Thereby, the expansion given in Eq.(\ref{eq_lemma1_phi_decompose}) remains valid for any $(\bm{x},\bm{r})\in\mathcal{D}_3$. 
    
    Therefore, by combining the three cases, we obtain the the expansion given in Eq.(\ref{eq_lemma1_phi_decompose}) holds for $\widetilde{\mathbb{P}}$-almost all $(\bm{x},\bm{r})$.

    Next, to prove that $\phi(\cdot,\cdot,\bm{M})$ is differentiable with respect to $\bm{M}$ in quadratic mean, it suffices to prove $\int \left\{ R(\bm{x},\bm{r},\bm{M},\bm{A}) \right\}^2 \;d\widetilde{\mathbb{P}}(\bm{x},\bm{r}) \rightarrow 0$ as $\|\bm{A}\|_F\rightarrow 0$, i.e., the $\mathcal{L}^2(\widetilde{\mathbb{P}})$ convergence of $R(\cdot,\cdot,\bm{M},\bm{A})$.  
    Because 
    \begin{align*}
        0\leq R(\bm{x},\bm{r},\bm{M},\bm{A}) 
        \leq \frac{ \sum_{l=1}^{k} \| \bm{a}_l \|_2^2 }{ \|\bm{A}\|_F } 
        = \frac{ \|\bm{A}\|_F^2 }{ \|\bm{A}\|_F }
        = \|\bm{A}\|_F,
    \end{align*}
    then we have $R(\bm{x},\bm{r},\bm{M},\bm{A}) \rightarrow 0$ as $\|\bm{A}\|_F\rightarrow 0$ for any $(\bm{x},\bm{r})$, i.e., the almost everywhere convergence of $R(\cdot,\cdot,\bm{M},\bm{A})$ holds. 
    Thus, the $\mathcal{L}^2(\widetilde{\mathbb{P}})$ convergence of $R(\cdot,\cdot,\bm{M},\bm{A})$ can be obtained by using Dominated Convergence Theorem. 
    In fact, on one hand, $\mathcal{D}_2$ has a zero $\widetilde{\mathbb{P}}$-measure. On the other hand, for any $(\bm{x},\bm{r})\notin \mathcal{D}_2$, we can bound $R(\bm{x},\bm{r},\bm{M},\bm{A})$ as follows: 
    \begin{align*}
        &|R(\bm{x},\bm{r},\bm{M},\bm{A})|\\
        &= \bigg| \frac{1}{\|\bm{A}\|_F} \bigg\{ \phi(\bm{x},\bm{r},\bm{M}+\bm{A}) - \phi(\bm{x},\bm{r},\bm{M}) - \big\langle \text{vec}(\bm{A}), \text{vec}(\Delta(\bm{x},\bm{r},\bm{M}) ) \big\rangle \bigg\} \bigg| \\
        &\leq \frac{1}{\|\bm{A}\|_F} \left\{ 
        \left| \min_{l=1,\dots,k} \| \bm{x}\circ\bm{r} - (\bm{\mu}_l + \bm{a}_l)\circ \bm{r}  \|_2^2  - \min_{l=1,\dots,k} \| \bm{x}\circ\bm{r} - \bm{\mu}_l \circ \bm{r}  \|_2^2  \right| 
        + \bigg| \big\langle \text{vec}(\bm{A}), \text{vec}(\Delta(\bm{x},\bm{r},\bm{M}) ) \big\rangle \bigg|
        \right\}\\
        &\leq \frac{1}{\|\bm{A}\|_F} \sum_{l=1}^{k} \bigg| \| \bm{x}\circ\bm{r} - (\bm{\mu}_l + \bm{a}_l)\circ \bm{r}  \|_2^2 -  \| \bm{x}\circ\bm{r} - \bm{\mu}_l \circ \bm{r}  \|_2^2 \bigg| + \frac{1}{\|\bm{A}\|_F} \cdot \|\text{vec}(\bm{A})\|_2 \cdot \| \text{vec}(\Delta(\bm{x},\bm{r},\bm{M}) ) \|_2\\
        &\leq \frac{1}{\|\bm{A}\|_F} \sum_{l=1}^{k} \bigg| \| \bm{a}_l\circ\bm{r} \|_2^2 + 2\cdot \|\bm{a}_l\circ\bm{r}\|_2\cdot \| (\bm{x}-\bm{\mu}_l)\circ\bm{r} \|_2 \bigg| + \| \text{vec}(\Delta(\bm{x},\bm{r},\bm{M}) ) \|_2\\
        &\leq \|\bm{A}\|_F + \frac{2}{\|\bm{A}\|_F}\cdot \left( \sum_{l=1}^{k}\|\bm{a}_l\circ\bm{r}\|_2^2 \right)^{1/2}\cdot \left( \sum_{l=1}^{k} \| (\bm{x}-\bm{\mu}_l)\circ\bm{r} \|_2^2 \right)^{1/2} +  \| \text{vec}(\Delta(\bm{x},\bm{r},\bm{M}) ) \|_2 \\
        &\leq \|\bm{A}\|_F + 2 \left( \sum_{l=1}^{k} \| (\bm{x}-\bm{\mu}_l)\circ\bm{r} \|_2^2 \right)^{1/2} + \| \text{vec}(\Delta(\bm{x},\bm{r},\bm{M}) ) \|_2\\
        &= \|\bm{A}\|_F + 2 \left( \sum_{l=1}^{k} \| (\bm{x}-\bm{\mu}_l)\circ\bm{r} \|_2^2 \right)^{1/2} + \left( \sum_{l=1}^{k} \| -2\mathds{1}(\ell(\bm{x},\bm{r},\bm{M}))\cdot (\bm{x}\circ\bm{r} - \bm{\mu}_l \circ \bm{r}) \|_2^2 \right)^{1/2} \\
        &\leq \|\bm{A}\|_F + 4 \left( \sum_{l=1}^{k} \| (\bm{x}-\bm{\mu}_l)\circ\bm{r} \|_2^2 \right)^{1/2}.
    \end{align*}
    Because each $\bm{\mu}_l\in \mathcal{X}$, implying $\|\bm{\mu}_l\|_2\leq B$ under Assumption~\ref{assumption_X_compact}, then there must exist a constant $C>0$ such that 
    \begin{align*}
        &\sum_{l=1}^{k} \| (\bm{x}-\bm{\mu}_l)\circ\bm{r} \|_2^2 
        \leq  k\|\bm{x}\circ\bm{r}\|_2^2 + \sum_{l=1}^{k} \|\bm{\mu}_l\|_2^2 + 2\|\bm{x}\circ\bm{r} \|_2\cdot\sum_{l=1}^{k}\|\bm{\mu}_l\|_2
        \leq C^2(1+\| \bm{x}\circ\bm{r}\|_2)^2. 
    \end{align*}
    It follows that as $\|\bm{A}\|_F$ is small enough, we have
    \begin{align*}
        |R(\bm{x},\bm{r},\bm{M},\bm{A})| \leq C(1+\|\bm{x}\circ\bm{r}\|_2).
    \end{align*}
    Moreover, under Assumption~\ref{assumption_X_compact}, since $\mathbb{E}_{\bm{\mathrm{x}}_1}[\|\bm{\mathrm{x}}_1\|_2^2]<\infty$, then we have
    \begin{align*}
        \int C^2(1+\|\bm{x}\circ\bm{r}\|_2)^2 \;d\widetilde{\mathbb{P}}(\bm{x},\bm{r}) 
        \leq \int C^2(1+\|\bm{x}\|_2)^2 \;d\widetilde{\mathbb{P}}(\bm{x},\bm{r})
        = \int C^2(1+\|\bm{x}\|_2)^2 \; d\mathbb{P}(\bm{x})< \infty,
    \end{align*}
    which means $C(1+\|\cdot \circ\cdot \|_2) \in \mathcal{L}^2(\widetilde{\mathbb{P}})$. 
    Therefore, applying the Dominated Convergence Theorem leads to the $\mathcal{L}^2(\widetilde{\mathbb{P}})$ convergence of $R(\cdot,\cdot,\bm{M},\bm{A})$ to zero as $\|\bm{A}\|_F\rightarrow 0$.

    (2) Second, we turn to prove $L(\bm{M})$, i.e., $\widetilde{\mathbb{P}}\phi(\cdot,\cdot,\bm{M})$ is differentiable, which suffices to prove: there exists a matrix-valued function $\gamma(\bm{M})$ with respect to $\bm{M}$ such that for all $\bm{A}\in\mathbb{R}^{p\times k}$, 
    \begin{align*}
        \widetilde{\mathbb{P}}\phi (\cdot,\cdot, \bm{M}+\bm{A}) = \widetilde{\mathbb{P}}\phi (\cdot,\cdot, \bm{M}) + \big\langle \text{vec}(\bm{A}), \text{vec}(\gamma(\bm{M})) \big\rangle + o(\|\bm{A}\|_F).
    \end{align*}
    Note that according to Eq.(\ref{eq_lemma1_phi_decompose}), we can write
    \begin{align*}
        \widetilde{\mathbb{P}}\phi (\cdot,\cdot, \bm{M}+\bm{A}) 
        = \widetilde{\mathbb{P}}\phi (\cdot,\cdot, \bm{M}) + \big\langle \text{vec}(\bm{A}), \text{vec}(\widetilde{\mathbb{P}}\Delta(\cdot,\cdot, \bm{M}) ) \big\rangle
        + \|\bm{A}\|_F \cdot \widetilde{\mathbb{P}} R(\cdot,\cdot,\bm{M},\bm{A}).
    \end{align*}
    Moreover, based on the $\mathcal{L}^2(\widetilde{\mathbb{P}})$ convergence of the function $R(\cdot,\cdot,\bm{M},\bm{A})$, we have as $\|\bm{A}\|_F$ is small enough, 
    \begin{align*}
        \widetilde{\mathbb{P}} R(\cdot,\cdot,\bm{M},\bm{A})
        &= \int R(\bm{x},\bm{r},\bm{M},\bm{A}) \;d\widetilde{\mathbb{P}}(\bm{x},\bm{r}) \\
        &\leq \left( \int \left\{ R(\bm{x},\bm{r},\bm{M},\bm{A}) \right\}^2 \;d\widetilde{\mathbb{P}}(\bm{x},\bm{r}) \right)^{1/2} \cdot \left( \int 1^2 \;d\widetilde{\mathbb{P}}(\bm{x},\bm{r}) \right)^{1/2} \\
        &= o(1).
    \end{align*}
    Therefore, by taking $\gamma(\bm{M})=\widetilde{\mathbb{P}}\Delta(\cdot,\cdot, \bm{M})$, we obtain the differentiability of $\widetilde{\mathbb{P}}\phi(\cdot,\cdot,\bm{M})$. 
\end{proof}

\subsubsection{Proof of Lemma~\ref{lemma_Gn_1st_derivative}}
\begin{proof}
    According to the decomposition Eq.(\ref{eq_lemma1_phi_decompose}), for a fixed $\bm{M}\in\mathbb{R}^{p\times k}$ and a sequence of random matrices $\bm{\mathrm{V}}_n\in\mathbb{R}^{p\times k}$ with $\|\text{vec}(\bm{\mathrm{V}}_n - \bm{M})\|_2 = o_P(1)$, we have 
    \begin{align*}
        \mathbb{G}_n\phi(\cdot,\cdot,\bm{\mathrm{V}}_n)
        &= \mathbb{G}_n\phi(\cdot,\cdot,\bm{M}) + \big\langle \text{vec}(\bm{\mathrm{V}}_n - \bm{M}) \;,\; \text{vec}(\mathbb{G}_n\Delta(\cdot,\cdot,\bm{M})) \big\rangle \\
        &\quad + \| \bm{\mathrm{V}}_n - \bm{M} \|_F \cdot \mathbb{G}_n R(\cdot,\cdot,\bm{M},\bm{\mathrm{V}}_n - \bm{M}) .
    \end{align*}
    Thus, it suffices to prove $\mathbb{G}_n R(\cdot,\cdot,\bm{M},\bm{\mathrm{V}}_n - \bm{M}) = o_P(1)$, that is, for any $\delta \in (0,1)$ and $\eta > 0$, there exists $N\in \mathbb{N}_+$ such that for all $n\geq N$, $\text{Pr}\big( \big| \mathbb{G}_n R(\cdot,\cdot,\bm{M},\bm{\mathrm{V}}_n - \bm{M}) \big|\leq \eta \big) \geq 1-\delta$. 

    Since $\|\text{vec}(\bm{\mathrm{V}}_n - \bm{M})\|_2 = o_P(1)$, then according to Lemma~\ref{lemma_L_1st_derivative}, we have $\int \left\{ R(\bm{x},\bm{r},\bm{M},\bm{\mathrm{V}}_n - \bm{M}) \right\}^2 \; d\widetilde{\mathbb{P}}(\bm{x},\bm{r}) = o_P(1)$. That is, for any $\delta\in (0,1)$ and $\epsilon>0$, there exists $N_1\in \mathbb{N}_+$ such that for any $n\geq N_1$, 
    \begin{align*}
        \text{Pr}\bigg( \left| \int \left\{ R(\bm{x},\bm{r},\bm{M},\bm{\mathrm{V}}_n - \bm{M})  \right\}^2 \; d\widetilde{\mathbb{P}}(\bm{x},\bm{r})  \right| \leq \epsilon  \bigg) \geq 1-\frac{\delta}{2}.
    \end{align*}
    Moreover, we consider a function class as follows: 
    \begin{align*}
        \mathcal{G} = \bigg\{ R(\cdot,\cdot,\bm{M},\bm{A} - \bm{M}) \;\bigg|\;& \forall \bm{A} \text{ in a neighborhood of } \bm{M} \text{ satisfying } \\
        &|R(\bm{x},\bm{r},\bm{M},\bm{A} - \bm{M})| \leq C\cdot (1+\|\bm{x}\circ\bm{r}\|_2 ),\;\forall (\bm{x},\bm{r})\in \mathcal{X}\times \{0,1\}^p
        \bigg\},
    \end{align*}
    where $C>0$ is the constant defined in the proof of Lemma~\ref{lemma_L_1st_derivative}. 
    According to Lemma~\ref{lemma_donsker_class}, the $\mathcal{G}$ is a Donsker class. Moreover, based on the asymptotic equicontinuity of empirical process for Donsker class (Theorem~1.5.7 and Chapter~2.1.2 of \cite{van1996weak}), we have: For any $\delta\in (0,1)$ and $\eta>0$, there exists an $\epsilon>0$ and $N_2\in\mathbb{N}_+$ such that for any $n\geq N_2$, 
    \begin{align*}
        &\text{Pr}\left( \sup_{\|g_1 - g_2 \|_{\mathcal{L}^2(\widetilde{\mathbb{P}})} \leq \epsilon }  \bigg| \mathbb{G}_n g_1(\cdot,\cdot) - \mathbb{G}_n g_2(\cdot,\cdot) \bigg| \leq \eta \right) \geq 1- \frac{\delta}{2}\\
        \Longleftrightarrow & 
        \text{Pr}\bigg( \forall g_1,g_2\in\mathcal{G} \text{ satisfying } \int \left\{ g_1(\bm{x},\bm{r}) - g_2(\bm{x},\bm{r}) \right\}^2 \; d\widetilde{\mathbb{P}}(\bm{x},\bm{r}) \leq \epsilon,\; 
        \bigg| \mathbb{G}_n g_1(\cdot,\cdot) - \mathbb{G}_n g_2(\cdot,\cdot) \bigg| \leq \eta 
        \bigg) \\
        &\geq 1-\frac{\delta}{2}.
    \end{align*}
    Therefore, by taking $g_1(\cdot,\cdot)=R(\cdot,\cdot,\bm{M},\bm{\mathrm{V}}_n - \bm{M})$ and $g_2(\cdot,\cdot)=0$, we can obtain: For any $n\geq \max\{N_1,N_2\}$, 
    \begin{align*}
        \text{Pr}\bigg( \big| \mathbb{G}_n R(\cdot,\cdot,\bm{M},\bm{\mathrm{V}}_n - \bm{M}) \big| \leq \eta \bigg) \geq 1-\delta,
    \end{align*}
    which completes the proof. 
\end{proof}

\begin{lemma}
\label{lemma_donsker_class}
    For a fixed constant $C>0$ and a fixed $\bm{M}\in\mathcal{X}^k$, the following function class is a Donsker class: 
    \begin{align*}
        \mathcal{G} = \bigg\{ R(\cdot,\cdot,\bm{M},\bm{A} - \bm{M}) \;\bigg|\;& \forall \bm{A} \text{ in a neighborhood of } \bm{M} \text{ satisfying } \\
        &|R(\bm{x},\bm{r},\bm{M},\bm{A} - \bm{M})| \leq C\cdot (1+\|\bm{x}\circ\bm{r}\|_2 ),\;\forall (\bm{x},\bm{r})\in \mathcal{X}\times \{0,1\}^p
        \bigg\}.
    \end{align*}
\end{lemma}
\begin{proof}
    Since for any $(\bm{x},\bm{r})$, we can write 
    \begin{align*}
        R(\bm{x},\bm{r},\bm{M},\bm{A} - \bm{M}) 
        = \sum_{\bm{s}\in \{0,1\}^p} \mathds{1}(\bm{s}=\bm{r}) \cdot R(\bm{x},\bm{r},\bm{M},\bm{A} - \bm{M}).
    \end{align*}
    Then, we define a function class $\mathcal{G}^{\bm{s}}$ for any fixed $\bm{s}\in \{0,1\}^p$ by
    \begin{align*}
        \mathcal{G}^{\bm{s}} = \bigg\{ (\bm{x},\bm{r})&\mapsto  \mathds{1}(\bm{s}=\bm{r})\cdot R(\bm{x},\bm{r},\bm{M},\bm{A} - \bm{M}) \;\bigg|\; \forall \bm{A} \text{ in a neighborhood of } \bm{M} \text{ satisfying } \\
        &|\mathds{1}(\bm{s}=\bm{r})\cdot R(\bm{x},\bm{r},\bm{M},\bm{A} - \bm{M})| \leq C\cdot (1+\|\bm{x}\circ\bm{s}\|_2 ),\;\forall (\bm{x},\bm{r})\in \mathcal{X}\times \{0,1\}^p
        \bigg\}.
    \end{align*}
    On one hand, according to the result of classical $k$-means given by the Lemma~B of \cite{Pollard1982}, the following class of function on $\mathcal{X}$ is a Donsker class: 
    \begin{align*}
        \bigg\{ \bm{x} \mapsto R(\bm{x},\bm{s},\bm{M},\bm{A} - \bm{M}) \;\bigg|\;& \forall \bm{A} \text{ in a neighborhood of } \bm{M} \text{ satisfying } \\
        &|R(\bm{x},\bm{s},\bm{M},\bm{A} - \bm{M})| \leq C\cdot (1+\|\bm{x}\circ\bm{s}\|_2 ),\;\forall \bm{x} \in \mathcal{X}
        \bigg\},
    \end{align*}
    which implies the the following class of function on $\mathcal{X}\times \{0,1\}^p$ is also a Donsker class:
    \begin{align*}
         \bigg\{ (\bm{x},\bm{r}) \mapsto R(\bm{x},\bm{s},\bm{M},\bm{A} - \bm{M}) \;\bigg|\;& \forall \bm{A} \text{ in a neighborhood of } \bm{M} \text{ satisfying } \\
        &|R(\bm{x},\bm{s},\bm{M},\bm{A} - \bm{M})| \leq C\cdot (1+\|\bm{x}\circ\bm{s}\|_2 ),\;\forall (\bm{x},\bm{r}) \in \mathcal{X}\times \{0,1\}^p
        \bigg\}.
    \end{align*}
    On the other hand, because the function $(\bm{x},\bm{r})\mapsto \mathds{1}(\bm{s}=\bm{r})$ is bounded and measurable with respect to $\widetilde{\mathbb{P}}$, then according to Theorem~2.10.6 of \cite{van1996weak}, we have each $\mathcal{G}^{\bm{s}}$ is a Donsker class with a common envelope function $(\bm{x},\bm{r})\mapsto C(1+\|\bm{x}\|_2)\in \mathcal{L}^2(\widetilde{\mathbb{P}})$. 
    Moreover, let 
    \begin{align*}
        \widetilde{\mathcal{G}} = \left\{ \sum_{\bm{s}\in \{0,1\}^p } g^{\bm{s}} \;\bigg|\; g^{\bm{s}}\in \mathcal{S}^{\bm{s}},\; \bm{s}\in \{0,1\}^p  \right\},
    \end{align*}
    which is a Donsker class based on the Lipschitz transformation of Donsker classes (Theorem~2.10.6 of \cite{van1996weak}). 
    Finally, since $\mathcal{G}\subset \widetilde{\mathcal{G}}$, then according to Theorem~2.10.1 of \cite{van1996weak}, we have $\mathcal{G}$ is thus a Donsker class. 
\end{proof}

\subsubsection{Proof of Lemma~\ref{lemma_L_2nd_derivative}}
\begin{proof}
    According to Lemma~\ref{lemma_L_1st_derivative}, we know that $\widetilde{\mathbb{P}}\phi(\cdot,\cdot,\bm{M})$ has the derivative $\gamma(\bm{M})$, whose $l$-th column is given by $\widetilde{\mathbb{P}}\Delta_l(\cdot,\cdot,\bm{M})$. Under Assumption~\ref{assumption_mcar}, we have
    \begin{align*}
        \widetilde{\mathbb{P}}\Delta_l(\cdot,\cdot,\bm{M})
        &=\int -2\mathds{1}(\ell(\bm{x},\bm{r},\bm{M}) =l )\cdot (\bm{x}\circ\bm{r} - \bm{\mu}_l\circ\bm{r})\; d\widetilde{\mathbb{P}}(\bm{x},\bm{r}) \\
        &= -2 \sum_{\bm{r}\in \{0,1\}^p } \textnormal{Pr}(\mathrm{\bm{r}}_1=\bm{r}) \cdot \int \mathds{1}(\ell(\bm{x},\bm{r},\bm{M}) =l )\cdot (\bm{x}\circ\bm{r} - \bm{\mu}_l\circ\bm{r})\; d\mathbb{P}(\bm{x}).
    \end{align*}
    Thus, it suffices to prove: For any fixed $\bm{r}\in\{0,1\}^p$ and any fixed $l=1,\dots,k$, the following mapping is differentiable: 
    \begin{align*}
        \bm{M} \mapsto \int \mathds{1}(\ell(\bm{x},\bm{r},\bm{M}) =l )\cdot (\bm{x}\circ\bm{r} - \bm{\mu}_l\circ\bm{r})\; d\mathbb{P}(\bm{x}) .
    \end{align*}
    Moreover, consider the fixed $\bm{M}^{\ast}$, we have the following decomposition: 
    \begin{align*}
        &\int \mathds{1}(\ell(\bm{x},\bm{r},\bm{M}) =l )\cdot (\bm{x}\circ\bm{r} - \bm{\mu}_l\circ\bm{r})\; d\mathbb{P}(\bm{x}) \\
        &= \int \mathds{1}(\ell(\bm{x},\bm{r},\bm{M}) =l )\cdot (\bm{\mu}_l^{\ast}\circ\bm{r} - \bm{\mu}_l\circ\bm{r}) \; d\mathbb{P}(\bm{x}) + \int \mathds{1}(\ell(\bm{x},\bm{r},\bm{M}) =l )\cdot (\bm{x}\circ\bm{r} - \bm{\mu}_l^{\ast}\circ\bm{r}) \; d\mathbb{P}(\bm{x}) \\
        &= \underbrace{ \text{diag}(\bm{r})\cdot (\bm{\mu}_l^{\ast} - \bm{\mu}_l) \cdot \int \mathds{1}(\ell(\bm{x},\bm{r},\bm{M}) =l ) \; d\mathbb{P}(\bm{x}) }_{\text{(I)}} +  \underbrace{ \int \mathds{1}(\ell(\bm{x},\bm{r},\bm{M}) =l )\cdot (\bm{x}\circ\bm{r} - \bm{\mu}_l^{\ast}\circ\bm{r}) \; d\mathbb{P}(\bm{x}) }_{\text{(II)}}.
    \end{align*}
    For $\text{(I)}$, the derivative at $\bm{M}=\bm{M}^{\ast}$ is given by 
    \begin{align*}
        \frac{\partial}{\partial \bm{\mu}_l} \text{(I)} \bigg|_{\bm{M}=\bm{M}^{\ast}}
        &= \text{diag}(\bm{r}) \cdot (-1\cdot \bm{I}_p)\cdot \int \mathds{1}(\ell(\bm{x},\bm{r},\bm{M}^{\ast}) =l ) \; d\mathbb{P}(\bm{x}) \\
         &\quad +  \text{diag}(\bm{r})\cdot (\bm{\mu}_l^{\ast} - \bm{\mu}_l) \cdot \frac{\partial \int \mathds{1}(\ell(\bm{x},\bm{r},\bm{M}) =l ) \; d\mathbb{P}(\bm{x})  }{ \partial \bm{\mu}_l } \bigg|_{\bm{M}=\bm{M}^{\ast}}  \\
        &= -\text{diag}(\bm{r}) \cdot \int \mathds{1}(\ell(\bm{x},\bm{r},\bm{M}^{\ast}) =l ) \; d\mathbb{P}(\bm{x}) + 0\cdot \bm{I}_p\\
        &= -\text{diag}(\bm{r}) \cdot \int \mathds{1}(\ell(\bm{x},\bm{r},\bm{M}^{\ast}) =l ) \; d\mathbb{P}(\bm{x}) 
    \end{align*}
    and 
    \begin{align*}
        \frac{\partial}{\partial \bm{\mu}_{l'}} \text{(I)} \bigg|_{\bm{M}=\bm{M}^{\ast}}
        &= \text{diag}(\bm{r})\cdot (0\cdot \bm{I}_p) \cdot \int \mathds{1}(\ell(\bm{x},\bm{r},\bm{M}^{\ast}) =l ) \; d\mathbb{P}(\bm{x}) \\
        &\quad +  \text{diag}(\bm{r})\cdot (\bm{\mu}_l^{\ast} - \bm{\mu}_l) \cdot \frac{\partial \int \mathds{1}(\ell(\bm{x},\bm{r},\bm{M}) =l ) \; d\mathbb{P}(\bm{x})  }{ \partial \bm{\mu}_{l'} } \bigg|_{\bm{M}=\bm{M}^{\ast}} \\
        &= 0\cdot \bm{I}_p +  0\cdot \bm{I}_p\\
        &= 0\cdot \bm{I}_p.
    \end{align*}
    For $\text{(II)}$, we define a subset of $\bm{x}$ assigned to cluster $l$ by
    \begin{align*}
        \mathcal{C}_{l}^{\bm{r}}(\bm{M})=\{\bm{x}\;|\;\ell(\bm{x},\bm{r},\bm{M}) =l \},
    \end{align*}
    and define a vector-valued function of $\bm{M}$ to be
    \begin{align*}
        W_{l}^{\bm{r}}(\bm{M})= \int_{\bm{x}\in \mathcal{C}_{l}^{\bm{r}}(\bm{M})} (\bm{x}\circ\bm{r} - \bm{\mu}_l^{\ast}\circ\bm{r}) \; d\mathbb{P}(\bm{x}) ,
    \end{align*}
    then we have $\text{(II)}=W_{l}^{\bm{r}}(\bm{M})$.
    To derive the derivative of $\text{(II)}$ at $\bm{M}=\bm{M}^{\ast}$, we consider three cases of $\bm{M}$: 
    (1) The $\{ \bm{\mu}_1\circ\bm{r}\dots,\bm{\mu}_k\circ\bm{r}\}$ are distinct with each other; 
    (2) There exists $l'> l$ such that $\bm{\mu}_{l'}\circ\bm{r}=\bm{\mu}_{l}\circ\bm{r}$;
    (3) There exists $l'< l$ such that $\bm{\mu}_{l'}\circ\bm{r}=\bm{\mu}_{l}\circ\bm{r}$;
    
    For the case of (1), since only the integral region of $W_{l}^{\bm{r}}(\bm{M})$ is related to $\bm{M}$ while the integral function not, then under Assumptions~\ref{assumption_about_density_f}-\ref{assumption_about_S_ll'_integral} and by using the boundary integral, we can give the derivative of $W_{l}^{\bm{r}}(\bm{M})$ as follows: 
    \begin{align*}
        \frac{\partial W_{l}^{\bm{r}}(\bm{M})}{\partial \bm{\mu}_t}
        = \int_{ \bm{x}\in\partial \mathcal{C}_{l}^{\bm{r}}(\bm{M}) } f(\bm{x})\cdot (\bm{x}\circ\bm{r} - \bm{\mu}_l^{\ast}\circ\bm{r}) \cdot \bm{v}_{t}(\bm{x})^T\; dS(\bm{x}) ,\; \forall t=1,\dots,k,
    \end{align*}
    where $\partial \mathcal{C}_{l}^{\bm{r}}(\bm{M})$ is the boundary of the region $\mathcal{C}_{l}^{\bm{r}}(\bm{M})$, and $ f(\bm{x})$ is the density function of $\mathbb{P}$, and $\bm{v}_t\in \mathbb{R}^p$ is a vector denoting the velocity vector for motion of $\mathcal{C}_{l}^{\bm{r}}(\bm{M})$ orthogonal to its boundary $\partial \mathcal{C}_{l}^{\bm{r}}(\bm{M})$ evaluated at $\bm{M}=\bm{M}^{\ast}$, and $dS$ is the Lebesgue measure on the surface in the $\mathbb{R}^p$ space. 
    Because the surface $\partial \mathcal{C}_{l}^{\bm{r}}(\bm{M})$ consists of no more than $k-1$ hyperplanes, that is, 
    \begin{align*}
        &\partial \mathcal{C}_{l}^{\bm{r}}(\bm{M}) = \bigcup_{l'\neq l} \mathcal{S}_{ll'}^{\bm{r}}(\bm{M}) \\
        \text{where}\quad & \mathcal{S}_{ll'}^{\bm{r}}(\bm{M}) = \left\{ \bm{x} \;|\; \| \bm{x}\circ\bm{r} - \bm{\mu}_l\circ\bm{r} \|_2^2 =\| \bm{x}\circ\bm{r} - \bm{\mu}_{l'}\circ\bm{r} \|_2^2
        \right\},
    \end{align*}
    then we can write
    \begin{align*}
        \frac{\partial W_{l}^{\bm{r}}(\bm{M})}{\partial \bm{\mu}_t}
        = \sum_{l'\neq l} \int_{ \bm{x}\in\mathcal{S}_{ll'}^{\bm{r}}(\bm{M}) } f(\bm{x})\cdot (\bm{x}\circ\bm{r} - \bm{\mu}_l^{\ast}\circ\bm{r}) \cdot \bm{v}_{t}(\bm{x})^T\; dS(\bm{x}) ,\; \forall t=1,\dots,k.
    \end{align*}
    Moreover, for each hyperplane $\mathcal{S}_{ll'}^{\bm{r}}(\bm{M})$, denote by $\bm{n}_{ll'}^{\bm{r}}$ the unit normal pointing outward from $\mathcal{S}_{ll'}^{\bm{r}}(\bm{M})$, that is, 
    \begin{align*}
        \bm{n}_{ll'}^{\bm{r}} = \frac{ \bm{\mu}_{l'}\circ\bm{r} - \bm{\mu}_{l}\circ\bm{r} }{ \|\bm{\mu}_{l'}\circ\bm{r} - \bm{\mu}_{l}\circ\bm{r}\|_2 }.
    \end{align*}
    Then, for any $\bm{x}\in \mathcal{S}_{ll'}^{\bm{r}}(\bm{M})$, we have 
    \begin{align*}
        &\| \bm{x}\circ\bm{r} - \bm{\mu}_l\circ\bm{r} \|_2^2 - \| \bm{x}\circ\bm{r} - \bm{\mu}_{l'}\circ\bm{r} \|_2^2 = 0\\
        \Longleftrightarrow & (\bm{n}_{ll'}^{\bm{r}})^T\cdot \left( \bm{x}\circ\bm{r} - \frac{ \bm{\mu}_{l'}\circ\bm{r} + \bm{\mu}_{l}\circ\bm{r} }{ 2 } \right) =0\\
        \Longleftrightarrow & (\bm{n}_{ll'}^{\bm{r}})^T\cdot \left( \bm{x} - \frac{ \bm{\mu}_{l'} + \bm{\mu}_{l} }{ 2 }  \right) =0.
    \end{align*}
    Moreover, taking the total differential at $\bm{M}=\bm{M}^{\ast}$ for the final equation leads to
    \begin{align*}
        (\bm{n}_{ll'}^{\bm{r}})^T\cdot \left( d\bm{x} - \frac{ d\bm{\mu}_{l'} + d\bm{\mu}_{l} }{ 2 } \right)  + \left( \bm{x} - \frac{ \bm{\mu}_{l'}^{\ast} + \bm{\mu}_{l}^{\ast} }{ 2 }  \right)^T\cdot d\bm{n}_{ll'}^{\bm{r}} =0 ,
    \end{align*}
    which follows that
    \begin{align*}
        (\bm{n}_{ll'}^{\bm{r}})^T d\bm{x} 
        &= \frac{1}{2} (\bm{n}_{ll'}^{\bm{r}})^T (d\bm{\mu}_l + d\bm{\mu}_{l'}) - \left( \bm{x} - \frac{ \bm{\mu}_{l'}^{\ast} + \bm{\mu}_{l}^{\ast} }{ 2 }  \right)^T\cdot d\bm{n}_{ll'}^{\bm{r}} \\
        &= \frac{1}{ \|\bm{\mu}_{l'}^{\ast}\circ\bm{r} - \bm{\mu}_{l}^{\ast}\circ\bm{r}\|_2 } \cdot \left[ \frac{1}{2} \left( \bm{\mu}_{l'}^{\ast}\circ\bm{r} - \bm{\mu}_{l}^{\ast}\circ\bm{r} \right)^T (d\bm{\mu}_l + d\bm{\mu}_{l'}) \right. \\
        &\quad\quad\quad\quad\quad\quad\quad\quad\quad\quad\quad  \left. - \left\{ \left( \bm{x} - \frac{ \bm{\mu}_{l'}^{\ast} + \bm{\mu}_{l}^{\ast} }{ 2 }  \right)\circ \bm{r}\right\}^T (d\bm{\mu}_{l'} - d\bm{\mu}_{l}) \right] \\
        &= \frac{1}{ \| \bm{\mu}_{l'}^{\ast}\circ\bm{r} - \bm{\mu}_{l}^{\ast}\circ\bm{r}\|_2 } \cdot \left[ (\bm{x}\circ\bm{r} - \bm{\mu}_l^{\ast}\circ \bm{r} )^T d\bm{\mu}_l - (\bm{x}\circ\bm{r} - \bm{\mu}_{l'}^{\ast})\circ \bm{r})^T d\bm{\mu}_{l'}  \right].
    \end{align*}
    Thereby, we can obtain for any $\bm{x}\in\mathcal{S}_{ll'}^{\bm{r}}(\bm{M})$, 
    \begin{align*}
        \bm{v}_t(\bm{x}) = \left\{
        \begin{array}{ll}
              \|\bm{\mu}_{l'}\circ\bm{r} - \bm{\mu}_{l}\circ\bm{r}\|_2^{-1} \cdot (\bm{x}\circ\bm{r} - \bm{\mu}_l\circ \bm{r} )  & \text{ if } t = l \\
             -\|\bm{\mu}_{l'}\circ\bm{r} - \bm{\mu}_{l}\circ\bm{r}\|_2^{-1} \cdot (\bm{x}\circ\bm{r} - \bm{\mu}_{l'}\circ\bm{r} ) & \text{ if } t = l' \\
             \bm{0}_p & \text{ if } t \neq l,l'
        \end{array}
        \right.
    \end{align*}
    It follows that the derivative at $\bm{M}=\bm{M}^{\ast}$ is given by 
    \begin{align*}
        &\frac{\partial W_{l}^{\bm{r}}(\bm{M})}{\partial \bm{\mu}_l}\bigg|_{\bm{M}=\bm{M}^{\ast}} \\
        &= \sum_{t\neq l} \|\bm{\mu}_{t}^{\ast}\circ\bm{r} - \bm{\mu}_{l}^{\ast}\circ\bm{r}\|_2^{-1}\cdot 
        \int_{ \bm{x}\in\mathcal{S}_{lt}^{\bm{r}}(\bm{M}^{\ast}) } f(\bm{x})\cdot (\bm{x}\circ\bm{r} - \bm{\mu}_l^{\ast}\circ\bm{r}) \cdot 
        (\bm{x}\circ\bm{r} - \bm{\mu}_l^{\ast}\circ\bm{r} )^T
        \; dS(\bm{x}) \\
         &\frac{\partial W_{l}^{\bm{r}}(\bm{M})}{\partial \bm{\mu}_{l'}}\bigg|_{\bm{M}=\bm{M}^{\ast}} \\
         &= - \|\bm{\mu}_{l'}^{\ast}\circ\bm{r} - \bm{\mu}_{l}^{\ast}\circ\bm{r}\|_2^{-1}\cdot \int_{ \bm{x}\in\mathcal{S}_{ll'}^{\bm{r}}(\bm{M}^{\ast}) } f(\bm{x})\cdot (\bm{x}\circ\bm{r} - \bm{\mu}_l^{\ast}\circ\bm{r}) \cdot 
        (\bm{x}\circ\bm{r} - \bm{\mu}_{l'}^{\ast}\circ\bm{r} )^T
        \; dS(\bm{x}).
    \end{align*}

    For the case of (2), we have 
    \begin{align*}
        &\mathcal{C}_{l}^{\bm{r}}(\bm{M})
        =\{\bm{x}\;|\; \| \bm{x}\circ\bm{r} - \bm{\mu}_l\circ\bm{r} \|_2^2 < \| \bm{x}\circ\bm{r} - \bm{\mu}_t\circ\bm{r} \|_2^2,\;\forall t\neq l,l' \} \\
        &\mathcal{C}_{l'}^{\bm{r}}(\bm{M})
        =\emptyset.
    \end{align*}
    It follows that the boundary of $\mathcal{C}_{l}^{\bm{r}}(\bm{M})$ is given by 
    \begin{align*}
        \partial \mathcal{C}_{l}^{\bm{r}}(\bm{M}) 
        = \bigcup_{\substack{t\neq l\\ \bm{\mu}_t\circ\bm{r}\neq \bm{\mu}_l\circ\bm{r} }} \mathcal{S}_{lt}^{\bm{r}}(\bm{M}).
    \end{align*}
    Then, the calculation of derivative of $W_{l}^{\bm{r}}(\bm{M})$ is similar to the case of (1), which is given by 
    \begin{align*}
        &\frac{\partial W_{l}^{\bm{r}}(\bm{M})}{\partial \bm{\mu}_l}\bigg|_{\bm{M}=\bm{M}^{\ast}} \\
        &= \sum_{t\neq l} \frac{ \mathds{1}(\bm{\mu}_{t}^{\ast}\circ\bm{r} \neq \bm{\mu}_{l}^{\ast}\circ\bm{r}) }{ \|\bm{\mu}_{t}^{\ast}\circ\bm{r} - \bm{\mu}_{l}^{\ast}\circ\bm{r}\|_2 }\cdot 
        \int_{ \bm{x}\in\mathcal{S}_{lt}^{\bm{r}}(\bm{M}^{\ast}) } f(\bm{x})\cdot (\bm{x}\circ\bm{r} - \bm{\mu}_l^{\ast}\circ\bm{r}) \cdot 
        (\bm{x}\circ\bm{r} - \bm{\mu}_l^{\ast}\circ\bm{r} )^T
        \; dS(\bm{x}) \\
         &\frac{\partial W_{l}^{\bm{r}}(\bm{M})}{\partial \bm{\mu}_{l'}}\bigg|_{\bm{M}=\bm{M}^{\ast}} \\
         &= - \frac{ \mathds{1}(\bm{\mu}_{l'}^{\ast}\circ\bm{r} \neq \bm{\mu}_{l}^{\ast}\circ\bm{r})  }{ \|\bm{\mu}_{l'}^{\ast}\circ\bm{r} - \bm{\mu}_{l}^{\ast}\circ\bm{r}\|_2 }\cdot \int_{ \bm{x}\in\mathcal{S}_{ll'}^{\bm{r}}(\bm{M}^{\ast}) } f(\bm{x})\cdot (\bm{x}\circ\bm{r} - \bm{\mu}_l^{\ast}\circ\bm{r}) \cdot 
        (\bm{x}\circ\bm{r} - \bm{\mu}_{l'}^{\ast}\circ\bm{r} )^T
        \; dS(\bm{x}).
    \end{align*}

    For the case of (3), we have 
    \begin{align*}
        &\mathcal{C}_{l}^{\bm{r}}(\bm{M})
        =\emptyset \\
        &\mathcal{C}_{l'}^{\bm{r}}(\bm{M})
        =\{\bm{x}\;|\; \| \bm{x}\circ\bm{r} - \bm{\mu}_{l'}\circ\bm{r} \|_2^2 < \| \bm{x}\circ\bm{r} - \bm{\mu}_t\circ\bm{r} \|_2^2,\;\forall t\neq l,l' \}.
    \end{align*}
    It follows that $W_{l}^{\bm{r}}(\bm{M})=\bm{0}_p$, implying the derivative of $\partial W_{l}^{\bm{r}}(\bm{M})$ to be
    \begin{align*}
        \frac{\partial W_{l}^{\bm{r}}(\bm{M})}{\partial \bm{\mu}_l}
        = 0\cdot \bm{I}_p
        \quad\text{and}\quad 
        \frac{\partial W_{l}^{\bm{r}}(\bm{M})}{\partial \bm{\mu}_{l'}}  
        = 0\cdot \bm{I}_p.
    \end{align*}

    Therefore, we obtained the differentiability of $\text{(II)}$ at $\bm{M}=\bm{M}^{\ast}$, which implies the differentiability of $\widetilde{\mathbb{P}}\Delta_l(\cdot,\cdot,\bm{M})$. 
    
    Finally, combining with the derivative of $\text{(I)}$, we can obtain the derivative of $\widetilde{\mathbb{P}}\Delta_l(\cdot,\cdot,\bm{M})$ at $\bm{M}=\bm{M}^{\ast}$ as follows: 
    \begin{align*}
        &\frac{\partial  \widetilde{\mathbb{P}}\Delta_l(\cdot,\cdot,\bm{M}) }{\partial \bm{\mu}_l } \bigg|_{\bm{M}=\bm{M}^{\ast}} \\
        &= \sum_{\bm{r}\in \{0,1\}^p } \text{Pr}(\bm{\mathrm{r}}_1=\bm{r}) \cdot \bigg[  
        2\cdot\text{diag}(\bm{r})\cdot \int \mathds{1}(\ell(\bm{x},\bm{r},\bm{M}^{\ast})=l ) \; d\mathbb{P}(x) \\
        &\quad \left. -2 \sum_{t\neq l} \frac{ \mathds{1}(\bm{\mu}_{t}^{\ast}\circ\bm{r} \neq \bm{\mu}_{l}^{\ast}\circ\bm{r}) }{ \|\bm{\mu}_{t}^{\ast}\circ\bm{r} - \bm{\mu}_{l}^{\ast}\circ\bm{r}\|_2 }\cdot 
        \int_{ \bm{x}\in\mathcal{S}_{lt}^{\bm{r}}(\bm{M}^{\ast}) } f(\bm{x})\cdot (\bm{x}\circ\bm{r} - \bm{\mu}_l^{\ast}\circ\bm{r}) \cdot 
        (\bm{x}\circ\bm{r} - \bm{\mu}_l^{\ast}\circ\bm{r} )^T
        \; dS(\bm{x})
        \right]
    \end{align*}
    and 
    \begin{align*}
        &\frac{\partial  \widetilde{\mathbb{P}}\Delta_l(\cdot,\cdot,\bm{M}) }{\partial \bm{\mu}_{l'} } \bigg|_{\bm{M}=\bm{M}^{\ast}} \\
        &= \sum_{\bm{r}\in \{0,1\}^p } \text{Pr}(\bm{\mathrm{r}}_1=\bm{r}) \cdot \bigg[ 
        0\cdot \bm{I}_p  \\
        &\quad \left. + 2\cdot \frac{ \mathds{1}(\bm{\mu}_{l'}^{\ast}\circ\bm{r} \neq \bm{\mu}_{l}^{\ast}\circ\bm{r}) }{ \|\bm{\mu}_{l'}^{\ast}\circ\bm{r} - \bm{\mu}_{l}^{\ast}\circ\bm{r}\|_2 }\cdot \int_{ \bm{x}\in\mathcal{S}_{ll'}^{\bm{r}}(\bm{M}^{\ast}) } f(\bm{x})\cdot (\bm{x}\circ\bm{r} - \bm{\mu}_l^{\ast}\circ\bm{r}) \cdot 
        (\bm{x}\circ\bm{r} - \bm{\mu}_{l'}^{\ast}\circ\bm{r} )^T
        \; dS(\bm{x})
        \right],
    \end{align*}
    which is continuous only at $\bm{M}^{\ast}$ satisfying: $\bm{\mu}_{l}^{\ast}\circ\bm{r} \neq \bm{\mu}_{l'}^{\ast}\circ\bm{r}$ for any $l\neq l'$ and any non-all-zero vector $\bm{r}\in \{0,1\}^p$. 
    We complete the proof. 
\end{proof}

\subsubsection{Proof of Theorem~\ref{theorem_quadratic_approximation}}
\begin{proof}
    According to Lemma~\ref{lemma_L_1st_derivative} and Lemma~\ref{lemma_L_2nd_derivative} and under Assumption~\ref{assumption_unique_minimizer} and \ref{assumption_M_star_distinct}, we have for any $\bm{\mathrm{V}}_n$ in a small neighborhood of $\bm{M}^{\ast}$, the $L(\bm{M})$ has a quadratic approximation as follows: 
    \begin{align*}
        L(\bm{\mathrm{V}}_n) &= \widetilde{\mathbb{P}}\phi(\cdot,\cdot,\bm{\mathrm{V}}_n) \\
        &=\widetilde{\mathbb{P}}\phi(\cdot,\cdot,\bm{M}^{\ast}) + \big\langle  \text{vec}(\bm{\mathrm{V}}_n - \bm{M}^{\ast} ) \;,\;  \text{vec}(\widetilde{\mathbb{P}}\Delta(\cdot,\cdot,\bm{M}^{\ast})) \big\rangle \\
        &\quad + \frac{1}{2} \text{vec}(\bm{\mathrm{V}}_n - \bm{M}^{\ast} )^T\cdot \bm{\mathrm{\Gamma}}(\bm{M}^{\ast})\cdot \text{vec}(\bm{\mathrm{V}}_n - \bm{M}^{\ast} ) + o_P(\|\text{vec}(\bm{\mathrm{V}}_n - \bm{M}^{\ast} )\|_2^2),
    \end{align*}
    where $\widetilde{\mathbb{P}}\Delta(\cdot,\cdot,\bm{M}^{\ast})=0\cdot\bm{I}_p$ because $\bm{M}^{\ast}$ is the minimizer of $L(\bm{M})$. 
    Moreover, let $\bm{\xi}_n=-\text{vec}(\mathbb{G}_n\Delta(\cdot,\cdot,\bm{M}^{\ast}))$, then according to Lemma~\ref{lemma_Gn_1st_derivative}, we have 
    \begin{align*}
        \mathbb{G}_n\phi(\cdot,\cdot,\bm{\mathrm{V}}_n) = \mathbb{G}_n\phi(\cdot,\cdot,\bm{M}^{\ast}) + \big\langle \text{vec}(\bm{\mathrm{V}}_n - \bm{M}^{\ast}) \;,\; -\bm{\xi}_n  \big\rangle + o_P(\|\text{vec}(\bm{\mathrm{V}}_n - \bm{M}^{\ast} )\|_2).
    \end{align*}
    Thereby, we have 
    \begin{align*}
        &\widehat{L}_n(\bm{\mathrm{V}}_n) \\
        &= L(\bm{\mathrm{V}}_n) + \frac{1}{\sqrt{n}} \mathbb{G}_n\phi(\cdot,\cdot,\bm{\mathrm{V}}_n)\\
        &= \left\{ \widetilde{\mathbb{P}}\phi(\cdot,\cdot,\bm{M}^{\ast}) + \frac{1}{\sqrt{n}} \mathbb{G}_n\phi(\cdot,\cdot,\bm{M}^{\ast})  \right\} + \frac{1}{\sqrt{n}}\big\langle \text{vec}(\bm{\mathrm{V}}_n - \bm{M}^{\ast}) \;,\; -\bm{\xi}_n  \big\rangle \\
        &\quad + \frac{1}{2} \text{vec}(\bm{\mathrm{V}}_n - \bm{M}^{\ast} )^T\cdot \bm{\mathrm{\Gamma}}(\bm{M}^{\ast})\cdot \text{vec}(\bm{\mathrm{V}}_n - \bm{M}^{\ast} ) \\
        &\quad + o_P(n^{-1/2}\|\text{vec}(\bm{\mathrm{V}}_n - \bm{M}^{\ast} )\|_2)
        + o_P(\|\text{vec}(\bm{\mathrm{V}}_n - \bm{M}^{\ast} )\|_2^2)\\
        &= \widehat{L}_n(\bm{M}^{\ast}) - n^{-1/2} \bm{\xi}_n^T\cdot \text{vec}(\bm{\mathrm{V}}_n - \bm{M}^{\ast}) + \frac{1}{2} \text{vec}(\bm{\mathrm{V}}_n - \bm{M}^{\ast} )^T\cdot \bm{\mathrm{\Gamma}}(\bm{M}^{\ast})\cdot \text{vec}(\bm{\mathrm{V}}_n - \bm{M}^{\ast} ) \\
        &\quad + o_P(n^{-1/2}\|\text{vec}(\bm{\mathrm{V}}_n - \bm{M}^{\ast} )\|_2)
        + o_P(\|\text{vec}(\bm{\mathrm{V}}_n - \bm{M}^{\ast} )\|_2^2).
    \end{align*}
    
    For $\bm{\xi}_n=-\text{vec}(\mathbb{G}_n\Delta(\cdot,\cdot,\bm{M}^{\ast}))$, which is a random vector in $\mathbb{R}^{kp}$, according to the Central Limitation Theory, we have $\bm{\xi}_n$ has an asymptotical normal distribution with the mean vector being $\bm{0}_{kp}$ because $\widetilde{\mathbb{P}}\Delta (\cdot,\cdot,\bm{M}^{\ast})=0\cdot \bm{I}_p$. 
    Moreover, the asymptotic variance matrix $\bm{\Xi}\in \mathbb{R}^{kp\times kp}$ is given by 
    \begin{align*}
       \widetilde{\mathbb{P}}\text{vec}(\Delta(\cdot,\cdot,\bm{M}^{\ast}))\text{vec}(\Delta(\cdot,\cdot,\bm{M}^{\ast}))^T,
    \end{align*}
    which consists of $k^2$ blocks and each block is a matrix in $\mathbb{R}^{p\times p}$. Since the $l$-th column of $\Delta(\cdot,\cdot,\bm{M})$ is only related to $(\bm{x},\bm{r})$ satisfying $\ell(\bm{x},\bm{r},\bm{M}^{\ast})=l$, then for any $l'\neq l$, the $(l,l')$-th block of $\bm{\Xi}$ i.e., $\int \Delta_l(\bm{x},\bm{r},\bm{M}^{\ast}) \Delta_{l'}(\bm{x},\bm{r},\bm{M}^{\ast})^T\; d\widetilde{\mathbb{P}}(\bm{x},\bm{r})$ is a zero matrix. 
    It implies that $\bm{\Xi}$ is a block diagonal matrix and the $l$-th block is given by 
    \begin{align*}
        \bm{\Xi}_l 
        &= \int \Delta_l(\bm{x},\bm{r},\bm{M}^{\ast}) \Delta_l(\bm{x},\bm{r},\bm{M}^{\ast})^T \; d\widetilde{\mathbb{P}}(\bm{x},\bm{r}) \\
        &= 4 \int \mathds{1}(\ell(\bm{x},\bm{r},\bm{M}^{\ast})=l)\cdot (\bm{x}\circ\bm{r} - \bm{\mu}_l^{\ast}\circ\bm{r})\cdot (\bm{x}\circ\bm{r} - \bm{\mu}_l^{\ast}\circ\bm{r})^T \; d\widetilde{\mathbb{P}}(\bm{x},\bm{r}), 
    \end{align*}
    which completes the proof. 
\end{proof}

\subsubsection{Proof of Corollary~\ref{corollary_sqrt_n_rate_asymptotic_normality}}
\begin{proof}
    For simplification of notation, we write $\bm{\mathrm{\Gamma}}=\bm{\mathrm{\Gamma}}(\bm{M}^{\ast})$ through this proof. 
    
    Since $\| \text{vec}(\widehat{\bm{\mathrm{M}}}_n - \bm{M}^{\ast})\|_2 = o_P(1)$ and $\widehat{L}_n(\widehat{\bm{\mathrm{M}}}_n)\leq \widehat{L}_n(\bm{M}^{\ast})$, then using Theorem~\ref{theorem_quadratic_approximation} leads to
    \begin{align*}
        n^{-1/2} \bm{\xi}_n^T\text{vec}(\widehat{\bm{\mathrm{M}}}_n - \bm{M}^{\ast}) &\geq \frac{1}{2} \text{vec}(\widehat{\bm{\mathrm{M}}}_n - \bm{M}^{\ast})^T\cdot \bm{\mathrm{\Gamma}}\cdot \text{vec}(\widehat{\bm{\mathrm{M}}}_n - \bm{M}^{\ast}) \\
        &\quad + o_P(n^{-1/2}\|\text{vec}(\widehat{\bm{\mathrm{M}}}_n - \bm{M}^{\ast})\|_2) + o_P(\|\text{vec}(\widehat{\bm{\mathrm{M}}}_n - \bm{M}^{\ast})\|_2^2).
    \end{align*}
    Because 
    \begin{align*}
        & n^{-1/2} \bm{\xi}_n^T\text{vec}(\widehat{\bm{\mathrm{M}}}_n - \bm{M}^{\ast}) \leq n^{-1/2} \|\bm{\xi}_n\|_2 \cdot \|\text{vec}(\widehat{\bm{\mathrm{M}}}_n - \bm{M}^{\ast})\|_2\\
        \text{and } & \frac{1}{2} \text{vec}(\widehat{\bm{\mathrm{M}}}_n - \bm{M}^{\ast})^T\cdot \bm{\mathrm{\Gamma}}\cdot \text{vec}(\widehat{\bm{\mathrm{M}}}_n - \bm{M}^{\ast}) \geq \frac{\lambda_{\min}}{2} \|\text{vec}(\widehat{\bm{\mathrm{M}}}_n - \bm{M}^{\ast})\|_2^2,
    \end{align*}
    where $\lambda_{\min}$ is the minimum singular value of $\bm{\mathrm{\Gamma}}$ and $\lambda_{\min}>0$ under Assumption~\ref{assumption_positive_definite_Gamma}, then we have 
    \begin{align*}
        & n^{-1/2} \|\bm{\xi}_n\|_2 \cdot \|\text{vec}(\widehat{\bm{\mathrm{M}}}_n - \bm{M}^{\ast})\|_2 \\
        & \geq \frac{\lambda_{\min}}{2} \|\text{vec}(\widehat{\bm{\mathrm{M}}}_n - \bm{M}^{\ast})\|_2^2 + o_P(n^{-1/2}\|\text{vec}(\widehat{\bm{\mathrm{M}}}_n - \bm{M}^{\ast})\|_2) + o_P(\|\text{vec}(\widehat{\bm{\mathrm{M}}}_n - \bm{M}^{\ast})\|_2^2),
    \end{align*}
    that is, 
    \begin{align*}
        n^{-1/2} \|\bm{\xi}_n\|_2 \geq \frac{\lambda_{\min}}{2} \|\text{vec}(\widehat{\bm{\mathrm{M}}}_n - \bm{M}^{\ast})\|_2 + o_P(n^{-1/2}) + o_P(\|\text{vec}(\widehat{\bm{\mathrm{M}}}_n - \bm{M}^{\ast})\|_2).
    \end{align*}
    Since $\bm{\xi}_n\xrightarrow{d}\mathcal{N}(\bm{0}_{kp},\bm{\mathrm{\Xi}})$, then we have $\bm{\xi}_n=O_P(1)$, which leads to $\|\text{vec}(\widehat{\bm{\mathrm{M}}}_n - \bm{M}^{\ast})\|_2=O_P(n^{-1/2})$. 
    
    Moreover, let $\bm{a}_n= \sqrt{n} \text{vec}(\widehat{\bm{\mathrm{M}}}_n - \bm{M}^{\ast})$, then applying Theorem~\ref{theorem_quadratic_approximation} to $\widehat{\bm{\mathrm{M}}}_n$ and under Assumption~\ref{assumption_positive_definite_Gamma} leads to
    \begin{align*}
        \widehat{L}_n(\widehat{\bm{\mathrm{M}}}_n)
        &= \widehat{L}_n(\bm{M}^{\ast}) - \frac{1}{n} \bm{\xi}_n^T \bm{a}_n + \frac{1}{2n} \bm{a}_n^T \bm{\Gamma} \bm{a}_n + o_P(n^{-1}) \\
        &= \widehat{L}_n(\bm{M}^{\ast}) - \frac{1}{2n} \bm{\xi}_n^T \bm{\Gamma}^{-1} \bm{\xi}_n + \frac{1}{2n} \| \bm{\Gamma}^{1/2} \bm{a}_n - \bm{\Gamma}^{-1/2} \bm{\xi}_n \|_2^2 + o_P(n^{-1}).
    \end{align*}
    Applying Theorem~\ref{theorem_quadratic_approximation} to the sequence $\{ \bm{M}^{\ast} + n^{-1/2} \bm{\Gamma}\bm{\xi}_n \}$ leads to 
    \begin{align*}
        &\widehat{L}_n(\bm{M}^{\ast} + n^{-1/2} \bm{\Gamma}\bm{\xi}_n ) \\
        &= \widehat{L}_n(\bm{M}^{\ast}) - n^{-1/2}\bm{\xi}_n^T \left(n^{-1/2}\bm{\Gamma}^{-1}\bm{\xi}_n\right) + \frac{1}{2} \left(n^{-1/2}\bm{\Gamma}^{-1}\bm{\xi}_n\right)^T \cdot \bm{\Gamma}\cdot \left(n^{-1/2}\bm{\Gamma}^{-1}\bm{\xi}_n\right) + o_P(n^{-1})\\
        &= \widehat{L}_n(\bm{M}^{\ast}) - \frac{1}{2n} \bm{\xi}_n^T\bm{\Gamma}^{-1}\bm{\xi}_n + o_P(n^{-1}).
    \end{align*}
    It follows that 
    \begin{align*}
        \widehat{L}_n(\widehat{\bm{\mathrm{M}}}_n) = \widehat{L}_n(\bm{M}^{\ast} + n^{-1/2} \bm{\Gamma}\bm{\xi}_n ) + \frac{1}{2n} \| \bm{\Gamma}^{1/2} \bm{a}_n - \bm{\Gamma}^{-1/2} \bm{\xi}_n \|_2^2 + o_P(n^{-1}).
    \end{align*}
    Because $\widehat{L}_n(\widehat{\bm{\mathrm{M}}}_n)\leq \widehat{L}_n(\bm{M}^{\ast} + n^{-1/2} \bm{\Gamma}\bm{\xi}_n ) $, then we obtain
    \begin{align*}
        \frac{1}{2n} \| \bm{\Gamma}^{1/2} \bm{a}_n - \bm{\Gamma}^{-1/2} \bm{\xi}_n \|_2^2 = o_P(n^{-1}),
    \end{align*}
    which implies $\bm{a}_n = \bm{\Gamma}^{-1}\bm{\xi}_n +o_P(1)$. Since $\bm{\xi}_n\xrightarrow{d}\mathcal{N}(\bm{0}_{kp},\bm{\mathrm{\Xi}})$, we have $\bm{a}_n\xrightarrow{d}\mathcal{N}(\bm{0}_{kp},\bm{\Gamma}^{-1}\bm{\mathrm{\Xi}}\bm{\Gamma}^{-1})$, which completes the proof. 
\end{proof}

\subsection{Proof of Section~\ref{sec_converge_to_truth}}

\subsubsection{Proof of Lemma~\ref{lemma_converge_to_truth_condition_centers_bound}}
\begin{proof}
    Consider a fixed $n$ large enough to satisfy the condition in Lemma~\ref{lemma_converge_to_truth_condition_centers_bound}. 
    For any $i=1,\dots,n$, we define the assigned true cluster label of $\bm{\mathrm{x}}_i$ to be
    \begin{align*}
        \mathrm{z}_i^{\ast\ast}=\mathop{\arg\min}_{l=1,\dots,k} \| \bm{\mathrm{x}}_i - \bm{\mu}_l^{\ast\ast} \|_2,
    \end{align*}
    where if the minimizer is multiple, then take the smaller one. 
    Moreover, we define
    \begin{align*}
        \widetilde{D}(\widehat{\bm{\mathrm{M}}}_n,\bm{M}^{\ast\ast})= \left( \frac{1}{n} \sum_{i=1}^{n} \min_{t=1,\dots,k} \| (\bm{\mu}_{\mathrm{z}_i^{\ast\ast}}^{\ast\ast} - \hat{\bm{\mu}}_t)\circ \bm{\mathrm{r}}_i \|_2^2 \right)^{1/2},
    \end{align*}
    and we claim that it has the following bounds: 
    \begin{itemize}
        \item[(1)] $\widetilde{D}(\widehat{\bm{\mathrm{M}}}_n,\bm{M}^{\ast\ast})\geq \min\left\{ \sqrt{\frac{\mathrm{n}_{\min}^{\textnormal{comp}}}{n}} D(\widehat{\bm{\mathrm{M}}}_n,\bm{M}^{\ast\ast}),\; \sqrt{\frac{\mathrm{n}_{\min}^{\textnormal{feature}}}{n}} \frac{\rho^{\ast\ast}}{2} \right\}$;
        \item[(2)] $\widetilde{D}(\widehat{\bm{\mathrm{M}}}_n,\bm{M}^{\ast\ast})\leq 2\mathrm{b}_n$.
    \end{itemize}
    Thus, it suffices to prove the two claims, respectively. 

    To prove the claim (1), we consider two cases:
    \begin{itemize}
        \item[(a)] There exists a permutation $\tilde{\pi}$ such that for any $l=1,\dots,k$ and any $\bm{r}\neq \bm{0}_p$, 
    \begin{align*}
        \min_{t=1,\dots,k} \| ( \bm{\mu}_l^{\ast\ast} - \hat{\bm{\mu}}_{t} )\circ \bm{r} \|_2 
        = \| ( \bm{\mu}_l^{\ast\ast} - \hat{\bm{\mu}}_{\tilde{\pi}(l)} )\circ \bm{r} \|_2,
    \end{align*}
    which is equivalent to: for any $l=1,\dots,k$ and any $j\in\{1,\dots,p\}$,
    \begin{align*}
        \min\limits_{t=1,\dots,k} |\mu_{lj}^{\ast\ast} - \hat{\mu}_{tj}|=|\mu_{lj}^{\ast\ast} - \hat{\mu}_{\tilde{\pi}(l),j}|.
    \end{align*}
        \item[(b)] Otherwise, for any permutation $\pi$, there exists $l_1\in\{1,\dots,k\}$ and $\bm{r}_0\neq \bm{0}_p$ such that 
        \begin{align*}
            \min_{t=1,\dots,k} \| ( \bm{\mu}_{l_1}^{\ast\ast} - \hat{\bm{\mu}}_{t} )\circ \bm{r}_0 \|_2 
            <\| ( \bm{\mu}_{l_1}^{\ast\ast} - \hat{\bm{\mu}}_{\tilde{\pi}(l)} )\circ \bm{r}_0 \|_2,
        \end{align*}
        which is equivalent to: there exist $l_1\neq l_2$ and $j_0 \in \{1,\dots,p\}$ and $t_0\in\{1,\dots,k\}$ such that 
        \begin{align*}
            \min\limits_{t=1,\dots,k} |\mu_{l_1,j_0}^{\ast\ast} - \hat{\mu}_{t,j_0}|
            =|\mu_{l_1,j_0}^{\ast\ast} - \hat{\mu}_{t_0,j_0}|
            \quad \textnormal{and} \quad
             \min\limits_{t=1,\dots,k} |\mu_{l_2,j_0}^{\ast\ast} - \hat{\mu}_{t,j_0}|
            =|\mu_{l_2,j_0}^{\ast\ast} - \hat{\mu}_{t_0,j_0}|.
        \end{align*}
    \end{itemize}

    In the case (a), we have 
    \begin{align*}
        &\widetilde{D}(\widehat{\bm{\mathrm{M}}}_n,\bm{M}^{\ast\ast}) \\
        &=\left\{ \frac{1}{n}\sum_{\bm{r}\in\{0,1\}^p} \sum_{l=1}^{k} \sum_{\substack{i:\bm{\mathrm{r}}_i=\bm{r},\\ \mathrm{z}_i^{\ast\ast}=l}} \| ( \bm{\mu}_{l}^{\ast\ast} - \hat{\bm{\mu}}_{\tilde{\pi}(l)} )\circ \bm{r} \|_2^2   \right\}^{1/2} \\
        &\geq \left\{ \frac{1}{n} \sum_{l=1}^{k} \sum_{\substack{i:\bm{\mathrm{r}}_i=\bm{1}_p,\\ \mathrm{z}_i^{\ast\ast}=l}} \| ( \bm{\mu}_{l}^{\ast\ast} - \hat{\bm{\mu}}_{\tilde{\pi}(l)} )\circ \bm{1}_p \|_2^2   \right\}^{1/2} 
        = \left\{ \sum_{l=1}^{k} \frac{\#\left\{i\;|\; \bm{\mathrm{r}}_i=\bm{1}_p,\mathrm{z}_i^{\ast\ast}=l  \right\}  }{n} \cdot \| \bm{\mu}_{l}^{\ast\ast} - \hat{\bm{\mu}}_{\tilde{\pi}(l)}  \|_2^2 \right\}^{1/2} \\
        &\geq \left\{ \frac{\mathrm{n}_{\min}^{\textnormal{comp}} }{n} \cdot \sum_{l=1}^{k} \| \bm{\mu}_{l}^{\ast\ast} - \hat{\bm{\mu}}_{\tilde{\pi}(l)}  \|_2^2 \right\}^{1/2} \\
        &\geq \left\{ \frac{\mathrm{n}_{\min}^{\textnormal{comp}} }{n} \cdot \max_{l=1,\dots,k} \| \bm{\mu}_{l}^{\ast\ast} - \hat{\bm{\mu}}_{\tilde{\pi}(l)}  \|_2^2 \right\}^{1/2} 
        = \sqrt{\frac{\mathrm{n}_{\min}^{\textnormal{comp}} }{n}} \max_{l=1,\dots,k} \| \bm{\mu}_{l}^{\ast\ast} - \hat{\bm{\mu}}_{\tilde{\pi}(l)}  \|_2 \\
        &\geq \sqrt{\frac{\mathrm{n}_{\min}^{\textnormal{comp}} }{n}} \min_{\pi} \max_{l=1,\dots,k} \| \bm{\mu}_{l}^{\ast\ast} - \hat{\bm{\mu}}_{\pi(l)}  \|_2 
        = \sqrt{\frac{\mathrm{n}_{\min}^{\textnormal{comp}} }{n}} D(\widehat{\bm{\mathrm{M}}}_n,\bm{M}^{\ast\ast}),
    \end{align*}
    where $\mathrm{n}_{\min}^{\textnormal{comp}}=\min\limits_{l=1,\dots,k} \sum_{i=1}^{n} \mathds{1}\left(\bm{\mathrm{r}}_i=\bm{1}_p,\mathrm{z}_i^{\ast\ast}=l  \right)$.

    In the case (b), because 
    \begin{align*}
       |\mu_{l_1,j_0}^{\ast\ast} - \mu_{l_2,j_0}^{\ast\ast} | \leq |\mu_{l_1,j_0}^{\ast\ast} - \hat{\mu}_{t_0,j_0}| + |\mu_{l_2,j_0}^{\ast\ast} - \hat{\mu}_{t_0,j_0}|,
    \end{align*}
    without loss of generality, we can suppose 
    \begin{align}
        |\mu_{l_1,j_0}^{\ast\ast} - \hat{\mu}_{t_0,j_0}| \geq \frac{1}{2} |\mu_{l_1,j_0}^{\ast\ast} - \mu_{l_2,j_0}^{\ast\ast} |.
    \end{align}
    Then we have 
    \begin{align*}
        &\widetilde{D}(\widehat{\bm{\mathrm{M}}}_n,\bm{M}^{\ast\ast}) \\
        &=\left\{ \frac{1}{n} \sum_{l=1}^{k} \sum_{i:\mathrm{z}_i^{\ast\ast}=l} \min_{t=1,\dots,k} \| (\bm{\mu}_{l}^{\ast\ast} - \hat{\bm{\mu}}_{t} )\circ\bm{\mathrm{r}}_i \|_2^2 \right\}^{1/2}\\
        &\geq \left\{ \frac{1}{n} \sum_{i:\mathrm{z}_i^{\ast\ast}=l_1} \min_{t=1,\dots,k} \| (\bm{\mu}_{l_1}^{\ast\ast} - \hat{\bm{\mu}}_{t} )\circ\bm{\mathrm{r}}_i \|_2^2  \right\}^{1/2} 
        = \left\{ \frac{1}{n} \sum_{i:\mathrm{z}_i^{\ast\ast}=l_1} \min_{t=1,\dots,k} \sum_{j=1}^{p} \mathrm{r}_{ij}\cdot (\mu_{l_1,j}^{\ast\ast} - \hat{\mu}_{tj} )^2  \right\}^{1/2}\\
        &\geq \left\{ \frac{1}{n} \sum_{i:\mathrm{z}_i^{\ast\ast}=l_1}  \min_{t=1,\dots,k} \mathrm{r}_{i,j_0} \cdot ( \mu_{l_1,j_0}^{\ast\ast} - \hat{\mu}_{t,j_0} )^2 \right\}^{1/2} \\
        &\geq \left\{ \frac{1}{n} \sum_{\substack{i:\mathrm{z}_i^{\ast\ast}=l_1\\ \mathrm{r}_{i,j_0}=1 }}  \min_{t=1,\dots,k} \mathrm{r}_{i,j_0} \cdot ( \mu_{l_1,j_0}^{\ast\ast} - \hat{\mu}_{t,j_0} )^2 \right\}^{1/2}
        = \left\{ \frac{1}{n} \sum_{\substack{i:\mathrm{z}_i^{\ast\ast}=l_1\\ \mathrm{r}_{i,j_0}=1 }}  \min_{t=1,\dots,k} ( \mu_{l_1,j_0}^{\ast\ast} - \hat{\mu}_{t,j_0} )^2 \right\}^{1/2}\\
        &= \left\{ \frac{1}{n} \sum_{\substack{i:\mathrm{z}_i^{\ast\ast}=l_1\\ \mathrm{r}_{i,j_0}=1 }}   ( \mu_{l_1,j_0}^{\ast\ast} - \hat{\mu}_{t_0,j_0} )^2 \right\}^{1/2} \\
        &\geq \left\{ \frac{1}{n} \sum_{\substack{i:\mathrm{z}_i^{\ast\ast}=l_1\\ \mathrm{r}_{i,j_0}=1 }}   \frac{1}{4} (\mu_{l_1,j_0}^{\ast\ast} - \mu_{l_2,j_0}^{\ast\ast} )^2 \right\}^{1/2} 
        =\sqrt{\frac{\#\{i\;|\; \mathrm{z}_i^{\ast\ast}=l_1, \mathrm{r}_{i,j_0}=1 \} }{n}} \cdot \frac{ |\mu_{l_1,j_0}^{\ast\ast} - \mu_{l_2,j_0}^{\ast\ast} | }{2}\\
        &\geq \frac{1}{2}\sqrt{\frac{ \mathrm{n}_{\min}^{\textnormal{feature}} }{n} } \cdot \min_{l\neq l'}\min_{j=1,\dots,k} |\mu_{l,j}^{\ast\ast} - \mu_{l',j}^{\ast\ast} |
        = \sqrt{\frac{ \mathrm{n}_{\min}^{\textnormal{feature}} }{n} } \cdot \frac{\rho^{\ast\ast}}{2},
    \end{align*}
    where $\mathrm{n}_{\min}^{\textnormal{feature}} = \min\limits_{l=1,\dots,k}\min\limits_{j=1,\dots,p} \sum_{i=1}^{n} \mathds{1}\left( \mathrm{z}_i^{\ast\ast}=l, \mathrm{r}_{ij}=1 \right)  $. 
    
    Combining the two cases (a) and (b), we obtain a lower bound of $\widetilde{D}(\widehat{\bm{\mathrm{M}}}_n,\bm{M}^{\ast\ast})$, which completes proving claim (1). 

    To prove the claim (2), we write $\hat{\mathrm{z}}_i$ for the estimated assigned cluster label of $\bm{\mathrm{x}}_i$ based on $\widehat{\bm{\mathrm{M}}}_n$, that is, 
    \begin{align*}
        \hat{\mathrm{z}}_i = \mathop{\arg\min}_{t=1,\dots,k} \| (\bm{\mathrm{x}}_i - \hat{\bm{\mu}}_t)\circ \bm{\mathrm{r}}_i \|_2,
    \end{align*}
    where if the minimizer is multiple, then take the smaller one. 
    We note that
    \begin{align*}
        \widetilde{D}(\widehat{\bm{\mathrm{M}}}_n,\bm{M}^{\ast\ast})
        &= \left\{ \frac{1}{n}\sum_{i=1}^{n} \min_{t=1,\dots,k} \| (\bm{\mu}_{\mathrm{z}_i^{\ast\ast}}^{\ast\ast} - \hat{\bm{\mu}}_t)\circ \bm{\mathrm{r}}_i \|_2^2 \right\}^{1/2}
        \leq \left\{ \frac{1}{n}\sum_{i=1}^{n} \| (\bm{\mu}_{\mathrm{z}_i^{\ast\ast}}^{\ast\ast} - \hat{\bm{\mu}}_{\hat{\mathrm{z}}_i} )\circ \bm{\mathrm{r}}_i \|_2^2 \right\}^{1/2}\\
        &\leq \underbrace{ \left\{ \frac{1}{n}\sum_{i=1}^{n} \| (\bm{\mathrm{x}}_i -  \bm{\mu}_{\mathrm{z}_i^{\ast\ast}}^{\ast\ast} )\circ \bm{\mathrm{r}}_i\|_2^2 \right\}^{1/2} }_{\text{(I)}} 
        + \underbrace{ \left\{ \frac{1}{n}\sum_{i=1}^{n} \| (\bm{\mathrm{x}}_i -  \hat{\bm{\mu}}_{\hat{\mathrm{z}}_i} )\circ \bm{\mathrm{r}}_i \|_2^2 \right\}^{1/2} }_{\text{(II)}}. \\
\end{align*}
Because
\begin{align*}
    \text{(I)} 
    &\leq \left\{ \frac{1}{n}\sum_{i=1}^{n} \| \bm{\mathrm{x}}_i -  \bm{\mu}_{\mathrm{z}_i^{\ast\ast}}^{\ast\ast} \|_2^2 \right\}^{1/2}  
    = \left\{ \frac{1}{n}\sum_{i=1}^{n} \min_{l=1,\dots,k}\| \bm{\mathrm{x}}_i -  \bm{\mu}_{l}^{\ast\ast} \|_2^2 \right\}^{1/2}
    \leq \mathrm{b}_n 
\end{align*}
and 
\begin{align*}
    \text{(II)}
    &= \left\{ \frac{1}{n}\sum_{i=1}^{n} \min_{t=1,\dots,k} \| (\bm{\mathrm{x}}_i -  \hat{\bm{\mu}}_{t} )\circ \bm{\mathrm{r}}_i \|_2^2 \right\}^{1/2}
    = \left\{ \widehat{L}_n(\widehat{\bm{\mathrm{M}}}_n) \right\}^{1/2} \\
    &\leq \left\{ \widehat{L}_n(\bm{M}^{\ast\ast}) \right\}^{1/2}
    = \left\{ \frac{1}{n}\sum_{i=1}^{n} \min_{l=1,\dots,k} \| (\bm{\mathrm{x}}_i -  \bm{\mu}_{l}^{\ast\ast} )\circ \bm{\mathrm{r}}_i \|_2^2 \right\}^{1/2} \\
    &\leq \left\{ \frac{1}{n}\sum_{i=1}^{n} \min_{l=1,\dots,k} \| \bm{\mathrm{x}}_i -  \bm{\mu}_{l}^{\ast\ast} \|_2^2 \right\}^{1/2}
    \leq \mathrm{b}_n,
\end{align*}
we obtain an upper bound of $\widetilde{D}(\widehat{\bm{\mathrm{M}}}_n,\bm{M}^{\ast\ast})$, that is, 
\begin{align*}
    \widetilde{D}(\widehat{\bm{\mathrm{M}}}_n,\bm{M}^{\ast\ast})
    \leq 2\mathrm{b}_n,
\end{align*}
which completes the proof of claim (2). 

Finally by combining the two claims (1) and (2), we complete the proof. 
\end{proof}

\subsubsection{Proof of Lemma~\ref{lemma_converge_to_truth_perfect_label}}
\begin{proof}
    For a fixed $n$ large enough such that $\rho^{\ast\ast}>  4\mathrm{b}_n \sqrt{n/\mathrm{n}_{\min}^{\textnormal{feature}}}$, according to Lemma~\ref{lemma_converge_to_truth_condition_centers_bound}, we have $D(\widehat{\bm{\mathrm{M}}}_n,\bm{M}^{\ast\ast}) \leq 2\mathrm{b}_n \sqrt{n/\mathrm{n}_{\min}^{\textnormal{comp}}} $. 
    Then there exists a permutation $\tilde{\pi}$ such that for any $l=1,\dots,k$, 
    \begin{align*}
        \| \bm{\mu}_{l}^{\ast\ast} - \hat{\bm{\mu}}_{\tilde{\pi}(l)} \|_2  
        \leq  2\mathrm{b}_n \sqrt{\frac{n}{\mathrm{n}_{\min}^{\textnormal{comp}}}}.
    \end{align*}
    Thus, it suffices to prove that if $\rho^{\ast\ast} > 4\mathrm{b}_n\sqrt{n/\mathrm{n}_{\min}^{\textnormal{comp}}} + 2\mathrm{b}_n$, then $\hat{\mathrm{z}}_i=\tilde{\pi}(\mathrm{z}_i^{\ast\ast})$ holds for all $i=1,\dots,n$ with unique $\hat{\mathrm{z}}_i$ and $\mathrm{z}_i^{\ast\ast}$. 
    We complete the proof by contradiction. 

    If there exists some $i\in \{1,\dots,n\}$ with unique $\hat{\mathrm{z}}_i$ and $\mathrm{z}_i^{\ast\ast}$ satisfying $\hat{\mathrm{z}}_i\neq \tilde{\pi}(\mathrm{z}_i^{\ast\ast})$, then it implies that 
    $\tilde{\pi}^{-1}(\hat{\mathrm{z}}_i)\neq \mathrm{z}_i^{\ast\ast}$, where $\tilde{\pi}^{-1}$ is the inverse of $\pi$. 
    Moreover, by the definition of $\hat{\mathrm{z}}_i$, we have 
    \begin{align*}
        \underbrace{\| ( \bm{\mathrm{x}}_i - \hat{\bm{\mu}}_{\hat{\mathrm{z}}_i} )\circ \bm{\mathrm{r}}_i \|_2}_{\text{(Left)}}
        = \min_{t=1,\dots,k} \| ( \bm{\mathrm{x}}_i - \hat{\bm{\mu}}_{t} )\circ \bm{\mathrm{r}}_i \|_2
        \leq 
        \underbrace{\| ( \bm{\mathrm{x}}_i - \hat{\bm{\mu}}_{\tilde{\pi}(\mathrm{z}_i^{\ast\ast})} )\circ \bm{\mathrm{r}}_i \|_2}_{\text{(Right)}}.
    \end{align*}
    For the right hand, it is upper bounded by  
    \begin{align*}
        \text{(Right)}&=\| ( \bm{\mathrm{x}}_i - \hat{\bm{\mu}}_{\tilde{\pi}(\mathrm{z}_i^{\ast\ast})} )\circ \bm{\mathrm{r}}_i \|_2
        \leq \| ( \bm{\mathrm{x}}_i - \bm{\mu}_{\mathrm{z}_i^{\ast\ast}}^{\ast\ast} )\circ \bm{\mathrm{r}}_i \|_2 
        + \| ( \bm{\mu}_{\mathrm{z}_i^{\ast\ast}}^{\ast\ast} - \hat{\bm{\mu}}_{\tilde{\pi}(\mathrm{z}_i^{\ast\ast})} )\circ \bm{\mathrm{r}}_i \|_2 \\
        &\leq \| \bm{\mathrm{x}}_i - \bm{\mu}_{\mathrm{z}_i^{\ast\ast}}^{\ast\ast}  \|_2 + \| \bm{\mu}_{\mathrm{z}_i^{\ast\ast}}^{\ast\ast} - \hat{\bm{\mu}}_{\tilde{\pi}(\mathrm{z}_i^{\ast\ast})} \|_2 \\
        &\leq \mathrm{b}_n + 2\mathrm{b}_n\sqrt{\frac{n}{\mathrm{n}_{\min}^{\textnormal{comp}}}}.
    \end{align*}
    For the left hand, it is lower bounded by 
    \begin{align*}
        \text{(Left)}&=\| ( \bm{\mathrm{x}}_i - \hat{\bm{\mu}}_{\hat{\mathrm{z}}_i} )\circ \bm{\mathrm{r}}_i \|_2
        \geq \| ( \hat{\bm{\mu}}_{\hat{\mathrm{z}}_i} - \bm{\mu}_{\mathrm{z}_i^{\ast\ast}}^{\ast\ast} )\circ \bm{\mathrm{r}}_i \|_2 
        - \| ( \bm{\mathrm{x}}_i - \bm{\mu}_{\mathrm{z}_i^{\ast\ast}}^{\ast\ast} )\circ \bm{\mathrm{r}}_i \|_2 \\
        &\geq \| (\bm{\mu}_{\mathrm{z}_i^{\ast\ast}}^{\ast\ast} - \bm{\mu}_{\tilde{\pi}^{-1}(\hat{\mathrm{z}}_i)}^{\ast\ast} )\circ \bm{\mathrm{r}}_i \|_2 
        -  \| ( \hat{\bm{\mu}}_{\hat{\mathrm{z}}_i} - \bm{\mu}_{\tilde{\pi}^{-1}(\hat{\mathrm{z}}_i)}^{\ast\ast} )\circ \bm{\mathrm{r}}_i \|_2 
        - \| ( \bm{\mathrm{x}}_i - \bm{\mu}_{\mathrm{z}_i^{\ast\ast}}^{\ast\ast} )\circ \bm{\mathrm{r}}_i \|_2 \\
        &\geq \min_{j:\mathrm{r}_{ij}=1} \left| \mu_{\mathrm{z}_{i}^{\ast\ast},j}^{\ast\ast} - \mu_{\tilde{\pi}^{-1}(\hat{\mathrm{z}}_i),j}^{\ast\ast} \right| 
        - \| \hat{\bm{\mu}}_{\hat{\mathrm{z}}_i} - \bm{\mu}_{\tilde{\pi}^{-1}(\hat{\mathrm{z}}_i)}^{\ast\ast} \|_2 
        - \| \bm{\mathrm{x}}_i - \bm{\mu}_{\mathrm{z}_i^{\ast\ast}}^{\ast\ast} \|_2\\
        &\geq \rho^{\ast\ast} - 2\mathrm{b}_n \sqrt{\frac{n}{\mathrm{n}_{\min}^{\textnormal{comp}}}} - \mathrm{b}_n.
    \end{align*}
    It follows that 
    \begin{align*}
        \rho^{\ast\ast} - 2\mathrm{b}_n \sqrt{\frac{n}{\mathrm{n}_{\min}^{\textnormal{comp}}}} - \mathrm{b}_n
        \leq \mathrm{b}_n + 2\mathrm{b}_n \sqrt{\frac{n}{\mathrm{n}_{\min}^{\textnormal{comp}}}},
    \end{align*}
    which leads to a contradiction. 
    Consequently, we complete the proof. 
\end{proof}

\subsubsection{Proof of Theorem~\ref{theorem_converge_to_truth}}
\begin{proof}
    On one hand, according to the definitions of $\widehat{\bm{\mathrm{M}}}_n$ and $\hat{\mathrm{z}}_i$, we can write for any $t=1,\dots,k$, $j=1,\dots,p$,
    \begin{align*}
        \hat{\mu}_{tj}= \frac{\sum_{i=1}^{n} \mathds{1}( \hat{\mathrm{z}}_i=t ) \mathrm{r}_{ij} \mathrm{x}_{ij} }{ \sum_{i=1}^{n} \mathds{1}( \hat{\mathrm{z}}_i=t ) \mathrm{r}_{ij} }. 
    \end{align*}
    Under MCAR mechanism (Assumption~\ref{assumption_mcar}), since $\mathrm{r}_{ij}$'s are independent with $\mathrm{x}_{ij}$'s, then we have 
    \begin{align*}
        \lim_{n\rightarrow\infty} \text{Pr}\left( \left| \hat{\mu}_{tj} - \frac{\sum_{i=1}^{n} \mathds{1}( \hat{\mathrm{z}}_i=t ) \mathrm{x}_{ij} }{ \sum_{i=1}^{n} \mathds{1}( \hat{\mathrm{z}}_i=t ) } \right| \leq \frac{\epsilon}{2} \right)  = 1,\; \forall \epsilon>0. 
    \end{align*}
    On the other hand, according to the definition of $\mathrm{z}_i^{\ast\ast}$ and the Law of Large Number, we have for any $l=1,\dots,k$,
    \begin{align*}
        \lim_{n\rightarrow\infty} \text{Pr}\left( 
        \left| 
        \frac{\sum_{i=1}^{n} \mathds{1}( \mathrm{z}_i^{\ast\ast}=l )  \mathrm{x}_{ij} }{ \sum_{i=1}^{n} \mathds{1}( \mathrm{z}_i^{\ast\ast}=l ) } 
        - \mu_{lj}^{\ast\ast} \right| \leq \frac{\epsilon}{2}
        \right)=1, \; \forall \epsilon>0.
    \end{align*}
    Moreover, for any fixed $n$ large enough, if $\rho^{\ast\ast}/\beta^{\ast\ast}> 2\max\left\{ 2\sqrt{n/\mathrm{n}_{\min}^{\textnormal{feature}}},\;  2\sqrt{n/\mathrm{n}_{\min}^{\textnormal{comp}}} + 1 \right\}$ holds, it implies that $\rho^{\ast\ast}> \max\left\{ 4\mathrm{b}_n \sqrt{n/\mathrm{n}_{\min}^{\textnormal{feature}}},\;  4\mathrm{b}_n\sqrt{n/\mathrm{n}_{\min}^{\textnormal{comp}}} + 2\mathrm{b}_n \right\}$ holds under Assumption~\ref{assumption_max_within_cluster_distance}. 
    Then according to Lemma~\ref{lemma_converge_to_truth_perfect_label}, there must exist a permutation $\pi$ such that $\hat{\mathrm{z}}_i=\pi(\mathrm{z}_i^{\ast\ast})$ for any $i=1,\dots,n$ with unique $\hat{\mathrm{z}}_i$ and $\mathrm{z}_i^{\ast\ast}$. 
    Since under Assumption~\ref{assumption_surface_zeromeasure}, the event $\{\hat{\mathrm{z}}_i \text{ or } \mathrm{z}_i^{\ast\ast} \text{ is not unique} \}$ has zero $\widetilde{\mathbb{P}}$-measure, then $\mathds{1}\big( \hat{\mathrm{z}}_i=t \big) =  \mathds{1}\big( \mathrm{z}_i^{\ast\ast}=\pi^{-1}(t) \big)$ holds for all $i=1,\dots,n$ almost surely, that is, 
    \begin{align*}
        \textnormal{Pr}\bigg(  \mathds{1}\big( \hat{\mathrm{z}}_i=t \big) =  \mathds{1}\big( \mathrm{z}_i^{\ast\ast}=\pi^{-1}(t) \big),\; \forall i=1,\dots,n \bigg) =1,\; \forall t=1,\dots,k.
    \end{align*}
    It follows that for any $t=1,\dots,k$,
    \begin{align*}
        \textnormal{Pr}\left( 
        \frac{\sum_{i=1}^{n} \mathds{1}( \hat{\mathrm{z}}_i=t )  \mathrm{x}_{ij} }{ \sum_{i=1}^{n} \mathds{1}( \hat{\mathrm{z}}_i=t ) } 
        =\frac{\sum_{i=1}^{n} \mathds{1}( \mathrm{z}_i^{\ast\ast}=\pi^{-1}(t) )  \mathrm{x}_{ij} }{ \sum_{i=1}^{n} \mathds{1}( \mathrm{z}_i^{\ast\ast}=\pi^{-1}(t) ) }
        \right) = 1. 
    \end{align*}
    Finally, combining the above two limitations lead to for any $t=1,\dots,k$,
    \begin{align*}
        \lim_{n\rightarrow\infty} \text{Pr}\left( 
       \left| \hat{\mu}_{tj} 
        - \mu_{\pi^{-1}(t),j}^{\ast\ast} \right| \leq \epsilon
        \right)=1, \; \forall \epsilon>0,
    \end{align*}
    which complete the proof. 
\end{proof}

\end{document}